\documentclass[final,leqno,onefignum,onetabnum]{siamart1116}
\usepackage{booktabs}
\usepackage{multirow, array}

\usepackage{amsmath,amssymb}
\usepackage{graphicx}

\DeclareGraphicsExtensions{.pdf,.jpeg,.png}

\newcommand{\Hc}{{\mathcal{H}}}

\newcommand{\Kc}{{\mathcal K}}

\newcommand{\Qc}{{\mathcal Q}}

\newcommand{\Sc}{{\mathcal S}}

\newcommand{\Wc}{{\mathcal W}}

\newcommand{\Yc}{{\mathcal Y}}

\renewcommand{\Re}{{\mathbb{R}}}
\newcommand{\Rd}{{\mathbb R}}
\newcommand{\Cd}{{\mathbb C}}

\newcommand{\tr}{\mathrm{Tr}}

\usepackage{algorithm,algpseudocode}

\newcommand{\hank}{\mathbb{H}}
\newcommand{\circul}{\mathfrak{C}}

\renewcommand{\vec}{\textsc{Vec}}

\newcommand{\rank}{\textsc{rank}}

  \newtheorem{remark}{Remark}

\newcommand{\beq}{\begin{equation}}
\newcommand{\eeq}{\end{equation}}
\newcommand{\beqa}{\begin{eqnarray}}
\newcommand{\eeqa}{\end{eqnarray}}

\title{Deep Convolutional Framelets: A General Deep Learning  Framework for Inverse Problems\thanks{The authors would like to thanks Dr. Cynthia MaCollough,  the Mayo Clinic, the American Association of Physicists in Medicine (AAPM), and grant EB01705 and EB01785 from the National
Institute of Biomedical Imaging and Bioengineering for providing the Low-Dose CT Grand Challenge data set.
This work is supported by National Research Foundation of Korea, Grant number
NRF-2016R1A2B3008104, NRF-2015M3A9A7029734, and NRF-2017M3C7A1047904.
}} 

\author{
Jong Chul Ye\footnotemark[1]\ \footnotemark[2]
\and
Yoseob Han\footnotemark[1]
\and
Eunju Cha\footnotemark[1]
}

\begin{document}
\maketitle
\newcommand{\slugmaster}{%
\slugger{siims}{xxxx}{xx}{x}{x--x}}%slugger should be set to juq, siads, sifin, or siims
\renewcommand{\thefootnote}{\fnsymbol{footnote}}
\footnotetext[1]{Bio Imaging and Signal Processing Lab., Department of Bio and Brain Engineering, Korea Advanced Institute of Science and Technology, 291 Daehak-ro, Yuseong-gu, Daejeon 34141, Republic of Korea (\email{jong.ye@kaist.ac.kr}; \email{hanyoseob@kaist.ac.kr}; \email{eunju.cha@kaist.ac.kr}).}
\footnotetext[2]{Address all correspondence to J. C. Ye  at \email{jong.ye@kaist.ac.kr}, Ph.:+82-42-3504320, Fax:+82-42-3504310.}
\renewcommand{\thefootnote}{\arabic{footnote}}

%
%\title{
%Deep Convolutional Framelets: A General Deep Learning   for Inverse Problems
%}
%
%\author{Jong~Chul~Ye$^{*}$, and Yo Seob Han% <-this % stops a space
%\thanks{J.C. Ye and Y. S. Han are with the Department of Bio and Brain Engineering, Korea Advanced Institute of Science and Technology (KAIST), 
%		Daejeon 34141, Republic of Korea (e-mail: \{jong.ye,hanyoseob\}@kaist.ac.kr).}% <-this % stops a space

% The only time the second header will appear is for the odd numbered pages after the title page when using the twoside option.
% *** Note that you probably will NOT want to include the author's name in the headers of peer review papers. ***
% You can use \ifCLASSOPTIONpeerreview for conditional compilation here if you desire.

% If you want to put a publisher's ID mark on the page you can do it like this:
% Remember, if you use this you must call \IEEEpubidadjcol in the second column for its text to clear the IEEEpubid mark.


%\maketitle

% As a general rule, do not put math, special symbols or citations in the abstract or keywords.
\begin{abstract}
Recently, deep learning approaches with various network architectures have achieved significant performance improvement over existing iterative reconstruction methods
 in various imaging problems.
However,   it is still unclear {\em why} these deep learning architectures work for specific inverse problems.
Moreover, in contrast to the usual evolution of signal processing theory  around the classical theories, 
 the link between deep learning  and the classical signal processing
  approaches such as wavelets, non-local processing, compressed sensing, etc, are not yet well understood.
To address these issues, here
we show that  the long-searched-for missing link is the {convolution framelets} 
for representing a signal
by convolving local and non-local bases.
The  convolution framelets was originally developed to generalize the   theory of low-rank Hankel matrix approaches for inverse problems, and this paper further extends the idea so that we can obtain a deep neural network using multilayer convolution framelets with perfect reconstruction (PR) under rectilinear linear unit nonlinearity (ReLU). Our analysis also shows that the popular deep  network components such as residual block, redundant filter channels,
and concatenated ReLU (CReLU) do indeed help to achieve the PR, while the pooling and unpooling layers should be augmented with high-pass branches to meet the PR condition. 
Moreover,  by changing the number of filter channels and bias, we can control the shrinkage behaviors of the neural network.
This discovery reveals the limitations of many existing deep learning architectures for inverse problems, and leads us to propose
a novel theory for {\em deep convolutional framelets}  neural network.
Using numerical experiments with various inverse problems,  we demonstrated that our deep convolution framelets network 
shows consistent improvement over  
existing deep architectures.
This discovery suggests that the success of deep learning is not from a magical power of a black-box, but rather comes from the
power of a  novel signal representation  using non-local basis combined with data-driven local basis, which is indeed a natural extension of  classical signal processing theory.

\end{abstract}

%\begin{IEEEkeywords}
%\end{IEEEkeywords}



\pagestyle{myheadings}
\thispagestyle{plain}
\markboth{JONG CHUL YE, YOSEOB HAN AND EUNJU CHA}{DEEP CONVOLUTIONAL FRAMELETS FOR INVERSE PROBLEMS}

% For peer review papers, you can put extra information on the cover page as needed:
% \fi
%
%\vspace{0.5in}
%\noindent Correspondence to:\\
%Jong Chul Ye, Ph.D\\
%KAIST Endowed Chair Professor\\
%Department of Bio and Brain Engineering\\
%Department of Mathematical Sciences \\
%Korea Advanced Institute of Science and Technology (KAIST)\\
%291 Daehak-ro, Yuseong-gu, Daejeon 34141, Republic of Korea\\
%Tel: +82-42-350-4320\\
%Email: jong.ye@kaist.ac.kr
%
%% For peerreview papers, this IEEEtran command inserts a page break and creates the second title. 
%% It will be ignored for other modes.
%
%\IEEEpeerreviewmaketitle

%\clearpage

\section{Introduction}
\label{sec:introduction}

{D}{eep} 
 learning approaches have achieved tremendous
success in  classification problems~\cite{krizhevsky2012imagenet} as well as low-level computer vision problems such as segmentation~\cite{ronneberger2015u}, denoising~\cite{zhang2016beyond}, super-resolution \cite{kim2015accurate,shi2016real}, etc. 
The theoretical origin of its success  has been investigated \cite{poole2016exponential,telgarsky2016benefits}, and
the exponential expressivity under a given network complexity (in terms of VC dimension \cite{anthony2009neural} or Rademacher complexity \cite{bartlett2002rademacher})
has been often attributed to its success.
A deep network is also known to learn    high-level abstractions/features of the data similar to the visual processing in human brain  using multiple layers of neurons with non-linearity
\cite{lecun2015deep}.

Inspired by the success of deep learning in low-level computer vision, several machine learning approaches have been recently
proposed for image reconstruction problems.
 In X-ray computed tomography (CT),
 Kang et al \cite{kang2016deep} provided the first systematic study of deep convolutional neural network (CNN) for low-dose CT and showed that a deep CNN using directional wavelets is more efficient in removing low-dose related CT noises. Unlike these low-dose artifacts  from reduced tube currents, the streaking artifacts originated from sparse projection views show globalized artifacts that are difficult to remove using conventional denoising CNNs \cite{chen2015learning, mao2016image}. Han et al \cite{han2016deep} and Jin et al \cite{jin2016deep} independently proposed a residual learning using U-Net \cite{ronneberger2015u} to remove the global streaking artifacts caused by sparse projection views. 
In MRI,  Wang et al~\cite{wang2016accelerating}  was the first to apply deep learning to compressed sensing MRI (CS-MRI). They trained the deep neural network from  downsampled reconstruction images to learn a fully sampled reconstruction. 
Then, they used the deep learning result  either as  an initialization  or as a  regularization term in classical CS approaches. 
Multilayer perceptron was developed for accelerated parallel MRI \cite{kwon2016learning,kwon2017parallel}.
Deep network architecture using unfolded iterative compressed sensing (CS) algorithm was also proposed 
\cite{hammernik2016learning}. 
Instead of  using handcrafted regularizers, the authors in \cite{hammernik2016learning}  tried to learn a set of optimal regularizers.
Domain adaptation from sparse view CT network to projection reconstruction MRI was also proposed \cite{han2017deep}.
These pioneering works have consistently demonstrated impressive reconstruction performances, 
which are often superior to the existing iterative approaches.
%
%However, the more we have observed impressive empirical results in image reconstruction problems,  we encounter more unanswered questions which includes: %There are several open problems in deep learning, of which the partial list includes
%%For example, % 
However, the more we have observed impressive empirical results in image reconstruction problems,  the more unanswered questions we encounter.
%There are several open problems in deep learning, of which the partial list includes
For example, to our best knowledge, we do not have the complete answers to the following questions that  are critical to a network design:
\begin{enumerate}
\item What is the role of the filter channels in convolutional layers ? %Can we  achieve the same performance with large filter size ?
\item Why do some networks need a fully connected layers whereas the others do not ?
\item What is the role of the nonlinearity such as rectified linear unit (ReLU) ?
\item Why do we need a pooling and unpooling in some architectures ?
\item What is the role of by-pass connection or residual network ?
\item How many layers do we need ?
%\item Is there any link to the classical signal processing theory such as wavelet, non-local means, compressed sensing, etc ?
\end{enumerate}
%While some of these questions  have been answered, at least partially, in the literature\footnote{For instance, the work of \cite{?}
% showed that once ReLU nonlinearity is employed, the
%forward pass of a network can be interpreted as a deep sparse coding algorithm.
%The work of \cite{?} discusses the importance of pooling for networks,
%proving that it leads to translation invariance. Moreover, several works
%provided explanations for residual networks, for instance \cite{?}.
%},
%to our best knowledge we still do not have the complete answers  that  are critical to a network design.
Furthermore, the most troubling issue for signal processing  community is that  the link to the classical signal processing theory is still not  fully understood.
% such that we do not know clearly why it works.
For example, wavelets \cite{daubechies1992ten}  has been extensively investigated as an efficient  signal representation theory for many image processing
applications by exploiting  energy compaction property of  wavelet bases. Compressed sensing theory \cite{Do06,CaRoTa06}  has further extended the idea to demonstrate that an accurate recovery is possible from undersampled data, if the signal is sparse in some frames and the sensing matrix is incoherent.
Non-local image processing techniques such as non-local means \cite{buades2005non},  BM3D \cite{dabov2007image}, etc have also demonstrated impressive
performance for many image processing applications.
The link between these algorithms have been extensively studied for last few years using various mathematical tools from harmonic analysis, convex optimization, etc.
However, recent years have witnessed that a blind application of  deep learning toolboxes sometimes provides even better
performance than mathematics-driven classical signal processing approaches. %for image processing applications becomes 
Does this imply  the dark age of signal processing or a new opportunity  ?

Therefore, the main goal of this paper is to address these open questions. In fact, our paper is
not the only attempt to address these issues. 
For instance, Papyan et al  \cite{papyan2017convolutional}
 showed that once ReLU nonlinearity is employed, the
forward pass of a network can be interpreted as a deep sparse coding algorithm.
Wiatowski et al \cite{wiatowski2015mathematical} discusses the importance of pooling for networks,
proving that it leads to translation invariance. Moreover, several works including  \cite{greff2016highway}
provided explanations for residual networks.
The interpretation of a deep network in terms of unfolded (or unrolled) sparse recovery is another prevailing view in research community \cite{gregor2010learning,xin2016maximal,hammernik2016learning,jin2016deep}. However, this interpretation still does not  give answers to several key questions: for example, why do we need multichannel filters ?
In this paper,  we therefore depart from  this existing views and propose a new interpretation of a deep network as a novel {\em signal representation} scheme.
In fact, signal representation theory  such as  wavelets and frames  have been active areas of researches for many years \cite{mallat1999wavelet}, and Mallat \cite{mallat2012group} and Bruna et al \cite{bruna2013invariant}  proposed the wavelet scattering network as a translation invariant and deformation-robust image representation. However, this approach does not have learning components
as in the existing deep learning networks.

Then, what is missing here?  One of the most important contributions of our work is to show that the geometry of deep learning
can be revealed by {\em lifting} a signal to a high dimensional space using Hankel structured matrix.
More specifically, many types of input signals that occur in signal processing can be factored into the left and right bases as well as a sparse matrix with energy compaction properties when lifted into the Hankel structure matrix.
 This results in a frame representation of the signal using the left and right bases, referred to as the non-local and local base matrices, respectively.
 The origin of this nomenclature will become clear later.
 %Our novel contribution is the realization that 
One of our novel contributions was the realization that the non-local base determines the network architecture such as pooling/unpooling, while the local basis allows the network to learn convolutional filters.
%Unlike the classical approaches for signal representation using frames, this gives us
%More specifically,  for a given nonlocal
%basis, the local basis is trained to provide the maximal energy compaction of the true signals,  while the noise and artifacts can be spread out.
More specifically,  the application-specific domain knowledge leads to a better choice of a  non-local basis, on which to learn the local basis  to maximize the performance.
 
%We demonstrate that the power of deep learning indeed comes from the optimal interplay between the nonlocal and local bases.
In fact, the idea of exploiting the two bases by the so-called  convolution framelets was originally proposed by Yin et al  \cite{yin2017tale}.
However,  the aforementioned close link to the deep neural network was not revealed in \cite{yin2017tale}.
Most importantly, we demonstrate for the first time that the convolution
framelet representation can be equivalently represented as an encoder-decoder convolution layer, and
multi-layer convolution framelet expansion is also feasible by
relaxing the  conditions in \cite{yin2017tale}. %, resulting in a multi-layer encoder-decoder network.
Furthermore, we derive the perfect reconstruction (PR) condition under
 rectified linear unit (ReLU). %For example, under ReLU,  we show that the optimal local basis is not unique but lies in a subspace. This is important because it may suggest the underlying principle of generalization power of a deep network.
The mysterious role of the redundant  multichannel filters  can be then easily understood as an important tool to  meet
the PR condition.
Moreover, by augmenting local filters with paired filters with opposite phase, the ReLU nonlinearity disappears
and the deep convolutional framelet becomes a linear signal representation.
However, in order for the deep network to satisfy the PR condition, the number of channels should increase exponentially along the layer, which is difficult to achieve in practice.
Interestingly, we can show that an insufficient number of filter channels results in shrinkage behavior via a low rank approximation of an extended Hankel matrix, and this shrinkage behavior can be exploited to maximize network performance.
%%Then, why do we need more layers ?
%%In spite of this seemingly contradictory behaviour of
%%network depth and PR conditions,
%%we show that  the number of layers  
%%
%In this paper, we discuss  this interplay between the PR
%and low-rank Hankel matrix approximation, 
%The resulting low-rank approximation 
%and its interplay with PR is discussed in detail, 
%which is believed to the origin of the superior performance of deep learning.
Finally,  to overcome the limitation of the pooling and unpooling layers,
we introduce a multi-resolution analysis (MRA) for convolution framelets using wavelet non-local
basis as a generalized pooling/unpooling.
%
%
%Seeing with the new eyes of deep convolutional framelets, we analyze the limitations of  the existing deep learning architecture for inverse problems.
%Accordingly, we  propose  a new  class of deep learning network  which results in consistent performance improvement over the existing 
%deep learning approaches. % for sparse view CT reconstruction.
We call the new class of  deep network using convolution framelets 
as  the  {\em deep convolutional framelets}.
%In the following, we begin with some mathematical preliminaries that are extensively used in this paper.

\subsection{Notations}
For a matrix $A$, $R(A)$ denotes the range space of $A$ and $N(A)$ refers to the null space of $A$. $P_{R(A)}$ denotes the projection to the range space
of $A$, whereas $P^\perp_{R(A)}$ denotes the projection to the orthogonal complement of $R(A)$.
The notation $1_{n}$  denotes  a $n$-dimensional vector with 1's. %zero matrix. 
The $n\times n$ identity matrix is referred to as $I_{n\times n}$.
For a given matrix $A\in\Rd^{m\times n}$, the notation $A^\dag$ refers to the generalized inverse. % right inverse such that
%$AA^R=I_{m\times m}.$ Similarly, $A^L\in\Rd^{n\times m}$ denotes  the left inverse such that
%$A^LA = I_{n\times n}$.
 The superscript $^{\top}$ of $A^{\top}$ denotes the Hermitian transpose. 
Because we are mainly
interested in real valued cases, $^{\top}$ is equivalent to the transpose $^T$. 
The inner product in matrix space is defined by $\langle A, B \rangle =  \tr(A^{\top}B),$  where $A,B\in \Rd^{n\times m}$.
For a matrix $A$, $\|A\|_F$ denotes its Frobenius norm.
For a given matrix $C\in \Rd^{n\times m}$,  $c_j$ denotes its $j$-th column,
and $c_{ij}$ is the $(i,j)$ elements of $C$.
If a matrix $\Psi \in \Rd^{pd\times q}$ is partitioned as $\Psi = \begin{bmatrix} \Psi_1^\top & \cdots & \Psi_p^\top \end{bmatrix}^\top$ with sub-matrix $\Psi_i \in \Rd^{d\times q}$,
then $\psi_j^i$ refers to the $j$-th column of $\Psi_i$.
A vector $\overline v \in \Rd^n$ is referred to the flipped version of a vector $v \in \Rd^n$, i.e. its indices are reversed.
Similarly, for a given  matrix $\Psi \in \Rd^{d\times q}$, the notation
$\overline \Psi \in \Rd^{d\times q}$ refers to a matrix composed of flipped vectors, i.e.
$\overline \Psi = \begin{bmatrix} \overline \psi_1 & \cdots & \overline \psi_q \end{bmatrix}.$
For a block structured matrix  $\Psi \in \Rd^{pd\times q}$,  
with a slight abuse of notation, we define $\overline \Psi$ as
\begin{eqnarray}\label{eq:block}
\overline\Psi = \begin{bmatrix} \overline\Psi_1 \\ \vdots \\\overline \Psi_p \end{bmatrix},\quad \mbox{where}\quad \overline\Psi_i
=\begin{bmatrix} \overline {\psi_1^i} & \cdots & \overline {\psi_q^i} \end{bmatrix}  \in \Rd^{d\times q} . 
\end{eqnarray}
Finally, Table~\ref{tbl:notation} summarizes the notation used throughout the paper.
%then 
% the notation
%$\overline \Phi \in \Rd^{n\times d}$ refers to a matrix composed of flipped vectors, i.e.
%$\overline \Phi = \begin{bmatrix} \overline \phi_1 & \cdots & \overline \phi_d \end{bmatrix}.$

%\section{Low-Rank Hankel Matrix Approaches}
%
%In this section, we review the theory of annihilating filter-based low-rank Hankel matrix approach  (ALOHA)\cite{ye2016compressive,jin2015annihilating,jin2016general,ongie2016off,lee2016acceleration,lee2016reference,jin2016mri}, which 
%was recently
%proposed as powerful image processing and inverse problem techniques.
%These turn out to be the key ingredients of the proposed method.
%%Note that these approaches 
%For simplicity,  our theory is developed using 1-D signal model, but the extension to multi-dimensional
%signal model is straightforward.  We first begin with some mathematical preliminaries.

\begin{table}[!hbt]
\begin{center}
\begin{tabular}{c|l}
\hline
Notation & Definition \\
\hline
 $\Phi$  &  non-local basis matrix at the  encoder\\
 $\tilde\Phi$  &  non-local basis matrix  at the  decoder \\
     $\Psi$  &  local basis matrix at the  encoder\\
      $\tilde\Psi$  &  local basis matrix  at the  decoder \\
      $b_{enc},b_{dec}$ & encoder and decoder biases \\
   $\phi_i$  &  $i$-th non-local basis or filter at the encoder \\
   $\tilde \phi_i$  &  $i$-th non-local basis or filter at the decoder \\
  $\psi_i$  &  $i$-th local basis or filter at the encoder \\
   $\tilde \psi_i$  &  $i$-th local basis or filter at the decoder \\
   $C$  & convolutional framelet coefficients at the encoder\\ 
    $\tilde C$  & convolutional framelet coefficients at the decoder\\ 
     $n$ &   input dimension\\
   $d$ &  convolutional filter length \\
   $p$ &  number of input  channels \\
   $q$  & number of output channels \\
   $f$ &  single channel input signal, i.e.  $f\in \Rd^n$ \\
   $Z$ & a $p$-channel input signal, i.e.  $Z\in \Rd^{n\times p}$ \\
     $\hank_d(\cdot)$  &  Hankel operator, i.e.  $\hank_d: \Rd^n \mapsto Y\subset \Rd^{n\times d}$ \\
      $\hank_{d|p}(\cdot)$  & extended Hankel operator, i.e.  $\hank_{d|p}: \Rd^{n\times p} \mapsto Y\subset \Rd^{n\times pd}$ \\
           $\hank_d^\dag(\cdot)$  &  generalized inverse of Hankel  operator, i.e.   $\hank_d^\dag: \Rd^{n\times d} \mapsto  \Rd^{n}$ \\
      $\hank_{d|p}^\dag(\cdot)$  &  generalized inverse of an extended Hankel  operator, i.e.   $\hank_{d|p}^\dag: \Rd^{n\times pd} \mapsto  \Rd^{n\times p}$ \\
      $U$ & left singular vector matrix of an (extended) Hankel matrix  \\
      $V$ & right singular vector matrix of an (extended) Hankel matrix \\    
      $\Sigma$  &  singular value matrix of an (extended) Hankel matrix \\  
      $\circul_d(\cdot)$ & $n\times d$-circulant matrix \\
%Low-band non-local basis & $\Phi_{low}^{(l)}$  &  $\frac{n}{2^{l-1}}\times \frac{n}{2^{l}}$\\
%High-band non-local basis & $\Phi_{high}^{(l)}$  &  $\frac{n}{2^{l-1}}\times \frac{n}{2^{l}}$\\
%Local basis & $\Psi^{(l)}$ & $p_{(l)}d_{(l)}\times p_{(l)}d_{(l)}$ \\
%Dual local basis & $\tilde\Psi^{(l)}$ & $p_{(l)}d_{(l)}\times p_{(l)}d_{(l)}$  \\
%Signal and its estimate & $C^{(l)}, \hat C^{(l)} $   & $\frac{n}{2^{l-1}} \times p_{(l)}d_{(l)}$ \\
%Approximate signal and its estimate & $C_{low}^{(l)},\hat C_{low}^{(l)}$   & $\frac{n}{2^{l}} \times p_{(l)}d_{(l)}$ \\
%Detail signal & $C_{high}^{(l)}$   & $\frac{n}{2^{l}} \times p_{(l)}d_{(l)}$ \\
%Hankel lifting of low-band signals  &  $\hank_{d_{(l)}|p_{(l)}}(C_{low}^{(l-1)})$  &  $\frac{n}{2^{l-1}} \times p_{(l)}d_{(l)}$ \\
\hline
\end{tabular}
\caption{Notation and definition used throughout the  paper.}
\label{tbl:notation}
\end{center}
\end{table}

\section{Mathematics of Hankel matrix} % Preliminaries}

Since the Hankel structured matrix is the key component in our theory,  this section discusses various properties of the Hankel matrix that will be extensively used throughout the paper.

%In this section,  we  discuss  basic properties of Hankel matrix 
%Amazingly, Hankel matrix has intriguing structure that are 
%This section discusses many important properties of Hankel matrix, which 
%will be extensively used in our analysis. 

\subsection{Hankel matrix representation of convolution}

Hankel matrices arise  repeatedly from many different contexts in signal processing and control theory,
such as system identification  \cite{fazel2013hankel}, harmonic retrieval,  array signal processing \cite{hua1990matrix},  subspace-based channel identification \cite{tong1994blind}, etc.
 A Hankel matrix can be also obtained  from a convolution operation \cite{ye2016compressive}, which is of particular interest in this paper.
Here,  to avoid special treatment of boundary condition, our theory is mainly derived using the circular convolution.

Let $f=[f[1],\cdots, f[n]]^T\in \Rd^n$ and $\psi=[\psi[1],\cdots, \psi[d]]^T\in\Rd^d$. % be an input image and a convolutional filter.
Then, a single-input single-output (SISO) convolution of the input $f$ and the filter $\overline \psi$  can be represented in a matrix form:
\begin{eqnarray}\label{eq:SISO}
y = f\circledast \overline\psi &=& \hank_d(f) \psi \ ,
\end{eqnarray}
where  $\hank_d(f)$ is a wrap-around  Hankel matrix:
 \begin{eqnarray} \label{eq:hank}
\hank_d(f) =\left[
        \begin{array}{cccc}
        f[1]  &   f[2] & \cdots   &   f[d]   \\%  &\cdots&   \yb_i(1)\\
       f[2]  &   f[3] & \cdots &     f[d+1] \\
           \vdots    & \vdots     &  \ddots    & \vdots    \\
              f[n]  &   f[1] & \cdots &   f[d-1] \\
        \end{array}
    \right] 
    \end{eqnarray}
Similarly, a single-input multi-output (SIMO) convolution using $q$ filters $\overline\psi_1,\cdots, \overline\psi_q \in \Rd^d$ can be represented by
\begin{eqnarray}\label{eq:simo}
Y = f \circledast \overline\Psi = 
\hank_d(f)  \Psi
\end{eqnarray}
where 
\begin{eqnarray*}%\label{eq:YH}
Y:=\begin{bmatrix}y_1 & \cdots & y_q \end{bmatrix} \in \Rd^{n\times q},~ \Psi:=  \begin{bmatrix} \psi_1  & \cdots & \psi_q \end{bmatrix} \in \Rd^{d\times q}.
\end{eqnarray*}
%and $q$ denotes the number of output channels.
On the other hand, multi-input multi-output (MIMO) convolution for the $p$-channel
input $Z=[z_1,\cdots,z_p]$ can be represented  by
\begin{eqnarray}\label{eq:MIMO}
y_i = \sum_{j=1}^{p} z_j\circledast \overline\psi_i^j,\quad i=1,\cdots, q
\end{eqnarray}
where $p$ and $q$ are the number of  input and output channels, respectively;
$\overline\psi_i^j \in \Rd^d$ denotes the length $d$- filter that convolves the $j$-th channel input to compute its contribution to 
the
$i$-th output channel. 
By defining the MIMO filter kernel $\Phi$ as follows:
\begin{eqnarray}
\Psi = \begin{bmatrix} \Psi_1 \\ \vdots \\ \Psi_p \end{bmatrix} \, \quad \mbox{where} \quad \Psi_j =  \begin{bmatrix} \psi_1^j  & \cdots & \psi_q^j \end{bmatrix} \in \Rd^{d\times q}  
\end{eqnarray}
the corresponding matrix representation of the MIMO convolution is then given by 
\begin{eqnarray}
Y &=& Z \circledast \overline\Psi \label{eq:MIMO_form}\\
&=& \sum_{j=1}^p \hank_d(z_j) \Psi_j  \label{eq:multifilter0}\\
%&=& \hank_{d|p}\left(Z\right) \begin{bmatrix} \Psi_1 \\ \vdots \\ \Psi_p \end{bmatrix} \notag\\
&=& \hank_{d|p}\left(Z\right) \Psi \label{eq:multifilter}
\end{eqnarray}
where    $\overline\Psi$ is a flipped block structured matrix in the sense of \eqref{eq:block},
and
$\hank_{d|p}\left(Z\right)$ is  an {\em extended Hankel matrix}  by stacking  $p$ Hankel matrices side by side: %  given by
\begin{eqnarray}\label{eq:ehank}
\hank_{d|p}\left(Z\right)  := \begin{bmatrix} \hank_d(z_1) & \hank_d(z_2) & \cdots & \hank_d(z_p) \end{bmatrix} \ . %\quad \in f(n,d;p)
\end{eqnarray}
For notational simplicity, we denote $\hank_{d|1}([z]) = \hank_d(z)$.
Fig.~\ref{fig:hankel} illustrates the procedure to construct an extended Hankel matrix from $[z_1,z_2,z_3] \in \Rd^{8\times 3}$
when the convolution filter length $d$ is 2.

  \begin{figure}[!bt] 
\center{\includegraphics[width=9cm]{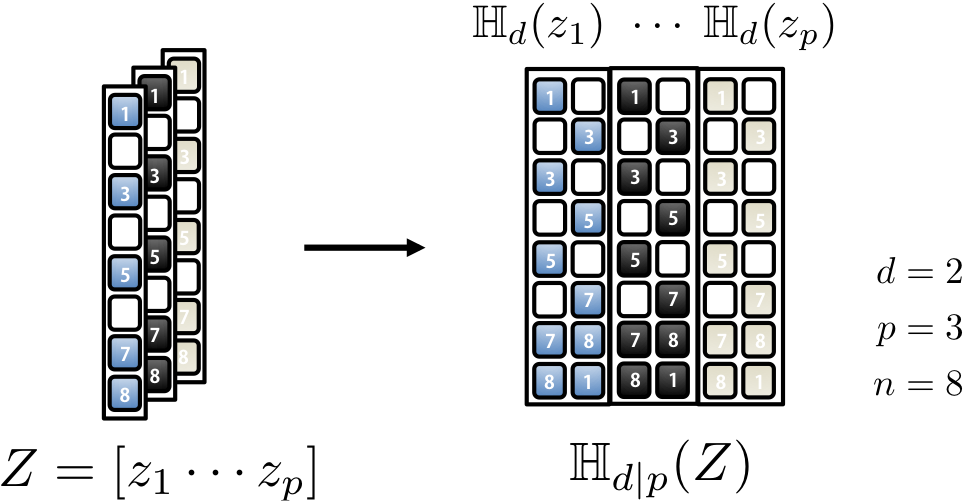}}
\caption{Construction of an extended Hankel matrix for 1-D multi-channel input patches.
}
\label{fig:hankel}
\end{figure}

Finally, as a special case of MIMO convolution for $q=1$,
the multi-input single-output (MISO) convolution is defined by
\begin{eqnarray}\label{eq:MISO}
y &=&  \sum_{j=1}^{p} z_j\circledast \overline\psi^j =  Z \circledast \Psi  = \hank_{d|p}\left(Z\right)  \Psi
\end{eqnarray}
where  
$$ \Psi = \begin{bmatrix}  \psi^1 \\ \vdots \\ \psi^p \end{bmatrix}.$$
The SISO, SIMO, MIMO, and MISO convolutional operations are illustrated in Fig.~\ref{fig:conv}(a)-(d).

  \begin{figure}[!t] 
\center{\includegraphics[width=14cm]{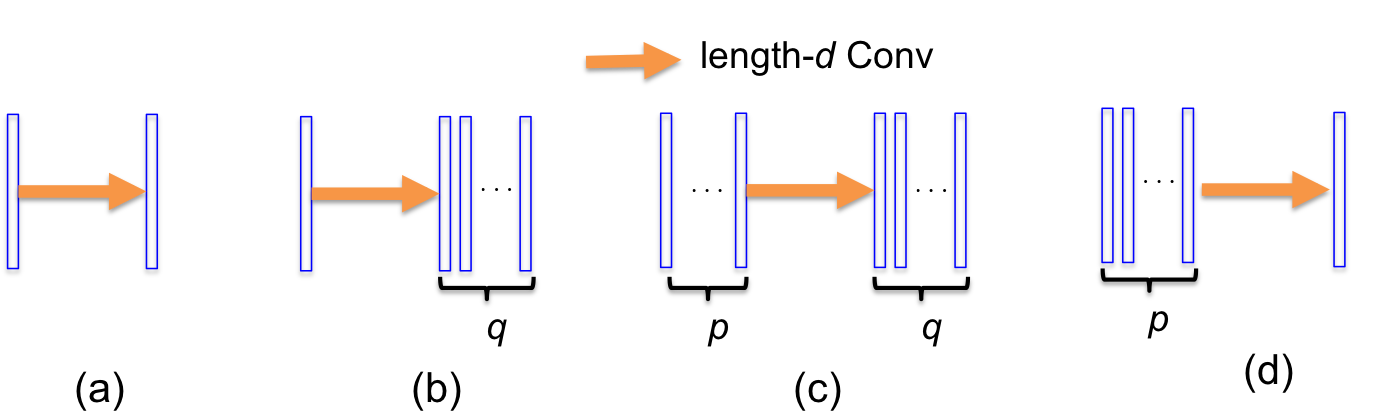}}
\caption{1-D convolutional operations  and their Hankel matrix representations. (a) Single-input single-output convolution  $y =f\circledast  \overline\psi$, (b) SIMO convolution $Y =f\circledast \overline\Psi$, 
(c) MIMO convolution $Y = Z \circledast \overline\Psi$,
and (d) MISO convolution $y = Z \circledast \overline\Psi$,
}
\label{fig:conv}
\end{figure}

%\subsubsection{2-D Convolution}

The extension to  the multi-channel 
2-D convolution operation for an image domain CNN (and multi-dimensional convolutions  in general) is straight-forward,
since similar matrix vector operations can be also used.
Only required change is the definition of the (extended) Hankel matrices, which is now defined as
{\em block} Hankel matrix. Specifically, for a 2-D input $X = [x_1,\cdots, x_{n_2}]\in \Rd^{n_1\times n_2}$
with $x_i \in \Rd^{n_1}$, the block Hankel matrix associated with filtering with $d_1\times d_2$ filter is given by
 \begin{eqnarray} 
\hank_{d_1,d_2}(X) =\left[
        \begin{array}{cccc}
        \hank_{d_1}(x_1)  &   \hank_{d_1}(x_2)& \cdots   &   \hank_{d_1}(x_{d_2})   \\%  &\cdots&   \yb_i(1)\\
      \hank_{d_1}(x_2) &  \hank_{d_1}(x_3) & \cdots &     \hank_{d_1}(x_{d_2+1}) \\
           \vdots    & \vdots     &  \ddots    & \vdots    \\
          \hank_{d_1}(x_{n_2}) &   \hank_{d_1}(x_1)& \cdots &   \hank_{d_1}(x_{d_2-1})  \\
        \end{array}
    \right] \in \Rd^{n_1n_2\times d_1d_2} . %\quad \quad\quad \in \Cd^{n\times d}  \  .
%\hank_d(f) =\left[
%        \begin{array}{cccc}
%        f[d]  &   f[d-1] & \cdots   &   f[1]   \\%  &\cdots&   \yb_i(1)\\
%       f[d+1]  &   f[1] & \cdots &     f[2] \\
%           \vdots    & \vdots     &  \ddots    & \vdots    \\
%              f[d-1]  &   f[d] & \cdots &   f[n] \\
%        \end{array}
%    \right] \ . %\quad \quad\quad \in \Cd^{n\times d}  \  .
    \end{eqnarray}

Similarly, an extended block Hankel matrix from the $p$-channel $n_1\times n_2$ input image  $X^{(i)}= [x_1^{(i)},\cdots, x_{n_2}^{(i)}],
i=1,\cdots, p$ is defined by
 \begin{eqnarray} 
\hank_{d_1,d_2|p}\left([X^{(1)}\cdots X^{(p)}]\right) =
\begin{bmatrix} \hank_{d_1,d_2}(X^{(1)})   & \cdots & \hank_{d_1,d_2}(X^{(p)}) \end{bmatrix}
\in \Rd^{n_1n_2\times d_1d_2p} . %\quad \quad\quad \in \Cd^{n\times d}  \  .
%\hank_d(f) =\left[
%        \begin{array}{cccc}
%        f[d]  &   f[d-1] & \cdots   &   f[1]   \\%  &\cdots&   \yb_i(1)\\
%       f[d+1]  &   f[1] & \cdots &     f[2] \\
%           \vdots    & \vdots     &  \ddots    & \vdots    \\
%              f[d-1]  &   f[d] & \cdots &   f[n] \\
%        \end{array}
%    \right] \ . %\quad \quad\quad \in \Cd^{n\times d}  \  .
    \end{eqnarray}
Then, the output $Y\in \Rd^{n_1\times n_2}$ from the
 2-D  SISO convolution for a given image $X\in \Rd^{n_1\times n_2}$ with 2-D filter $\overline K\in \Rd^{d_1\times d_2}$ 
can be represented by a matrix vector form:
\begin{eqnarray*}
\vec(Y) = \hank_{d_1,d_2}(X){\vec(K)}
\end{eqnarray*}
where $\vec(Y)$ denotes the  vectorization operation by stacking the column vectors of the 2-D matrix $Y$.
Similarly,  2-D MIMO convolution for  given $p$ input images $X^{(j)}\in \Rd^{n_1\times n_2},j=1,\cdots, p$ with 2-D filter $\overline K_{(i)}^{(j)}\in \Rd^{d_1\times d_2}$ 
can be represented by a matrix vector form:
\begin{eqnarray}\label{eq:2dConv}
\vec(Y^{(i)})  &=&  \sum_{j=1}^p \hank_{d_1,d_2}(X^{(j)}){\vec(K_{(i)}^{(j)})} ,\quad i=1,\cdots, q
\end{eqnarray}
Therefore, by defining
\begin{eqnarray}
\Yc = \begin{bmatrix} \vec(Y^{(1)})  & \cdots & \vec(Y^{(q)})  \end{bmatrix}  
\end{eqnarray}
\begin{eqnarray}
 \Kc = \begin{bmatrix}  {\vec(K_{(1)}^{(1)})} & \cdots & {\vec(K_{(q)}^{(1)})} \\ 
\vdots & \ddots & \vdots \\
 {\vec(K_{(1)}^{(p)})} & \cdots & {\vec(K_{(q)}^{(p)})} 
 \end{bmatrix} 
\end{eqnarray} 
the 2-D MIMO convolution can be represented by
\begin{eqnarray}
\Yc =   \hank_{d_1,d_2|p}\left([X^{(1)}\cdots X^{(p)}]\right)  \Kc. \label{eq:Yc}
\end{eqnarray}
Due to these  similarities between 1-D and 2-D convolutions,  
we will therefore  use the 1-D notation throughout the paper for the sake of simplicity; however,  readers
are advised  that the same theory applies to 2-D cases.

  \begin{figure}[!hbt] 
\center{\includegraphics[width=8cm]{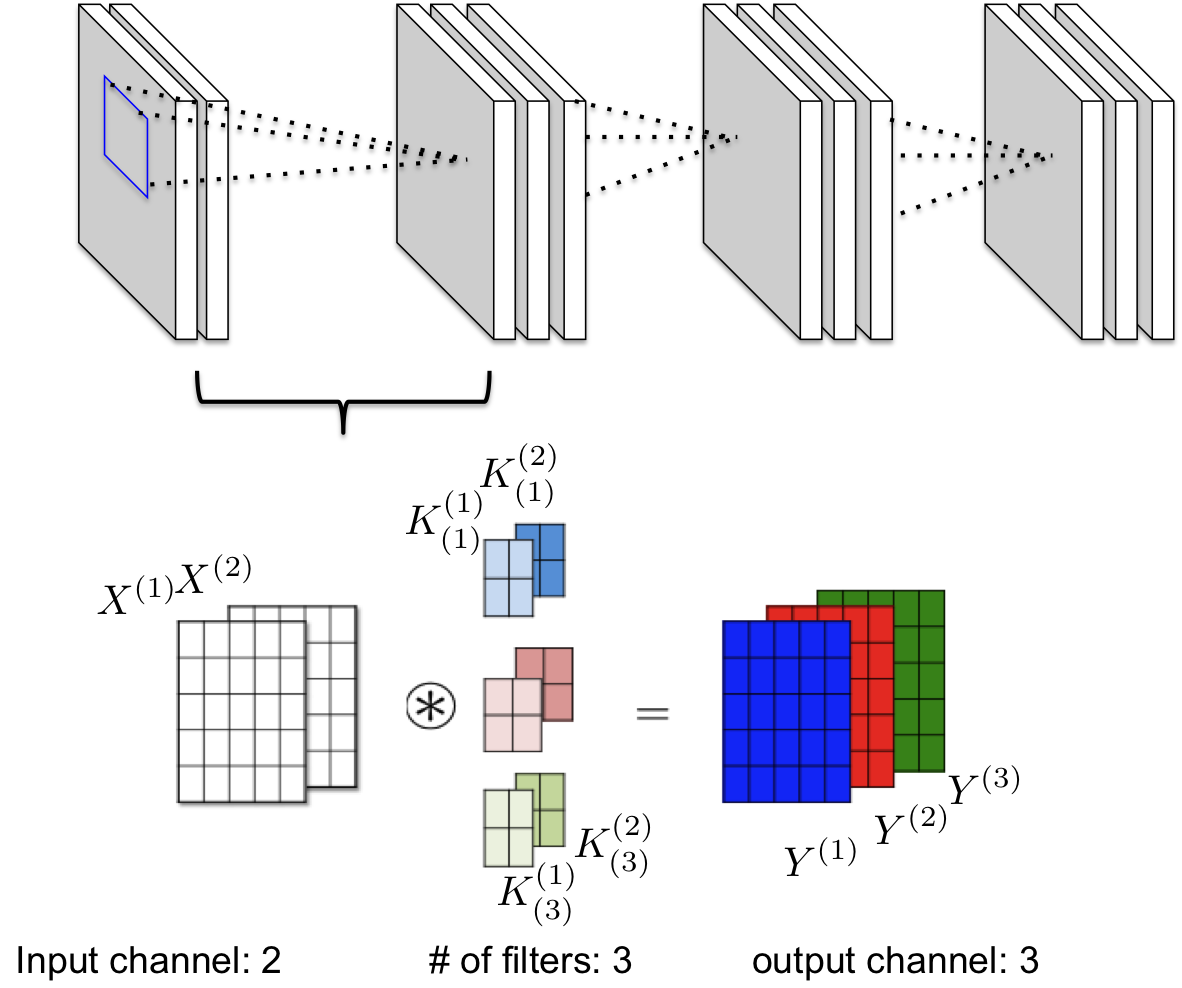}}
\caption{2-D CNN convolutional operation.  For the first layer filter, the input and output channel numbers are $p=2, q=3$, respectively, and the
filter dimension is $d_1=d_2=2$.   Thus,
the corresponding convolution
operation can be represented by $\vec(Y^{(j)})= \sum_{i=1}^2\hank_{d_1,d_2}\left(X^{(i)}\right) {\vec(K_{(j)}^{(i)})} $ where
$X^{(i)}$ and $Y^{(j)}$ denotes the $i$-th input and $j$-th output channel image, respectively; and $K_{(j)}^{(i)}$ denotes the $i$-th input channel filter to yield
$j$-th channel output.
%$\Yc = \begin{bmatrix} \vec(Y^{(1)})  & \cdots & \vec(Y^{(3)})  \end{bmatrix}$
}
\label{fig:cnnConv}
\end{figure}

In convolutional neural networks (CNN), unique multi-dimensional convolutions are used.
Specifically,  to generate $q$ output channels from  $p$ input channels, 
 each channel output is computed by first  convolving $p$- 2D filters and $p$- input channel images, and then applying the weighted sum to the 
 outputs (which is often referred to as $1\times 1$ convolution).
For 1-D signals,  this operation can be written by
\begin{eqnarray}\label{eq:cnnConv}
y_i  &=& \sum_{j=1}^p w_j \left(z_j \circledast\overline \psi_i^j\right),\quad i=1,\cdots, q
\end{eqnarray}
where $w_j$ denotes the 1-D weighting.
%Note that this operation is different from standard multi-dimensional convolution, since just a 1-D weighting is used along the channel direction.
Note that this is equivalent to an MIMO  convolution, since we have
\begin{eqnarray}
Y  &=& \sum_{j=1}^p w_j \hank_d(z_j) \Psi_j \notag \\
&=& \sum_{j=1}^p \hank_d(z_j) \Psi^w_j \notag\\
&=& \hank_{d|p}\left(Z\right)  \Psi^w \label{eq:convw} = Z \circledast \overline \Psi^w
\end{eqnarray}
where
\begin{eqnarray}
\overline{\Psi}^w = \begin{bmatrix} w_1\overline\Psi_1 \\ \vdots \\ w_p \overline\Psi_p \end{bmatrix}  \ .
\end{eqnarray}
%which is equivalent to the MIMO convolution with  $\overline\Psi^w$ filter kernel.
The aforementioned matrix vector operations  using the  extended Hankel matrix  also describe the  filtering operation \eqref{eq:2dConv} in 2-D CNNs  as shown in  Fig.~\ref{fig:cnnConv}.

 % for the case of 2-D CNN). % as shown in Fig.~\ref{fig:conv}(a)-(d).
%
%
%Similar to \eqref{eq:convw}, the multi-dimensional convolution in an image domain CNN can be 
%also equivalently described by 2-D MIMO convolution in \eqref{eq:Yc} with
%weighted multi-channel 2-D filters as shown in Fig.~\ref{fig:cnnConv}.

Throughout the paper, we denote the space of the wrap-around Hankel structure matrices  of the form in \eqref{eq:hank} as $\Hc(n,d)$,
and an extended Hankel matrix composed of $p$ Hankel matrices of the form in \eqref{eq:ehank} as $\Hc(n,d;p)$.
The basic properties of Hankel matrix used in this paper are described in Lemma~\ref{lem:calculus} in Appendix~\ref{ap1}.
In the next section, we describe advanced properties of the Hankel matrix that will be extensively used in this paper.

\subsection{Low-rank property  of Hankel Matrices}

\begin{figure}[!bt]
\centering
\includegraphics[width=10cm]{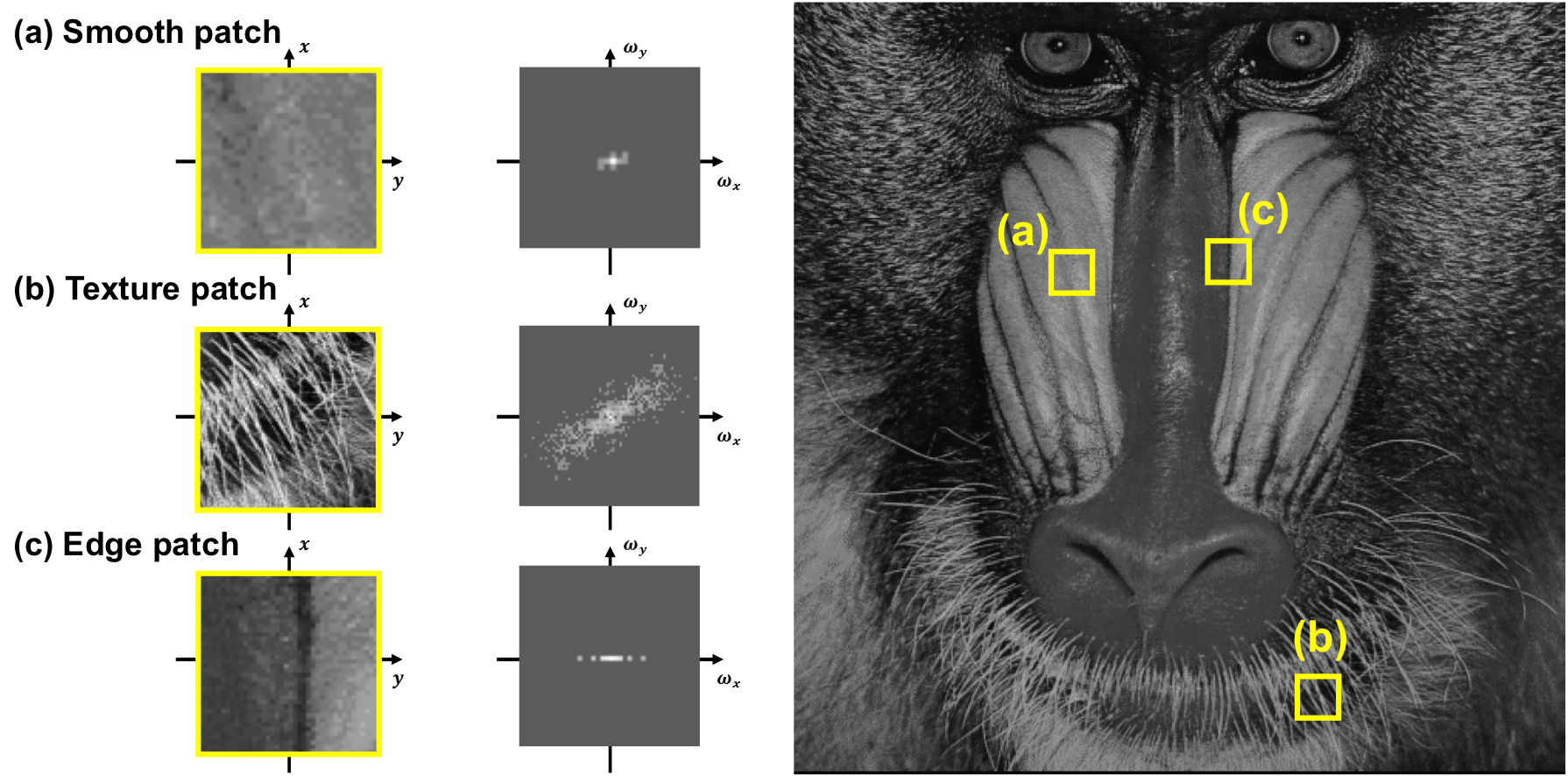}
\caption{Spectral components of patches from  (a) smooth background,  (b) texture, and (c) edge.}%The sampled spectrum of a patch can be represented in (c).}
\label{fig:flowchart}
\end{figure}

One of the most intriguing features of the Hankel matrix is that it often has a low-rank structure and its low-rankness is related to the sparsity in the Fourier domain (for the case of Fourier samples,  it is related to the sparsity in the spatial domain)\cite{ye2016compressive,jin2015annihilating}.

Note that many types of image patches have sparsely distributed Fourier spectra.
For example,   as shown in Fig. \ref{fig:flowchart}(a),  a smoothly varying patch usually has spectrum content in the low-frequency regions, while the other frequency regions have very few spectral components.
Similar spectral domain sparsity can be observed in the texture patch shown in  Fig. \ref{fig:flowchart}(b), where the spectral components of patch are determined by the spectrum
of the  patterns.
For the case of an abrupt transition along the edge as shown in Fig. \ref{fig:flowchart}(c), the spectral components are mostly localized along the $\omega_x$ axis. 
%One of the most intriguing observation in \cite{jin2015annihilating} is that
%the spectral domain sparsity produces a low-rank Hankel matrix in the signal domain.
In these cases,  if we construct a Hankel matrix using the corresponding image patch,  the resulting Hankel matrix is low-ranked \cite{ye2016compressive}.
This property is  extremely useful   as demonstrated by many applications \cite{jin2015annihilating,jin2016general,ongie2016off,lee2016acceleration,lee2016reference,jin2016mri}.
For example,
this idea can be used for image denoising \cite{jin2015sparse+} %artifact removal \cite{jin2016mri} 
and deconvolution \cite{min2015fast} by
modeling the underlying intact signals to have low-rank Hankel structure, from which the artifacts or blur components can be easily removed.

In order to understand this intriguing relationship,  consider a 1-D signal, whose spectrum  in the Fourier domain is sparse and can
be modelled as the sum of Diracs:
\begin{equation}\label{eq:signal3}
\hat f(\omega) = 2\pi \sum_{j=0}^{r-1} c_{j} \delta \left( \omega- \omega_j \right) \, \quad \omega_j \in [0, 2\pi] ,
\end{equation}
where $\{\omega_j\}_{j=0}^{r-1}$  refer to the corresponding harmonic components in the Fourier domain.
Then, the corresponding discrete time-domain signal is  given by:
 \begin{eqnarray}\label{eq:fs}
f[k] = \sum_{j=0}^{r-1} c_{j} e^{-i k \omega_j } \ .
\end{eqnarray}
Suppose that  we have a $r+1$-length  filter $h[k]$ which has the following z-transform representation  \cite{vetterli2002sampling}: 
\begin{eqnarray}\label{eq:afilter}
\hat h(z)  &=& \sum_{l=0}^r  h[l] z^{-l} = \prod_{j=0}^{r-1} (1- e^{-i\omega_j} z^{-1}) \ .
\end{eqnarray}
Then, it is easy to see that
\begin{eqnarray}\label{eq:annf}
( f\circledast h ) [k]=0,\quad \forall k,
\end{eqnarray}
 because
\begin{eqnarray}
( h\ast  f)[k] &=& \sum_{l=0}^r h[l] f[k-l] \nonumber \\
&=& \sum_{l=0}^r \sum_{j=0}^{r-1} c_{j}  h[l]u_j^{k-l}  \nonumber \\
&=& \sum_{j=0}^{r-1}c_j\underbrace{\left( \sum_{l=0}^r  h[p]u_j^{-l} \right)}_{\hat h(u_j)}u_j^k = 0 \label{eq:fri}
\end{eqnarray}
where $u_j = e^{-i\omega_j }$ and the last equality comes from \eqref{eq:afilter} \cite{vetterli2002sampling}. % \cite{vetterli2002sampling,dragotti2007sampling,maravic2005sampling}.
Thus, the filter $h$ annihilates the signal $f$, so it is referred to as the {\em annihilating filter}.
Moreover, using the notation in \eqref{eq:SISO},  
 Eq.~\eqref{eq:annf} can be represented by
 $$\hank_d(f) \overline h =0 \quad \ . $$
% if $r< d$.
% when $d>r$
This implies that Hankel matrix $\hank_d(f)$ is rank-deficient.   In fact, the rank of the Hankel matrix
can be explicitly calculated as shown in the following theorem:
%of the Hankel matrix 
\begin{theorem}\label{thm:hrank}\cite{ye2016compressive}
Let $r+1$ denote the minimum length of  annihilating filters that annihilates the signal $f=[f[1],\cdots, f[n]]^T$. %Suppose, furthermore,   $d$ is given by $d \geq r+1$. 
%Assume that $\min\{n-d+1,d\}> r$.
Then,
for a given Hankel structured matrix $\hank_d(f) \in \Hc(n,d)$ with $d>r$, we have
\begin{eqnarray}\label{eq:rankr}
\rank \hank_d(f)=  r, %\min\{r, n-d+1\},  % \quad\leq \min\{n-d+1, d\} \ ,
\end{eqnarray} 
where $\rank(\cdot)$ denotes a matrix rank.
\end{theorem}
Thus, if we choose a sufficiently large $d$,  the resulting Hankel matrix is low-ranked.
This relationship is quite general, and Ye et al \cite{ye2016compressive} further showed that 
% 
%Another important property  exploited in this paper  is  the low-rank property 
%of Hankel matrix \cite{ye2016compressive}.
%Specifically, Ye et al \cite{ye2016compressive} showed that
% for a given signal $f[k], k=1,\cdots,n$, 
the rank of the associated Hankel matrix $\hank_d(f)$  is $r$ if and only if
$f$ can be represented by  
\begin{equation}\label{eq:signalFRI}
f [k]=\sum\limits_{j=0}^{p-1}\sum\limits_{l=0}^{m_j-1}c_{j,l}k^l{\lambda_j}^k~, \quad \mbox{where}~ \quad r=\sum\limits_{j=0}^{p-1}m_j ~ < d
\end{equation}
for some $|\lambda_j|\leq 1,~j=1,\cdots, m_j$.
If  $\lambda_j = e^{-i\omega_j}$,  then it is directly
related to the signals with the finite rate of innovations (FRI)  \cite{vetterli2002sampling}.
Thus, the low-rank Hankel matrix provides an important link between  FRI sampling theory and compressed sensing such
that a sparse recovery problem can be solved using the measurement domain low-rank interpolation \cite{ye2016compressive}.

In \cite{jin2016general},  we also showed that the rank of the extended Hankel matrix in \eqref{eq:ehank}  is low,
when the multiple signals $Z=[z_1,\cdots z_p]$ 
has the following structure:
\beq
\hat z_i =  f  \circledast \overline h_i, \quad i=1,\cdots, p
\eeq 
such that the Hankel matrix $\hank_{d}(z_i)$  has the following decomposition:
\beq
\hank_{d}(z_i)  = \hank_n(f) \circul_d(h_i) \quad \in \Cd^{n\times d}
\eeq
where $\hank_n(f)$ is $n\times n$  wrap-around Hankel matrix,  and $ \circul_d(h)$  for any $h\in \Rd^m$ with $m\leq n$ is defined by
  \begin{eqnarray}\label{eq:circul}
\circul_d(h ) &=& \overbrace{\left[
        \begin{array}{ccc}
        h [1]  &   \cdots   &  0    \\
        \vdots &   \ddots  &  0 \\
         h [m]    &   \ddots  & h [1]  \\ 
%         \psi[d]   &  \psi ^{(1)}[d-1] & \cdots &   \ddots &  0 \\
          0 &   \ddots&   \vdots \\
         \vdots    & \vdots       &   h [m] \\
         \vdots & \vdots & \vdots \\
                   0  &  0 &   0 \\
        \end{array}
    \right]}^{d} \in \Cd^{n\times d}  \  .
    \end{eqnarray}
%  \begin{eqnarray*}% \label{eq:U}
%\Cc(h_i) &=& \left[
%        \begin{array}{cccc}
%        h_i[1]  &   h_i[2] & \cdots   &   h_i[d]   \\%  &\cdots&   \yb_i(1)\\
%       h_i[n]  &   h_i[1] & \cdots &     h_i[d-1] \\
%         \vdots    & \vdots     &  \ddots    & \vdots    \\
%                     h_i[2]  &   h_i[3] & \cdots &   h_i[d+1] \\
%        \end{array}
%    \right] \in \Cd^{n\times d}  \  .
%    \end{eqnarray*}
Accordingly, the extended Hankel matrix $\hank_{d|p}(Z)$
% , for a given %and we have 
% horizontally agumented matrix $\Yc_h$:
%\begin{eqnarray}%\label{eq:Ych}
%\Yc_h = \begin{bmatrix}  \hank_c(\hat \gb_1) & \cdots  & \hank_c(\hat \gb_{N_c}) \end{bmatrix}   \in \Cd^{ n \times (N_cd)}\  .
%\end{eqnarray}
%with $\hank(\hat \gb_i) \in \Cd^{n\times d}$,
%we 
has the following decomposition:
\begin{eqnarray}\label{eq:Ydecomp}
\hank_{d|p}(Z) =  \hank_n(f)\begin{bmatrix} \circul_d(h_1) & \cdots & \circul_d(h_p) \end{bmatrix}.
\end{eqnarray}
Due to the rank inequality $\rank(AB) \leq\min\{\rank(A),\rank(B) \}$,
we therefore have the following rank bound:
\begin{eqnarray}\label{eq:rankYch}
\rank \hank_{d|p}(Z)  &\leq&  \min\{ \rank   \hank_n(f), \rank \begin{bmatrix} \circul(h_1) & \cdots & \circul(h_p) \end{bmatrix}\} \notag\\
&=&  \min\{r, pd\}  \quad . %< \min \{ n-d+1, N_c d\}
\end{eqnarray}
Therefore, if the filter length $d$ is chosen such that the number of column of the extended matrix is sufficiently large, i.e. $pd > r$,  then the concatenated matrix becomes low-ranked.

% We will show  later that the low-rankness of extended Hankel matrix is important for the perfect reconstruction 
%condition for deep convolutional framelet expansion.
%However, the main disadvantages of the existing low rank Hankel matrix approaches \cite{jin2015annihilating,jin2016general,ongie2016off,lee2016acceleration,lee2016reference,jin2016mri} is the computational complexity due to the matrix factorization in each iterative step.
%
%
Note that the  low-rank Hankel matrix algorithms are usually performed in a  patch-by-patch manner \cite{jin2015sparse+,jin2015annihilating}.
It is also remarkable that this is similar to the current practice of deep CNN for low level computer vision applications,
where the network input  is usually given as a patch.
Later, we will show that this is not a coincidence; rather it suggests an important link between the low-rank Hankel matrix approach and a CNN.

% such as denosing \cite{jin2015sparse+} or inpainting \cite{jin2015annihilating}, whereas
%the global k-space data processing is used for MR applications \cite{jin2016general,ongie2016off,lee2016acceleration,lee2016reference}.
% while the whole k-space processing is often used for deep learning-based
%MR reconstruction \cite{hammernik2016learning}.

\subsection{Hankel matrix decomposition and  the convolution framelets}

The last but not  least important property of Hankel matrix   is that a Hankel matrix decomposition results in a framelet representation
whose bases are constructed by the convolution  of so-called local and non-local bases \cite{yin2017tale}.
%
%Unlike the pooling and unpooling, the spectral basis still guarantees the perfect reconstruction even though the corresponding
%local or non-local transforms perform dimension reduction.  %More specfically,
More specifically, for a given input vector $f\in \Rd^{n}$,
suppose that the Hankel matrix $\hank_{d}(f)$  with the rank $r< d$ has the following singular value decomposition:
\begin{eqnarray}\label{eq:svd0}
\hank_{d}(f) = U \Sigma V^{\top}
\end{eqnarray}
where $U =[u_1 \cdots u_r] \in \Rd^{n \times r}$ and $V=[v_1\cdots v_r]\in \Rd^{d\times r}$  denote the left and right singular vector bases matrices, respectively; and $\Sigma \in \Rd^{r\times r}$ is the diagonal
matrix whose diagonal components contains the singular values.
Then,  by multiplying $U^\top$ and $V$ to the left and right of the Hankel matrix, we have
\begin{eqnarray}\label{eq:Sigma}
\Sigma &=& U^\top\hank_{d}(f) V \  .
\end{eqnarray}
Note that the $(i,j)$-th element of $\Sigma$ is given by
\begin{eqnarray}\label{eq:sigma}
\sigma_{ij} = u_i^\top \hank_{d}(f)v_j  =\langle f,  u_i \circledast v_j \rangle,\quad 1\leq i,j \leq r    \  ,
\end{eqnarray}
where the last equality comes from \eqref{eq:inner}.
Since the number of rows and columns of $\hank_d(f)$ are $n$ and $d$,
 the right-multiplied vector $v_j$   interacts {locally} with  the $d$ neighborhood of the $f$ vector, 
whereas the left-multiplied vector $u_i$ has a  global interaction with the entire $n$-elements of  the $f$ vector.
Accordingly, \eqref{eq:sigma} represents the strength of simultaneous global and local interaction of the signal $f$ with bases. % with the  $u_i$-th global and  $j$-th local bases.
Thus,  we call $u_i$ and $v_j$ as {\em non-local} and {\em local} bases, respectively.

This relation holds for arbitrary  bases matrix $\Phi =[\phi_1,\cdots,\phi_m] \in \Rd^{n\times n}$ and $\Psi=[\psi_1,\cdots, \psi_d] \in \Rd^{d\times d}$ that are multiplied to the left and right
of the Hankel matrix, respectively, to yield the coefficient matrix:
\begin{eqnarray}\label{eq:C0}
c_{ij} =  \phi_i^\top\hank_{d}(f)  \psi_j  = \langle f,  \phi_i\circledast \psi_j\rangle, \quad i=1,\cdots, n, ~j=1,\cdots, d,
\end{eqnarray}
which represents the interaction of $f$ with the non-local basis $\phi_i$ and local basis $\psi_j$.
Using \eqref{eq:C0} as expansion coefficients,  %for a given $n$-dimensional vector $f \in \Rd^n$, let the image patch  lifted to Hankel structured matrix be denoted by
%$\hank_d(f)  \in \Hc(n,d).$
%Suppose, furthermore, that $\Phi$ and  $\Psi$ denote the two orthonormal matrices of dimension $n\times n$ and $d\times d$, respectively, where $\phi_i$ and $\psi_j$ refer to 
% the columns of $\Phi$ and $\Psi$, respectively.  %, accordingly, where $1\leq i\leq n$ and $1\leq j\leq d$.
%Then, 
Yin et al derived  the following signal expansion, which they called the {\em convolution framelet} expansion \cite{yin2017tale}:
\begin{proposition}[\cite{yin2017tale}]\label{prp:yin}
Let  $\phi_i$ and $\psi_j$ denotes the $i$-th and $j$-th columns of orthonormal matrix $\Phi \in \Rd^{n\times n}$ and $\Psi \in \Rd^{d\times d}$, respectively.
Then, for any $n$-dimensional vector $f \in \Rd^n$, 
\begin{eqnarray}\label{eq:frame0}
f  %&=& \hank_d^\dag(F)  =  \hank_d^\dag(\Phi\Phi^\top F\Psi\Psi^\top) \notag\\
&=&  \frac{1}{d} \sum_{i=1}^{n}\sum_{j=1}^d\langle f,  \phi_i \circledast \psi_j \rangle   \phi_i \circledast \psi_j
\end{eqnarray}
Furthermore, $ \phi_i \circledast \psi_j$ with $ i=1,\cdots, n; j=1,\cdots, d$ form a tight frame  for $\Rd^n$ with the frame constant $d$.
\end{proposition}
This implies that  any input signal $f\in \Rd^n$ can be expanded using the
convolution frame $\phi_i\circledast \psi_j$ and the expansion coefficient
$\langle f,  \phi_i \circledast \psi_j \rangle $.
Although the framelet coefficient matrix $[c_{ij}]$ in \eqref{eq:C0}  for general non-local and local bases is not as sparse as  \eqref{eq:Sigma}  from SVD bases,  Yin et al \cite{yin2017tale} showed that the framelet coefficients
can be made sufficiently sparse by optimally learning $\Psi$ for a given non-local basis $\Phi$.
Therefore, the choice of the non-local bases is one of the key factors in determining the efficiency of the framelet expansion.
In the following,  several examples of non-local bases $\Phi$ in \cite{yin2017tale} are discussed.
\begin{itemize}
\item SVD: From the singular value decomposition in \eqref{eq:svd0},  the SVD basis is constructed by augmenting the left singular vector basis $U \in \Rd^{n\times r}$ with an orthogonal
matrix $U_{ext}\in \Rd^{n\times (n-r)}$:
$$\Phi_{SVD} = \begin{bmatrix} U & U_{ext} \end{bmatrix}$$
such that $\Phi_{SVD}^\top \Phi_{SVD}=I$. 
%We often refer the SVD basis to the {\em spectral basis}.
Thanks to \eqref{eq:sigma}, this is the most energy compacting basis. 
However, the  SVD  basis is input-signal dependent and the calculation of the
SVD is computationally expensive.
\item Haar: 
Haar basis comes from the Haar wavelet transform and is constructed as follows:
$$\Phi = \begin{bmatrix} \Phi_{low}  & \Phi_{high} \end{bmatrix} \ ,$$
where 
the low-pass and high-pass operators $\Phi_{low},\Phi_{high} \in \Rd^{n\times \frac{n}{2}}$  are defined by
$$\Phi_{low}= \frac{1}{\sqrt{2}} \begin{bmatrix} 1 & 0 & \cdots & 0 \\ 1 & 0 & \cdots & 0 \\  0 & 1 & \cdots & 0 \\ 
 0 & 1 & \cdots & 0 \\  \vdots & \vdots & \ddots & \vdots \\ 0 & 0 & \cdots & 1  \\ 0 & 0 & \vdots & 1 \end{bmatrix} ,\quad
 \Phi_{high}= \frac{1}{\sqrt{2}} \begin{bmatrix} 1 & 0 & \cdots & 0 \\ -1 & 0 & \cdots & 0 \\  0 & 1 & \cdots & 0 \\ 
 0 & -1 & \cdots & 0 \\  \vdots & \vdots & \ddots & \vdots \\ 0 & 0 & \cdots & 1  \\ 0 & 0 & \vdots & -1 \end{bmatrix}
  $$
Note that the non-zero elements of each column of Haar basis is two, so one level of Haar decomposition does not represent
a global interaction.  However, by cascading the Haar basis, the interaction becomes global, resulting
in a multi-resolution decomposition of the input signal.
Moreover,  Haar basis is a useful global basis because it can sparsify the piecewise constant signals. 
Later, we will show that the average pooling operation is closely related to the Haar basis.
\item DCT: The discrete cosine transform (DCT) basis is an interesting global basis proposed by Yin et al \cite{yin2017tale} due to its energy
compaction property proven by JPEG image compression standard. The DCT bases matrix is a fully populated dense matrix,
which clearly represents a global interaction.  To the best of our knowledge, the DCT basis have never been used in deep CNN, which could be an interesting
direction of research.
%\item Random: The random basis is constructed as a dense matrix with random numbers.  This was used as a baseline model
%for comparison \cite{yin2017tale}.
\end{itemize}
In addition to the non-local bases used in \cite{yin2017tale},   we will also investigate the following non-local bases:
\begin{itemize}
\item Identity matrix:  In this case, $\Phi=I_{n\times n}$, so there is no global interaction between the basis and the signal.  Interestingly,
this non-local basis is quite often used in CNNs that do not have a pooling layer.
In this case, it is believed that the local structure of the signal is  more important
and local-bases  are trained such that they can maximally capture the local correlation
structure of the signal.
%\item Contourlet transform:  Here, we consider non-subsampled contourlet transform  \cite{?} as a non-local
%basis. Contourlet transform is composed of two steps:  a non-decimated multi-scale decomposition to produce highpass and lowpass subbands, which is followed by
% a directional decomposition to divide the highpass subband into directional components. The $k$-th level filter banks are generated by iterating these two steps. As there is no down-sampling or up-sampling in the filter banks, it is a shift invariant and redundant bases.
\item Learned basis: In extreme case where we do not have specific knowledge of the signal, the non-local bases can be also learnt.
However, a care must be taken, since the learned non-local basis has size of $n\times n$ that quickly becomes very large for image processing applications.
For example, if one is interested in
processing  $512\times 512$ (i.e. $n=2^{9}\times 2^9$) image, the required memory to store the learnable non-local basis becomes $2^{37}$,
which is not possible to store or estimate. However, if the input patch size is sufficiently small, this may be another interesting 
direction of research in deep CNN.
\end{itemize}

\section{Main Contributions: Deep Convolutional Framelets Neural Networks}

In this section, which is our main theoretical contribution,  we will show that the convolution framelets by Yin et al \cite{yin2017tale} is directly related to the deep neural network
if we relax the condition of the original convolution framelets  to allow
 multilayer implementation. The multi-layer extension of convolution framelets, which we call the {\em deep convolutional framelet},
 can explain many important components of deep learning.
% : multi-channel convolutional filters, number of network layers,
% and advanced network components such as residual block \cite{he2016deep}.

\subsection{Deep Convolutional Framelet Expansion}

While the original convolution framelets by Yin et al \cite{yin2017tale}  exploits the  advantages of the low rank Hankel matrix approaches using two bases,
there are several limitations.
First, their convolution framelet uses only orthonormal basis. 
Second,
the  significance of multi-layer implementation  was not noticed.
Here, we 
discuss its extension to relax
 these limitations. 
 As will become clear, this
 is a basic building step toward a  deep convolutional framelets neural network.

\begin{proposition}\label{prp:1}
Let $\Phi  = [\phi_1,\cdots, \phi_m] \in \Rd^{n\times m}$ and $\Psi  = [\psi_1,\cdots, \psi_q] \in \Rd^{d\times q}$ denote the
non-local and local bases matrices, respectively.
Suppose, furthermore, that $\tilde\Phi = [\tilde\phi_1,\cdots, \tilde\phi_m] \in \Rd^{n\times m}$ and $\tilde\Psi  = [\tilde\psi_1,\cdots,\tilde\psi_q] \in \Rd^{d\times q}$ denote
their {\em dual bases} matrices  such that they satisfy the frame condition:
\begin{eqnarray}
%\sum_{k=1}^p T^{(k)\top}T^{(k)} &=& I_{n\times n},~ \label{eq:T} \\
 \tilde \Phi \Phi^\top &=& \sum_{i=1}^m \tilde \phi_i \phi_i^\top = I_{n\times n},~ \label{eq:phi0} \\
  \Psi \tilde \Psi^{\top} &=&  \sum_{j=1}^q   \psi_j\tilde \psi_j^{\top} = I_{d\times d} \ .  \label{eq:ri0}
 \end{eqnarray}
%for $k=1,\cdots, p$,
Then, for any input signal $f \in \Rd^n$, we have
\begin{eqnarray}\label{eq:frame}
f &=& \frac{1}{d} \sum_{i=1}^{m}\sum_{j=1}^q  \langle f,  \phi_i \circledast \psi_j \rangle   \tilde \phi_i \circledast \tilde \psi_j   \  , %\notag \\
%&=& \sum_{k=1}^q T^{(k)\top} f_k
\end{eqnarray}
or equivalently,
\begin{eqnarray}\label{eq:frameeq}
f &=& \frac{1}{d} \sum_{i=1}^{m} \left( \tilde \Phi c_j\right) \circledast \tilde \psi_j   \  ,
\end{eqnarray}
where $c_j$ is the $j$-th column of the framelet coefficient matrix
\begin{eqnarray}
C &=& \Phi^\top \left( f \circledast  \overline\Psi\right) \label{eq:coef} \\
&=& \begin{bmatrix} \langle f,  \phi_1 \circledast \psi_1 \rangle & \cdots &  \langle f,  \phi_1 \circledast \psi_q  \rangle \\
\vdots & \ddots & \vdots \\
 \langle f,  \phi_m \circledast \psi_1 \rangle & \cdots &  \langle f,  \phi_m \circledast \psi_q  \rangle\end{bmatrix} \in \Rd^{m\times q} \quad \ .
 \end{eqnarray}
%where
%\begin{eqnarray}\label{eq:fk}
%f_k:= \frac{1}{d} \sum_{j=1}^d\sum_{k=1}^m  \langle  T^{(k)} f,  \phi_i \circledast\psi_j^{(k)} \rangle   \tilde \phi_i \circledast \tilde \psi_j^{(k)} = T^{(k)}f,~
%\end{eqnarray}
\end{proposition}
\begin{proof}
Using the frame condition \eqref{eq:phi0} and \eqref{eq:ri0},  we have % can be represented by these biorthonormal bases:
\begin{eqnarray*}
 \hank_d(f)  %&=& \sum_{i=1}^{n}\sum_{j=1}^d \langle \Psi_{ij}, \sum_{k} T^{(k)\top}T^{(k)} F \rangle \tilde \Psi_{ij} \\
 &=&  \tilde \Phi \Phi^\top  \hank_d(f)   \Psi\tilde \Psi^{\top} 
 = \tilde \Phi C \tilde\Psi^{\top}  \  ,
% &=& \sum_{i=1}^{m}\sum_{j=1}^q   \tilde \phi_i \phi_i^\top F  \psi_j\tilde \psi_j^{\top} \\ 
%&=&\sum_{i=1}^{m}\sum_{j=1}^q  \langle  \Psi_{ij},   F \rangle \tilde \Psi_{ij}
\end{eqnarray*}
where $C \in \Rd^{m\times q}$ denotes the framelet coefficient matrix  computed by
$$C = \Phi^\top  \hank_d(f)   \Psi = \Phi^\top \left( f \circledast \overline \Psi\right) $$
and its $(i,j)$-th element is given by
%
% $\Psi_{ij} := \phi_i \psi_j^{\top}$ and $\tilde\Psi_{ij} =\tilde\phi_{i} \tilde\psi_{j}^{\top}$
%and we use the following identity:
\begin{eqnarray*}
c_{ij} =  \phi_i^{\top}\hank_d(f)   \psi_j  &=  \langle f, \phi_i \circledast \psi_j \rangle
%  \tr\left( \Psi_{ij}^{\top} F \right) = \langle \Psi_{ij}, F \rangle    % &= \langle f, \phi_i \circledast \psi_j^{(k)} \rangle \ .
\end{eqnarray*}
where we use  \eqref{eq:inner} for the last equality.
Furthermore, using \eqref{eq:recon1} and \eqref{eq:invfilter}, we have
\begin{eqnarray*}
f = \hank_d^\dag\left( \hank_d(f) \right) &=& \hank_d^\dag\left( \tilde \Phi C \tilde\Psi^{\top} \right) \\
&=& \frac{1}{d}\sum_{j=1}^q  \left(\tilde\Phi c_j\right) \circledast \tilde \psi_j \\
&=& \sum_{i=1}^m\sum_{j=1}^q \langle f, \phi_i \circledast \psi_j \rangle \tilde \phi_i \circledast 
\tilde \psi_j
\end{eqnarray*}
This concludes the proof.
\end{proof}

%If we construct  matrices $\Phi, \tilde \Phi \in \Rd^{n\times m}$ and  $\Psi, \tilde \Psi \in \Rd^{d\times q}$ such that their $i$-th column
%corresponds to  $\phi_i, \tilde \phi_i$ and  $\psi_i, \tilde \psi_i$, respectively,
Note that
%$$\Phi = \begin{bmatrix} \phi_1 & \cdots & \phi_n \end{bmatrix}$ 
%
%Using the generalized inverse of Hankel operator defined by \eqref{eq:F2f}, 
 the so-called perfect recovery condition (PR) represented by \eqref{eq:frame} can be equivalently studied using:
 \begin{eqnarray}\label{eq:PR}
f =  \hank_d^{\dag} \left( \tilde\Phi ( \Phi^{\top}\hank_d(f)  \Psi) \tilde\Psi^{\top} \right) \ .
\end{eqnarray}
Similarly, for a given matrix input $Z\in \Rd^{n\times p}$, the perfect reconstruction condition for a matrix input $Z$ can be given by
 \begin{eqnarray}\label{eq:PRZ}
Z =  \hank_{d|p}^{\dag} \left( \tilde\Phi ( \Phi^{\top}\hank_{d|p}(Z)  \Psi) \tilde\Psi^{\top} \right) \ .
\end{eqnarray}
which is explicitly represented in the following proposition:
\begin{proposition}\label{prp:2}
Let  $\Phi,\tilde \Phi \in \Rd^{n\times m}$ denote the non-local basis and its dual,  and  $\Psi, \tilde \Psi \in \Rd^{pd\times q}$ denote
the local basis and its dual, respectively, which satisfy the frame condition:
\begin{eqnarray}
%\sum_{k=1}^p T^{(k)\top}T^{(k)} &=& I_{n\times n},~ \label{eq:T} \\
 \tilde \Phi \Phi^\top &=& \sum_{i=1}^m \tilde \phi_i \phi_i^\top = I_{n\times n},~ \label{eq:phi1} \\
  \Psi \tilde \Psi^{\top} &=&  \sum_{j=1}^q   \psi_j\tilde \psi_j^{\top} = I_{pd\times pd} \ .  \label{eq:ri1}
 \end{eqnarray}
Suppose, furthermore, that the local bases matrix have block structure:
\begin{eqnarray}\label{eq:blockPsi}
\Psi^\top = \begin{bmatrix}
\Psi_1^\top & \cdots & \Psi_p^\top \end{bmatrix} ,\quad \tilde\Psi^\top = \begin{bmatrix}
\tilde\Psi_1^\top & \cdots & \tilde\Psi_p^\top \end{bmatrix}
\end{eqnarray}
 with $\Psi_i, \tilde \Psi_i \in \Rd^{d\times q}, $
whose $j$-th column is represented by $\psi_{j}^i$ and $\tilde \psi_{j}^i$, respectively.
Then, for any  matrix $Z =[z_1\cdots z_p] \in \Rd^{n\times p}$, %under the following resolution of indentities \eqref{eq:phi0} and \eqref{eq:ri0},
we have
\begin{eqnarray}\label{eq:frameZ}
Z &=& \frac{1}{d} \sum_{i=1}^m\sum_{j=1}^q    \sum_{k=1}^p \begin{bmatrix}  \langle z_k, \phi_i \circledast \psi_j^k \rangle \tilde \phi_i \circledast \tilde \psi_j^1 & \cdots &    \langle z_k, \phi_i \circledast \psi_j^k \rangle \tilde \phi_i\circledast \tilde \psi_j^p
\end{bmatrix}
\end{eqnarray}
or equivalently,
\begin{eqnarray}
Z &=& \frac{1}{d} \begin{bmatrix} \sum_{j=1}^q   \left(\tilde \Phi c_j\right) \circledast \tilde \psi_j^1 & \cdots &  \sum_{j=1}^q  \left( \tilde \Phi c_j\right) \circledast \tilde \psi_j^p
\end{bmatrix} \label{eq:Znonlocal}
\end{eqnarray}
where $c_j$ is the $j$-th column of the framelet coefficient matrix
\begin{eqnarray}\label{eq:coef10}
C &=&  \Phi^\top \left( Z \circledast  \overline\Psi\right) \\
&=& \sum_{k=1}^p \begin{bmatrix}  \langle z_k,  \phi_1 \circledast \psi_1^k \rangle & \cdots &  \langle z_k,  \phi_1 \circledast \psi_q^k  \rangle \\
\vdots & \ddots & \vdots \\
 \langle z_k,  \phi_m \circledast \psi_1^k \rangle & \cdots &  \langle z_k,  \phi_m \circledast \psi_q^k  \rangle\end{bmatrix} \in \Rd^{m\times q} \quad \ . \notag
 \end{eqnarray}
\end{proposition}
\begin{proof}
For a given $Z\in \Rd^{n\times p}$,
using the frame condition \eqref{eq:phi0} and \eqref{eq:ri0},  we have % can be represented by these biorthonormal bases:
\begin{eqnarray*}
 \hank_{d|p}(Z)  %&=& \sum_{i=1}^{n}\sum_{j=1}^d \langle \Psi_{ij}, \sum_{k} T^{(k)\top}T^{(k)} F \rangle \tilde \Psi_{ij} \\
 &=&  \tilde \Phi \Phi^\top  \hank_{d|p}(Z)   \Psi \tilde \Psi^{\top} 
 = \tilde \Phi C \tilde\Psi^{\top}  \  .
% &=& \sum_{i=1}^{m}\sum_{j=1}^q   \tilde \phi_i \phi_i^\top F  \psi_j\tilde \psi_j^{\top} \\ 
%&=&\sum_{i=1}^{m}\sum_{j=1}^q  \langle  \Psi_{ij},   F \rangle \tilde \Psi_{ij}
\end{eqnarray*}
where $C \in \Rd^{m\times q}$ denotes the framelet coefficient matrix  computed by
$$C = \Phi^\top  \hank_{d|p}(Z)  \Psi = \Phi^\top \left( Z \circledast  \overline\Psi\right) $$
and its $(i,j)$-th element is given by
%
% $\Psi_{ij} := \phi_i \psi_j^{\top}$ and $\tilde\Psi_{ij} =\tilde\phi_{i} \tilde\psi_{j}^{\top}$
%and we use the following identity:
\begin{eqnarray*}
c_{ij} =  \phi_i^{\top}\hank_{d|p}(Z)  \psi_j  =  \sum_{k=1}^p \langle z_k, \phi_i \circledast \psi_j^k \rangle
%  \tr\left( \Psi_{ij}^{\top} F \right) = \langle \Psi_{ij}, F \rangle    % &= \langle f, \phi_i \circledast \psi_j^{(k)} \rangle \ .
\end{eqnarray*}
%where we use  \eqref{eq:inner} for the last equality.
Furthermore, using \eqref{eq:recon1}, \eqref{eq:invfilter} and \eqref{eq:recon2}, we have
\begin{eqnarray*}
Z = \hank_{d|p}^\dag\left( \hank_{d|p}(Z) \right) &=&
 \hank_{d|p}^\dag\left( \tilde \Phi C \tilde\Psi^{\top} \right) \\
 &=&  \begin{bmatrix} \hank_d^\dag\left( \tilde \Phi C \tilde\Psi_1^{\top} \right)  & \cdots &  \hank_d^\dag\left( \tilde \Phi C \tilde\Psi_p^{\top} \right)  \end{bmatrix} \\
&=&  \frac{1}{d} \begin{bmatrix} \sum_{j=1}^q   \left(\tilde \Phi c_j\right) \circledast \tilde \psi_j^1 & \cdots &  \sum_{j=1}^q  \left( \tilde \Phi c_j\right) \circledast \tilde \psi_j^p
\end{bmatrix} \\
&=&  \frac{1}{d} \sum_{i=1}^m\sum_{j=1}^q    \sum_{k=1}^p \begin{bmatrix}  \langle z_k, \phi_i \circledast \psi_j^k \rangle \tilde \phi_i \circledast \tilde \psi_j^1 & \cdots &    \langle z_k, \phi_i \circledast \psi_j^k \rangle \tilde \phi_i\circledast \tilde \psi_j^p
\end{bmatrix}
\end{eqnarray*}
This concludes the proof.
\end{proof}

\begin{remark}
Compared to Proposition~\ref{prp:yin},
Propositions~\ref{prp:1} and \ref{prp:2} are more general, since they consider the redundant and non-orthonormal
 non-local and local bases by allowing relaxed conditions, i.e.
 $m\geq n$ or $q\geq d$.  The specific reason for $q\geq d$ is to investigate existing CNNs that have  large number of filter channels at lower layers.
% , so we need to consider the redundant local filters (i.e. $q>d$) for our analysis.
%Moreover, there are recent proposals of using directional wavelet decomposition followed by CNN \cite{?}, which
The redundant global basis with $m\geq n$ is also believed to be useful  for future research, so Proposition~\ref{prp:1}
is derived by considering further extension. However, since 
most of the existing deep networks use the condition  $m=n$,  we will mainly focus on this special case for the rest of the paper.
% 
%satisfied if Eqs.~\eqref{eq:phi0} and \eqref{eq:ri0} are satisfied.
\end{remark}

\begin{remark}
For the given SVD  in \eqref{eq:svd0}, the frame conditions \eqref{eq:phi0} and \eqref{eq:ri0}  can be further relaxed to the following conditions:
\begin{eqnarray*}
%\sum_{k=1}^p T^{(k)\top}T^{(k)} &=& I_{n\times n},~ \label{eq:T} \\
 \tilde \Phi \Phi^\top =  P_{R(U)} &,\quad & 
  \Psi \tilde \Psi^{\top} =  P_{R(V)}
  \end{eqnarray*}
 due to the following matrix identity:
 $$\hank_d(f) =   P_{R(U)} \hank_d(f)   P_{R(V)} =   \tilde \Phi \left( \Phi^\top \hank_d(f)   \Psi\right) \tilde \Psi^{\top}   .$$
 In these case, the number of bases for non-local and local basis matrix can be smaller than that of Proposition~\ref{prp:1} and Proposition~\ref{prp:2},
 i.e. $m=r<n$ and $q=r< d$. Therefore,
 smaller number of bases still suffices for PR.
\end{remark}

%\subsection{The Deep Convolutional Framelets}

Finally, using Propositions~\ref{prp:1} and \ref{prp:2} we will show that the convolution framelet expansion can be realized by two matched convolution layers, which has
striking similarity to   neural networks with  encoder-decoder structure \cite{noh2015learning}.
Our main contribution is summarized in the following Theorem.
\begin{theorem}[\textbf{Deep Convolutional Framelets Expansion}]
Under the assumptions of Proposition~\ref{prp:2}, we have the following  decomposition of input $Z\in \Rd^{n\times p}$:
\begin{eqnarray}
Z &=& \left(\tilde \Phi C\right) \circledast \nu(\tilde \Psi)  \label{eq:dec}  \\
C &=& \Phi^{\top}\left(Z\circledast \overline\Psi\right) \ , \label{eq:enc}
\end{eqnarray} 
where the decoder-layer convolutional filter $\nu(\tilde\Psi)$ is defined by
\begin{eqnarray}\label{eq:tau}
\nu(\tilde\Psi) &:=&  \frac{1}{d} \begin{bmatrix}  \tilde \psi_1^1 & \cdots &  \tilde \psi_1^p  \\ \vdots & \ddots & \vdots \\
\tilde \psi_q^1 & \cdots &  \tilde \psi_q^p 
\end{bmatrix}  \in \Rd^{dq \times p}
\end{eqnarray}
Similarly, under the assumptions of Proposition~\ref{prp:1}, we have the following decomposition of $f\in \Rd^n$:
\begin{eqnarray}
f &=& \left(\tilde \Phi C\right) \circledast \nu(\tilde \Psi) \label{eq:decf}\\% &,\quad & \mbox{where}\quad
C &=& \Phi^{\top}\left(f\circledast \overline\Psi\right). \label{eq:encC}
\end{eqnarray}
where
\begin{eqnarray}\label{eq:tau1}
\nu(\tilde\Psi) &:=&  \frac{1}{d} \begin{bmatrix}  \tilde \psi_1 \\ \vdots \\
\tilde \psi_q
\end{bmatrix}  \in \Rd^{dq}
\end{eqnarray}
\end{theorem}
\begin{proof}
First, note that \eqref{eq:Znonlocal} corresponds to a decoder layer convolution:
\begin{eqnarray*}
Z %&=&  \hank_{d|p}^{\dag} \left( \tilde\Phi C \tilde\Psi^{\top} \right) \notag\\
&=&  \frac{1}{d} \begin{bmatrix} \sum_{j=1}^q   \left(\tilde \Phi c_j\right) \circledast \tilde \psi_j^1 & \cdots &  \sum_{j=1}^q  \left( \tilde \Phi c_j\right) \circledast \tilde \psi_j^p \end{bmatrix} \notag\\
&=& \left(\tilde \Phi C\right) \circledast \nu(\tilde \Psi)    \end{eqnarray*} 
where $\nu(\tilde \Psi)$ is defined by \eqref{eq:tau} and
the last equality comes from the definition of MIMO convolution in  \eqref{eq:MIMO} and \eqref{eq:MIMO_form}.
On the other hand,
from \eqref{eq:coef10}, we have
% the encoder layer computes the framelet coefficient matrix $C$ using the following convolution:
\begin{eqnarray*}
C %&=& \Phi^{\top}\hank_{d|p}(Z)  \Psi \notag\\
&= &\Phi^{\top}\left(Z\circledast \overline\Psi\right). %\label{eq:enc}
\end{eqnarray*}
Similarly,  using the defintion of the MISO convolution \eqref{eq:MISO},  Eq.~\eqref{eq:frameeq} can be represented by
$$f = \left(\tilde \Phi C\right) \circledast \nu(\tilde \Psi) $$
where $\nu(\tilde \Psi)$ is defined by \eqref{eq:tau1}.
Finally, \eqref{eq:encC} comes from \eqref{eq:coef}.
%Eqs.~\eqref{eq:dec} and \eqref{eq:enc} for a vector input $f$ can be represented by
%\begin{eqnarray}
%\\ %\label{eq:decf}\\% &,\quad & \mbox{where}\quad
%C &=& \Phi^{\top}\left(f\circledast \overline\Psi\right). %\label{eq:encC}
%\end{eqnarray}
This concludes the proof.
\end{proof}

\begin{remark}[{\bf Non-local basis as a generalized pooling/unpooling}]
Note that there exists a major difference in the encoder and decoder layer convolutions.
Aside from the difference in the specific convolutional filters,
 the non-local basis matrix $\Phi^\top$ should be applied later to the filtered signal in the case of
encoder \eqref{eq:enc}, whereas the non-local basis matrix  $\tilde\Phi$ should be multiplied first
before the local filtering is applied in the decoder layer.
This is in fact  similar to the pooling and unpooling operations because the pooling is performed after filtering   while
unpooling is applied before filtering. Hence, we can see that the non-local basis is a generalization of  the pooling/unpooling operations.
\end{remark}

 \begin{figure*}[!bt]
\centering
\includegraphics[width=12cm]{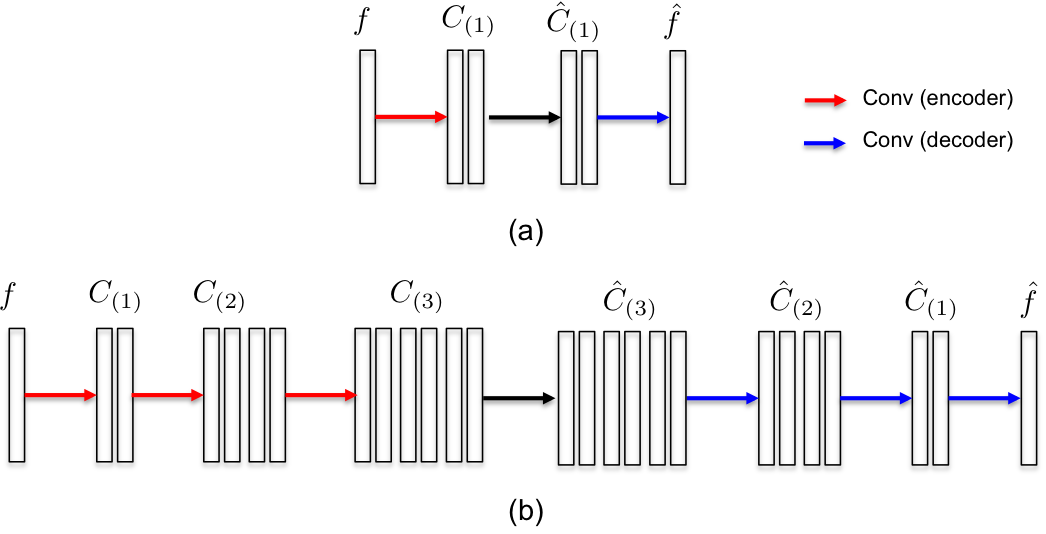}
\caption{(a) One layer encoder-decoder, and (b) multi-layer encoder-decoder architectures.}
\label{fig:ed}
\end{figure*}

Existing CNNs often incorporate the bias estimation for each layer of convolution. % Here, we investigate its role.
Accordingly, we are interested in extending our deep convolutional framelet expansion under bias estimation.
Specifically, with the bias estimation, the encoder and decoder convolutions  should be modified as:
\begin{eqnarray}
C &=& \Phi^\top(Z \circledast \overline\Psi  + 1_{n}b_{enc}^\top ) \label{eq:biasEnc}\\
\hat Z &=& \left(\tilde \Phi C\right) \circledast \nu(\tilde \Psi) +1_{n}b_{dec}^\top  \label{eq:biasDec}
\end{eqnarray}
where $1_n \in \Rd^n$ denotes the vector with 1, and $b_{enc}\in \Rd^q$ and $b_{dec}\in \Rd^p$ denote the  encoder and decoder layer biases, respectively.
Then,  the following theorem shows that there exists a unique bias vector $b_{dec} \in \Rd^p$  for
a given encoder bias vector $b_{enc}\in\Rd^q$ that satisfies the PR:

\begin{theorem}[\bf Matched Bias]\label{thm:bias}
Suppose that $\Phi, \tilde\Phi \in \Rd^{n\times m}$ and $\Psi, \tilde \Psi \in \Rd^{pd\times q}$ satisfies
the assumptions of Proposition~\ref{prp:2}.
%the frame conditions \eqref{eq:phi0} and \eqref{eq:ri0}, respectively. Moreover, assume that $\tilde \Psi$ has the block structure in \eqref{eq:blockPsi}.
Then, for a given bias
$b_{enc}\in \Rd^q$ at the encoder layer,  \eqref{eq:biasEnc} and \eqref{eq:biasDec} satisfy the PR if   the decoder bias is given by
\begin{eqnarray}\label{eq:optbias0}
1_{n}b_{dec}^\top  = -  \hank_{d|p}^\dag \left(   1_{n}b_{enc}^\top\tilde \Psi^\top\right),
\end{eqnarray}
or equivalently 
\begin{eqnarray}\label{eq:optbias}
b_{dec}[i] = - 1_d^\top \tilde \Psi_i b_{enc}, \quad i=1,\cdots, p.
\end{eqnarray}
\end{theorem}
\begin{proof}
See Appendix~\ref{ap8}.
\end{proof}

The simple convolutional framelet expansion using  \eqref{eq:encC}\eqref{eq:decf} and \eqref{eq:enc}\eqref{eq:dec} (or \eqref{eq:biasEnc} and \eqref{eq:biasDec} when including bias)
is so powerful
that a CNN with the encoder-decoder architecture emerges from them by inserting the encoder-decoder pair
 \eqref{eq:encC} and \eqref{eq:decf} between the encoder-decoder pair \eqref{eq:enc} and \eqref{eq:dec} 
 as illustrated by the red and blue lines, respectively, in  Fig.~\ref{fig:ed}(a)(b).
%The associated encoder layer corresponds to the red line in Fig.~\ref{fig:ed}(a).
 In general, the $L$-layer implementation of the convolutional framelets can be  recursively defined.
 For example, % we consider a multi-layer convolutional framelets where the convolutional framelets are connected one-by-one.
%Here,
the first layer encoder-decoder architecture without considering bias is given by
 \begin{eqnarray*}\label{eq:PR2}
%f %&=& \sum_{k=1}^m T^{(k)\top}T^{(k)}f \notag\\
%&=&   \hank_{d_{(1)}}^{\dag} \left(\tilde\Phi^{(1)} \left(\Phi^{(1)\top}  \hank_{d_{(1)}}(f)  \Psi^{(1)}\right)  \tilde\Psi^{(1)\top} \right) \\
f %&=& \sum_{k=1}^m T^{(k)\top}T^{(k)}f \notag\\
&=&   \hank_{d_{(1)}|p_{(1)}}^{\dag} \left(\tilde\Phi^{(1)}  \hat C^{(1)}  \tilde\Psi^{(1)\top} \right)  \notag\\
&=& \left( \tilde\Phi^{(1)}\hat C^{(1)} \right)\circledast \nu(\tilde \Psi^{(1)}) 
\end{eqnarray*}
where %$b_{dec}\in \Rd$ is the decoder bias,
%and  %the subscript and superscript index $(l)$ represents the $l$-th layer and
the decoder part of framelets coefficient  at the $i$-th layer,  $\hat C^{(i)}  \in \Rd^{n\times q_{(i)}}$, $q_{(i)}\geq d_{(i)}p_{(i)}$, is given by
\begin{eqnarray}
\hat C^{(i)}  &=&\begin{cases} \hank_{d_{(i+1)}|p_{(i+1)}}^\dag\left(\tilde\Phi^{(i+1)} \hat C^{(i+1)} \tilde \Psi^{(i+1)}\right) =
\left(\tilde\Phi^{(i+1)} \hat C^{(i+1)} \right) \circledast \nu (\tilde \Psi^{(i+1)}),  & 1\leq i <L \\
C^{(L)}, & i=L \end{cases}
\end{eqnarray}
whereas the encoder part framelet coefficients  $C^{(i)}  \in \Rd^{n\times q_{(i)}}$  are given by
\begin{eqnarray}\label{eq:Cenc}
C^{(i)}  &=&\begin{cases}  \Phi^{(i)\top}  \hank_{d_{(i)}|p_{(i)}}(C^{(i-1)} )  \Psi^{(i)} =  \Phi^{(i)\top} \left(C^{(i-1)} \circledast \overline\Psi^{(i)}\right),  & 1\leq i \leq L \\
f, & i=0 \end{cases}
\end{eqnarray}
Here,  $d_{(i)}$ and $p_{(i)}$ denotes the filter length and the number of input channels at the $i$-th layer, respectively,
and $q_{(i)}$ refers to the number of output channels. The specific number of channels will be analyzed in the following section.

\subsection{Properties of Deep Convolutional Framelets}

In this section, several important properties of deep convolutional framelets are explained in detail.
First, the perfect reconstruction (PR) conditions in \eqref{eq:encC}\eqref{eq:decf} and \eqref{eq:enc}\eqref{eq:dec}
 can be also analyzed in the Fourier domain as shown in the following Proposition:
\begin{proposition}[\textbf{Fourier Analysis of Filter Channels}]\label{prp:PRFourier}
Suppose that  $\Phi$ and $\tilde\Phi$  satisfy the frame condition \eqref{eq:phi0}.
Then, the PR condition given by \eqref{eq:enc} and \eqref{eq:dec} with
$\Psi^\top = \begin{bmatrix}\Psi_1^\top & \cdots & \Psi_p^\top\end{bmatrix}$ 
and
$\tilde\Psi^\top = \begin{bmatrix}\tilde\Psi_1^\top & \cdots & \tilde\Psi_p^\top\end{bmatrix}$ 
can be represented by
\begin{eqnarray}\label{eq:Id}
I_{p\times p} =  \frac{1}{d}  \begin{bmatrix} \widehat{\psi_1^1}^* & \cdots &  \widehat{\psi_q^1}^*  \\ \vdots & \ddots & \vdots \\
\widehat{ \psi_1^p}^* & \cdots &  \widehat{ \psi_q^p}^*
\end{bmatrix}  \begin{bmatrix}  \widehat{\tilde \psi_1^1} & \cdots &  \widehat{\tilde \psi_1^p}  \\ \vdots & \ddots & \vdots \\
\widehat{\tilde \psi_q^1} & \cdots &  \widehat{\tilde \psi_q^p }
\end{bmatrix}  
\end{eqnarray}
where $\widehat{\psi}$ denotes the Fourier transform of $\psi$, and the superscript $^*$ is the complex conjugate.
In particular, the PR condition given by \eqref{eq:encC} and \eqref{eq:decf} can be represented by
\begin{eqnarray}\label{eq:prdual}
\frac{1}{d} \sum_{i=1}^q \widehat{\psi_i}^*\widehat {\tilde \psi_i} = 1 . 
\end{eqnarray}
Furthermore, if the local basis $\Psi$ is orthormal, then we have
\begin{eqnarray}\label{eq:prorth}
\frac{1}{d} \sum_{i=1}^d |  \widehat{\psi_i}|^2  = 1 . 
\end{eqnarray}
\end{proposition}
\begin{proof}
See Appendix~\ref{ap2}.
\end{proof}

\begin{remark}
Note that   \eqref{eq:prorth} and  \eqref{eq:prdual} are equivalent to the perfect reconstruction conditions for the orthogonal and biorthogonal 
wavelet decompositions, respectively \cite{mallat1999wavelet,daubechies1992ten}. However, without the frame condition \eqref{eq:phi0} for the non-local bases, such simplification is not true. This is another reason why we are interested in imposing the frame condition \eqref{eq:phi0} for the non-local
bases in order
to use the connection to the classical wavelet theory \cite{mallat1999wavelet}. % to make it compatible to the wavelet and frame theory.
\end{remark}

The following proposition shows that a sufficient condition for fulfilling PR is that the number of input channels should increase multiplicatively along the layers:
\begin{proposition}[\textbf{Number of Filter Channels}]\label{prp:exp}
A sufficient condition to achieve PR is that  the number of output channel $q_{(l)}, l=1,\cdots, L$  %is given by
satisfies
\begin{eqnarray}\label{eq:prod}
q_{(l)} \geq  q_{(l-1)}d_{(l)},   
\quad  l=1,\cdots, L  , %\prod_{i=1}^l d_{(i)} .
\end{eqnarray} 
where $d_{(i)}$ is the filter length at the $i$-th layer and $q_{(0)}=1$.
\end{proposition}
\begin{proof} 
See Appendix~\ref{ap3}.
\end{proof}

An intuitive explanation of Proposition~\ref{prp:exp} is that
the cascaded application of two  multi-channel filter banks is equivalent to applying a combined filter bank generated by each combination
of two filters in each stage,  so the output dimension of the filter bank increases multiplicatively.
As a special case of Proposition~\ref{prp:exp},  we
can derive the following sufficient condition:
$$q_{(l)}=\prod_{i=1}^l d_{(i)}, \quad l=1,\cdots, L \quad ,$$
which is obtained by choosing the minimum number of output channels at each layer. %satisfying the frame condition \eqref{eq:ri0} and/or \eqref{eq:ri1}.
This implies that 
 the number of channels
increase exponentially with respect to layers  as shown in Fig.~\ref{fig:channel},  which is difficult to meet in practice due
to the memory requirement.
Then, a natural
question is why we still prefer a deep network to a shallow one.  Proposition~\ref{prp:depth} provides an answer to this question.

  \begin{figure}[!bt] 
\center{\includegraphics[width=6cm]{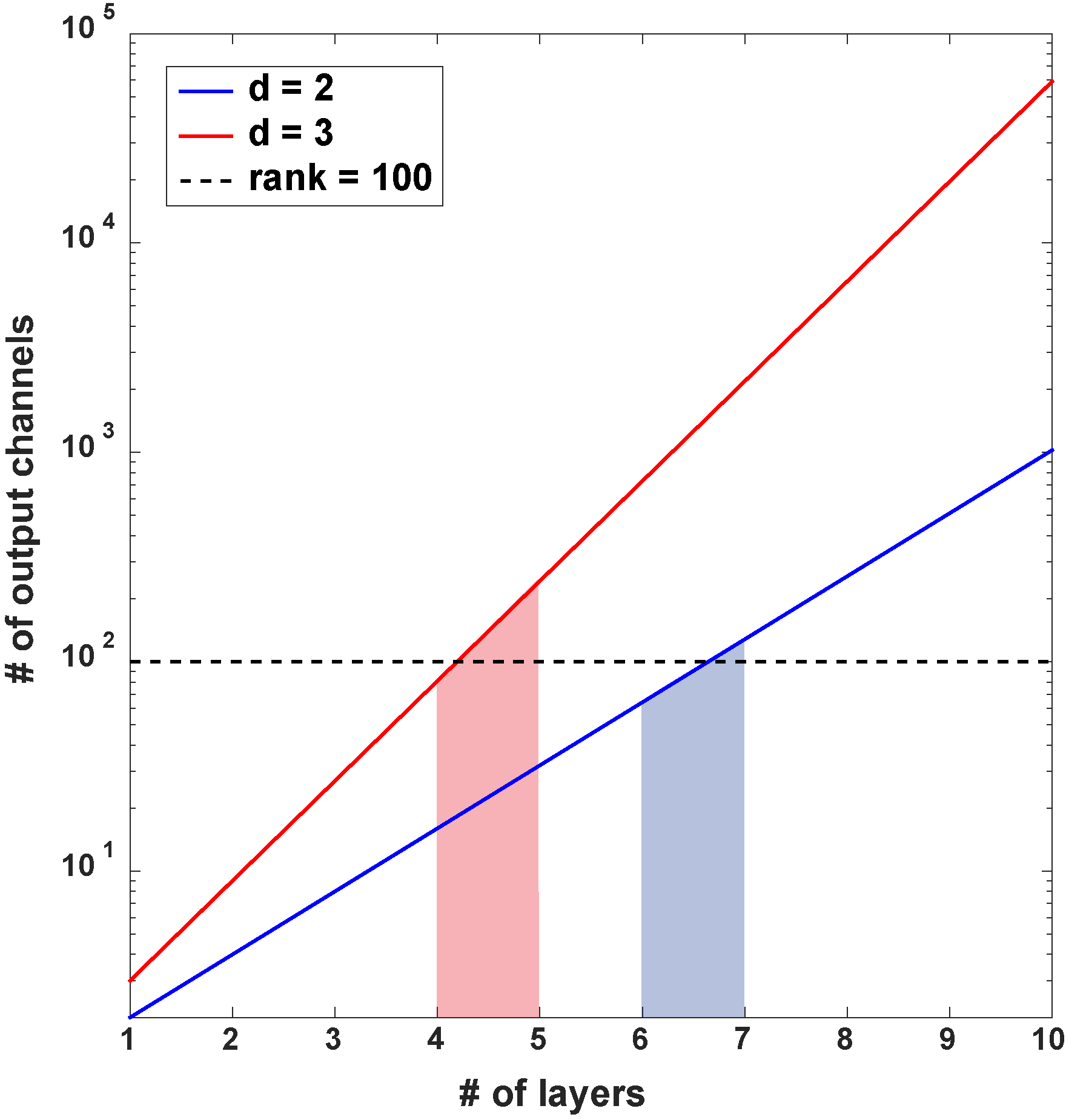}}
\caption{The exponential increase in the number of output channels.  Note that the $y$-axis is in the log-scale.
}
\label{fig:channel}
\end{figure}

\begin{proposition}[\textbf{Rank Bound of Hankel Matrix}]\label{prp:depth}
Suppose that $\Phi^{(l)} = I_{n\times n}, \forall l\geq 1$. Then,   the rank of the extended Hankel matrix $\hank_{d_{(l)}|p_{(l)}}(C^{(l-1)})$ in \eqref{eq:Cenc}
is upper-bounded by
\begin{eqnarray}\label{eq:rankdepth}
%\rank \hank_{n|p_{(l)}}(C^{(l-1)})
\rank \hank_{d_{(l)}|p_{(l)}}(C^{(l-1)}) & \leq &%\rank  \hank_{n|p_{(l-1)}}(C^{(l-2)})  \leq 
 \min \{ \rank  \hank_{n}(f), d_{(l)} p_{(l)} \}
%\rank \hank_{d_{(l)}|p_{(l)}}(C^{(l-1)})& \leq & \min \{ \rank  \hank_{n}(f), d_{(l)} p_{(l)} \}
\end{eqnarray}
\end{proposition}
\begin{proof}
See Appendix~\ref{ap4}.
\end{proof}

\begin{remark}
Proposition~\ref{prp:depth} is derived by assuming the identity matrix as non-local basis. While we expect  similar results for general non-local
basis that satisfies the frame condition \eqref{eq:phi0}, the introduction of such non-local basis makes the analysis very complicated due to the non-commutative
nature of matrix multiplication and convolution, so we defer its analysis for future study.
\end{remark}

We believe that Eq.~\eqref{eq:rankdepth} is the key inequality to reveal the role of the depth.
Specifically,  to exploit the low-rank structure of the input signal, 
the upper bound of \eqref{eq:rankdepth} should be determined by the intrinsic property of the signal, $r:=\rank  \hank_{n}(f)$,
instead of the number of columns of the Hankel matrix, $d_{(l)} p_{(l)}$, which can be chosen arbitrary.
Thus,   if $d_{(l)}=d, \forall l$,  then Eqs.\eqref{eq:rankdepth} and \eqref{eq:rankdepth} inform  that the number of encoder-decoder layer depth $L$ is given by
\begin{eqnarray}
r \leq d^L  & \Longleftrightarrow  & L \geq \log_d (r) \  .
\end{eqnarray}
This implies that the number of minimum encoder-decoder layer is dependent on the rank structure of the input signal and
 we need a deeper network  for a more complicated
signal, which is consistent with the empirical findings.
Moreover, as shown in Fig.~\ref{fig:channel}, for a given intrinsic rank, the depth of the network also depends on the filter length.
For example, if the intrinsic rank of the Hankel matrix is 100, the network depth with respect to $d=2$ is 7, whereas
it is 5 when longer filter with $d=3$ is used.

Suppose that there are not sufficient number of output channels at a specific layer, say $l=l^*$, and the
signal content is approximated. 
Then, from the proof of Proposition~\ref{prp:depth} (in particular, \eqref{eq:rank1}), we can easily see
that $\rank \hank_{d_{(l)}|p_{(l)}}(C^{(l-1)})$ is upper-bounded by the rank structure
of the approximated signal at $l=l_*$. 
Moreover, if the numbers of channels are not sufficient for all layers, the rank will gradually decrease along the layers.
 This theoretical prediction will be confirmed later in Discussion using empirical data.

\subsection{ReLU Nonlinearity}

In order to reveal the  link between the convolutional framelets and the deep network,  we further need to consider
a nonlinearity.
 Note that the ReLU nonlinearity \cite{glorot2011deep,he2015delving} is currently most widely used 
for  deep learning approaches.
% and 
%this section investigates  perfect reconstruction conditions for convolutional framelets expansion when the expansion coefficient are processed with  ReLU.
%
%
%\begin{definition}[{\bf ReLU}]
Specifically, the ReLU $\rho(\cdot)$ is an element-wise operation for a matrix such that for a matrix $X=[x_{ij}]_{i,j=1}^{n,m} \in \Rd^{n\times m}$, the ReLU
operator provides non-negative part, i.e.
$\rho(X) =  \left[\max(0,x_{ij})\right]_{i,j=1}^{n,m}$.
By inserting ReLUs,
the $L$- layer encoder-decoder architecture neural network with bias
is defined  (with a slight abuse of notation) by
%
% can be equivalently represented using an encoder-decoder convolutional layer form:
%$$%\mathbbm{F}
%$$
 \begin{eqnarray}\label{eq:G}
\Qc(f) &:=& \Qc\left(f ;\{\Psi^{(j)},\tilde \Psi^{(j)},b_{enc}^{(j)},b_{dec}^{(j)}\}_{j=1}^L\right) \notag\\
&=&  \left( \tilde\Phi^{(1)}\rho(\hat C^{(1)}) \right)\circledast \nu(\tilde \Psi^{(1)})  + 1_{n}b_{dec}^{(1)\top} 
 \end{eqnarray}
 where % the decoder part of framelets coefficient at the $i$-th layer is given by
\begin{eqnarray}
\hat C^{(i)}  &=&\begin{cases} 
\left(\tilde\Phi^{(i+1)} \rho(\hat C^{(i+1)}) \right) \circledast \nu (\tilde \Psi^{(i+1)})+ 1_{n}b_{dec}^{(i+1)\top} ,  & 1\leq i <L \\
\rho(C^{(L)}), & i=L \end{cases}
\end{eqnarray}
and %whereas the encoder part framelet coefficients are given by
\begin{eqnarray}
C^{(i)}  &=&\begin{cases}  \Phi^{(i)\top} \left(\rho(C^{(i-1)}) \circledast \overline\Psi^{(i)} + 1_{n}b_{enc}^{(i)\top} \right),  & 1\leq i \leq L \\
f, & i=0 \end{cases}
\end{eqnarray}

Recall that the PR condition for the deep convolutional framelets was derived without assuming any nonlinearity.
Thus, the introduction of ReLU appears counter-intuitive in the context of PR.
Interestingly, in spite of ReLU nonlinearity, the following theorem shows that the PR condition can be satisfied
 when  filter channels  and bias having opposite phase are available.
% Here, by redundant filter channels, we refer the case with $\Psi \in \Rd^{d\times 2m}$ where $m \geq d$. % as shown in Fig.~\ref{fig:redun}.

  \begin{proposition}[\textbf{PR under ReLU using Opposite Phase Filter Channels}]\label{prp:prN}
% Consider the training problem in \eqref{eq:training} with $y_i=f_i$ for all $i=1,\cdots,N$.
  Suppose that %, furthermore,
   $\tilde\Phi^{(l)}\Phi^{(l)\top}=I_{n\times n}$ 
  and  $\Psi^{(l)},\tilde \Psi^{(l)} \in \Rd^{p_{(l)}d_{(l)}\times 2m_{(l)}}$  with $m_{(l)}\geq p_{(l)}d_{(l)}$ 
  for all $l=1,\cdots, L$.
  Then, the neural network output in \eqref{eq:G} satisfies the perfect reconstruction 
  condition, i.e. $f =  \Qc\left(f \right)$ if  the encoder-decoder filter channels are given by 
  %a solution for the optimization problem \eqref{eq:training} is given by
%  
%For a given input $C\in\Rd^{n\times p}$, let $\hank_p^d(C)$ denotes its extended Hankel matrix.
%Suppose that $\Phi\in \Rd^{n\times n}$ denotes the orthonormal basis.
%Then,
%for a given basis $\Psi =[\Psi_1,\Psi_2]\in \Rd^{pd\times 2r}$  and  dual basis $\tilde \Psi=[\tilde\Psi_1,\tilde \Psi_2]\in\Rd^{pd\times 2r}$ such that $r\geq pd$  and 
 \begin{eqnarray}\label{eq:CReLU}
\Psi^{(l)} = \begin{bmatrix} \Psi_+^{(l)} & -\Psi_+^{(l)} \end{bmatrix}&,\quad &\tilde\Psi^{(l)} = \begin{bmatrix} \tilde\Psi_+^{(l)} & -\tilde\Psi_+^{(l)} \end{bmatrix}  
\end{eqnarray}
satisfying the condition:
\begin{eqnarray}\label{eq:id}
\Psi_+^{(l)}\tilde \Psi_+^{(l)\top}= I_{p_{(l)}d_{(l)}\times p_{(l)}d_{(l)}}, \quad\quad\quad \Psi_+^{(l)},\tilde \Psi_+^{(l)} \in \Rd^{p_{(l)}d_{(l)}\times m_{(l)}}
\end{eqnarray}
and the encoder-decoder bias pairs $ b_{enc}^{(l)}, b_{dec}^{(l)}$ are given by
\begin{eqnarray}\label{eq:Cbias}
b_{enc}^{(l)}= \begin{bmatrix} b_{enc,+}^{(l)T} & -b_{enc,+}^{(l)T} \end{bmatrix}&,\quad  & 1_{n}b_{dec}^{(l)\top}  = -  \hank_{d|p}^\dag \left(   1_{n}b_{enc,+}^{(l)\top}\tilde \Psi^{(l)\top}\right) \  .
% \Psi_2 = -\Psi_+ &,  \tilde \Psi_2 = - \tilde\Psi_+ = - \Psi_+^R \ ,
\end{eqnarray}
% $$  such that
%which guarantees the PR, i.e.
%$$f_i =g\left(f_i;\{\Phi^{(j)},\tilde \Phi^{(j)}\}_{j=1}^L\right),\quad \forall i.$$
\end{proposition}
\begin{proof}
See Appendix~\ref{ap5}.
\end{proof}

\begin{remark}
 Eq.~\eqref{eq:CReLU} in Proposition~\ref{prp:prN} predicts the existence of   filter pairs with opposite phase. % as described in \eqref{eq:CReLU}. 
Amazingly, this theoretical prediction coincides
with the empirical observation in deep learning literature.  For example, Shang et al \cite{shang2016understanding} observed 
 an intriguing property that the filters in the lower layers form pairs (i.e., filters with opposite phase). To exploit this property for further network
 performance improvement, the authors proposed
 so called concatenated ReLU (CReLU) network to explicitly retrieve the negative part of $\Phi^{\top}F\Psi_+$ using
 $\rho\left(\Phi^{\top}F(-\Psi_+)\right)$ \cite{shang2016understanding}.
\end{remark}

\begin{remark}
Note that there are infinite number of filters satisfying \eqref{eq:CReLU}.
In fact, the most important requirement for PR is the {\em existence} of  opposite phase filters as in  \eqref{eq:CReLU} 
satisfying the frame condition \eqref{eq:id} rather than the specific filter coefficients.
This may suggest the excellent generalization performance of a deep network even from small set of training data set.  
%As long as the training data is sufficient to learn
%Moreover, there could be other global minimizers for \eqref{eq:training} that may not assume the form in \eqref{eq:CReLU}. % is one of the (possibly) many global minimizers.
%Still,  Theorem~\ref{thm:prN} is important because the filter satisfying the PR is still valid from neural network training.
\end{remark}

Proposition~\ref{prp:prN}  deals with the PR condition using redundant local filters
 $\Psi \in \Rd^{pd\times 2m}$ with $m \geq pd$. 
 This can be easily satisfied at the lower layers of the deep convolutional
framelets; however, the number of filter channels for PR   grows exponentially according to layers as shown in \eqref{eq:prod}.
%(see \cite{?}).
Thus, at higher layers of deep convolutional framelets, the condition  $m \geq pd$ for Proposition~\ref{prp:prN} may not be satisfied.
However, even in this case, we can still achieve PR as long as the extended Hankel matrix at that layer is sufficiently
low-ranked.

 \begin{proposition}[\textbf{Low-Rank Approximation with Insufficient Filter Channels}]\label{prp:PRinsufficient}
For a given input $X\in\Rd^{n\times p}$, let $\hank_{d|p}(X)$ denotes its extended Hankel matrix whose
rank is $r$.
%and its  SVD  is given by 
%\begin{eqnarray}\label{eq:SVD}
%\hank_{d|p}(X)=U\Sigma V^\top = \sum_{i=1}^{pd}\sigma_iu_iv_i^\top,
%\end{eqnarray}
%where
%$u_i$ and $v_i$ denotes the left and right singular vectors, and 
%$\sigma_1\geq \sigma_2 \geq \cdots \geq \sigma_{pd}\geq 0$ are the singular values.
Suppose, furthermore, that $\tilde\Phi\Phi^\top=I_{n\times n}$. % denotes the orthonormal basis.
Then,
there exists  $\Psi\in \Rd^{pd\times 2m}$  and $\tilde \Psi\in\Rd^{pd\times 2m}$ with  $r\leq m< pd$  such that 
$$X = \hank_{d|p}^\dag( \Phi\rho\left(\Phi^\top\hank_{d|p}(X)\Psi\right) \tilde \Psi^\top).$$
\end{proposition}
\begin{proof}
See Appendix~\ref{ap6}.
\end{proof}
%Thus, the small number of channels are usually inserted in convolutional framelets as shown in Fig.~\ref{fig:ed}(a)(b).
%
However, as the network gets deeper, more layers cannot satisfy the condition for Proposition~\ref{prp:PRinsufficient}. % \eqref{eq:prod}.
%Thus,
%successive applications of the CNN can also reduce the important signals. 
To address this, the residual
net \cite{he2016deep}  is useful.
Recall that   the residual net (ResNet) has been widely used for image classification as well as image reconstruction.
More specifically, the residual net architecture shown in Fig.~\ref{fig:resnet}(a) can be represented by
\begin{eqnarray}\label{eq:resnet}
\hat X  = R(F;\Psi,\tilde\Psi): = \rho\left(  F- \rho( F\Psi)\tilde \Psi^\top\right)
\end{eqnarray}
where $F:=\hank_{d|p}(X)$. % and \
The following result shows that the
residual network truncates the  least significant subspaces:% if the filter length is not sufficiently large.

  \begin{figure}[!bt] 
\center{\includegraphics[width=8cm]{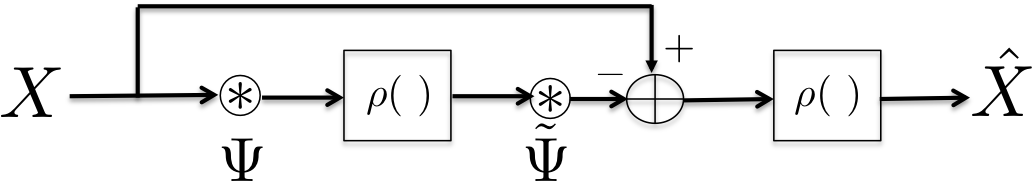}}
\centerline{\mbox{(a)}}
\vspace{0.5cm}
\center{\includegraphics[width=7cm,height=1.5cm]{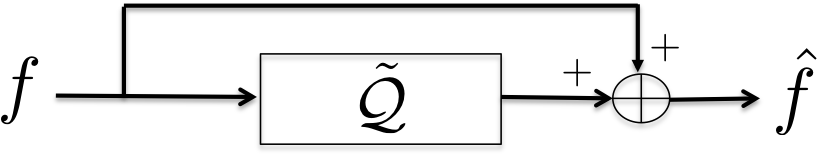}}
\centerline{\mbox{(b)}}
\vspace{-0.5cm}
\caption{Block diagram of  (a) a residual block and (b) skipped connection.
}
\label{fig:resnet}
\end{figure}

 \begin{proposition}[\textbf{High-rank Approximation using Residual Nets}]\label{prp:resnet}
 For a given nonnegative input $X\in\Rd^{n\times p}$, let $\hank_{d|p}(X)$ denotes its extended Hankel matrix
and its  SVD  is given by% \eqref{eq:SVD}.
\begin{eqnarray}\label{eq:SVD}
\hank_{d|p}(X)=U\Sigma V^\top = \sum_{i=1}^{r}\sigma_iu_iv_i^\top,
\end{eqnarray}
where
$u_i$ and $v_i$ denotes the left and right singular vectors, and 
$\sigma_1\geq \sigma_2 \geq \cdots \geq \sigma_{r}> 0$ are the singular values and $r$ denotes the rank.
%Suppose, furthermore, that $\tilde\Phi\Phi^\top = I_{n\times n}$.
Then,
there exists   $\Psi\in \Rd^{pd\times 2m}$  and   $\tilde \Psi\in\Rd^{pd\times 2m}$ with  $r < pd$  such that 
\begin{eqnarray}\label{eq:PRres}
X = \hank_{d|p}^\dag(  R(F;\Phi,\Psi,\tilde\Psi) ),
\end{eqnarray}
where $R(F;\Psi,\tilde\Psi)$ is given by \eqref{eq:resnet}.
\begin{proof}
See Appendix~\ref{ap7}.
\end{proof}
\end{proposition}

Note that the PR condition for the residual network is much more relaxed than that of Proposition~\ref{prp:PRinsufficient},
since we do not have any constraint on $m$ that determines the output filter channels number.  Specifically,  for the case of Proposition~\ref{prp:PRinsufficient}, the number of output channel filters should be  $q=2m\geq 2r$.  On the other hand,  the PR condition for residual net only requires
the existence of the null space in the extended Hankel matrix (i.e. $r< pd$), and it does not depend on the number of output channels. Therefore,
we have more freedom to choose filters.

In addition, Proposition~\ref{prp:PRinsufficient} also implies the low rank approximation, if the number of  output channels is not sufficient.
Specifically, with matched opposite phase filters, we have
\begin{eqnarray}\label{eq:mrank}
 \Phi\rho\left(\Phi^\top\hank_{d|p}(X)\Psi\right) \tilde \Psi^\top)=\hank_{d|p}(X)\Psi_+ \tilde \Psi_+^\top \  .
 \end{eqnarray}
Thus, we can choose $\Psi_+,  \tilde \Psi_+$  such that  \eqref{eq:mrank} results in the rank-$m$ approximation of
$\hank_{d|p}(X)$.
On the other hand,   from the proof in Appendix~\ref{ap7}, we can see that %Theorem~\ref{thm:resnet} implies a high rank approximation  because
 \begin{eqnarray}\label{eq:resnet_insuf}
R(F;\Phi,\Psi_+,\tilde\Psi_+)  
 &=&  \rho\left(\hank_{d|p}(X) -  U\Sigma V^\top \Psi_+\tilde \Psi_+^\top \right)
\end{eqnarray}
Accordingly, by choosing  $\Psi_+,  \tilde \Psi_+$  such that $\Psi_+\tilde \Psi_+^\top = P_{R(v_{pd})}$ where $v_{pd}$ is the singular vector for the least
singular value,   we can minimize the error of approximating $\hank_{d|p}(X)$ using $R(F;\Psi_+,\tilde\Psi_+)$.
This is why we refer Proposition~\ref{prp:resnet}  to as the high rank approximation using residual net.
%This implies that  if we do not use either sufficient number of filter channels,
%the network performs a low-rank (or high-rank) approximation of the Hankel matrix from the input signal. 

\subsection{Role of perfect reconstruction condition and shrinkage behaviour of network}

%and

So far, we have investigated the perfect recovery (PR) condition for  deep convolutional framelets.
% However, in practice, the
%conditions for PR may be still stringent to meet.
Here,  we are now ready to
explain why the PR  is useful for inverse problems. % and  what happens when the PR is not satisfied in deep convolutional framelets.

Suppose that an analysis operator $\Wc$, which are composed of frame bases, and  the associated synthesis operator  $\tilde\Wc^\top$ 
satisfy the following frame condition:
\begin{eqnarray}\label{eq:PRC}
f = \tilde \Wc^\top \Wc f , \quad \forall f \in H
\end{eqnarray}
where $H$ denotes the Hilbert space of our interest.
Signal processing using frame basis satisfying \eqref{eq:PRC} has been extensively studied for frame-based image processing \cite{mallat1999wavelet,cai2008framelet,cai2009convergence}.
One of the important advantages of these frame-based algorithm is the proven convergence \cite{cai2008framelet,cai2009convergence},
which makes the algorithm powerful.
%The frame-based image processing algorithms \cite{?} then exploits the PR condition \eqref{eq:PRC}.

For example, consider a signal denosing algorithm to  recover noiseless signal $f^*$ from the
noisy measurement $g=f^*+e$, where $e$ is the additive noise. 
Then, a frame-based denoising algorithm \cite{daubechies2003framelets,cai2009split,cai2012image} recovers the unknown signal  by substituting $g$  for $f$ in the right side of \eqref{eq:PRC}  and applies
a shrinkage operator $S_\tau$ to the framelet coefficients to eliminate the small magnitude noise signals:
\begin{eqnarray}\label{eq:denoising}
\hat f = \tilde \Wc^\top \Sc_\tau( \Wc g)
\end{eqnarray}
where $\tau$ denotes a shrinkage parameter. In denoising, soft- or hard- thresholding shrinkage operations are most widely used.
Thus, to make the frame-based denoising algorithm successful, the frame should be energy compacting so that most of the signals are concentrated in
a small number of framelet coefficients, whereas the noises are spread out across all framelet coefficients.

As an another example, consider  a signal inpainting problem \cite{cai2008framelet,cai2009convergence}, where % .
for a given signal $f\in H$,
we measure $g$ only on the index set $\Lambda$ and our 
 goal is then to estimate the unknown image $f$ on the complementary index set $\Lambda^c$.
Let $P_\Lambda$ be the diagonal matrix with diagonal entries 1 for the indices in $\Lambda$ and 0 otherwise. 
Then, we have the following identities:
\begin{eqnarray}
f %&=& P_\Lambda f + (I-P_\Lambda) f \notag\\
&=& \mu P_\Lambda f + (I-\mu P_\Lambda) f \notag\\
&=& \mu P_\Lambda g + (I-\mu P_\Lambda) \tilde\Wc^\top \Wc f  \label{eq:update0}
\end{eqnarray}
where we use the  frame condition \eqref{eq:PRC} for the last equality.
A straight-forward frame-based inpainting algorithms \cite{cai2008framelet,cai2009convergence} can be then obtained from \eqref{eq:update0} as: 
\begin{eqnarray}\label{eq:inpainting}
f_{n+1} = \mu P_\Lambda g + (I-\mu P_\Lambda) \Wc^\top \Sc_\tau \left( \Wc f_n \right) \  ,
\end{eqnarray}
where $\Sc_\tau$ denotes a shrinkage operator with the parameter $\tau$ and $f_0$ is initialized to $P_\Lambda g$.

Now,  note that the computation of our deep convolutional framelet coefficients can  represent an 
analysis operator:
$$\Wc f: = C = \Phi^\top (f\circledast \overline \Psi)$$
whereas the synthesis operator is given by the decoder part of convolution:
$$\tilde \Wc^\top C := (\Phi C) \circledast \nu (\tilde \Psi). $$
If the non-local and local basis satisfy the frame condition \eqref{eq:phi0} and \eqref{eq:ri0} for all layers,  they satisfy the  PR condition \eqref{eq:PRC}.
Thus, we could use \eqref{eq:denoising} and \eqref{eq:inpainting} for denoising and inpainting applications.
Then, what is the shrinkage operator in our deep convolutional framelets?  

Although one could still use similar soft- or hard- thresholding operator,
one of the unique aspects of deep convolutional framelets is that
% the ReLU nonlinearity, bias, and  
 the number of filter channels can control  the shrinkage behaviour.
 Moreover, the optimal local basis are learnt from the training data such that they give the best shrinkage behavior.
  More specifically, %Proposition~\ref{prp:PRinsufficient} and  Proposition~\ref{prp:resnet} tells that
   the
low-rank shrinkage behaviour emerges  when the number of output filter channels are not sufficient.
These results  are very useful in practice, since the low-rank approximation of Hankel matrix is good for reducing the noises and artifacts as demonstrated in image denoising \cite{jin2015sparse+}, artifact removal \cite{jin2016mri} 
and deconvolution \cite{min2015fast}.
%Furthermore, the bias can control the polarity of $C$ in \eqref{eq:biasEnc}, so it is believed that
%the bias can shift the significant convolutional framelet coefficients to positive values such that ReLU nonlinearity provides a shrinkage operation to the insignificant framelet coefficients. Therefore, in addition to the low-rank shrinkage operation provided by the number of filter channels, the bias may work as the sparsity imposing shrinkage operation.  %Therefore,  deep convolutional framelet expansion with ReLU may exploit both sparsity and low-rank shrinkage operations.
 
 To understand this claim, consider the following regression problem under the Hankel structured low-rank constraint:
% Now,   a generic form of the low-rank Hankel structured constrained regression problem can be formulated as
\begin{eqnarray}
\min_{f\in \Rd^{n}}  & \|f^* -f\|^2 \notag\\
\mbox{subject to }  &\quad \rank \hank_d(f) \leq r < d . \label{eq:fcost}
\end{eqnarray}
where $f^*\in \Rd^d$ denotes the ground-truth signal and we are interested in finding rank-$r$ approximation.
Then, for any feasible solution $f$ for \eqref{eq:fcost}, its Hankel structured matrix $\hank_d(f)$
 has  the singular value decomposition
%\begin{align}\label{eq:svd0}
$\hank_{d}(f) = U \Sigma V^{\top}$
%\end{align}
where $U =[u_1 \cdots u_r] \in \Rd^{n\times r}$ and $V=[v_1\cdots v_r]\in \Rd^{d\times r}$ denote the left and the right singular vector bases matrices, respectively;
$\Sigma=(\sigma_{ij})\in\mathbb{R}^{r\times r}$ is the diagonal matrix with singular values.  
%such that $\Sigma_{ii}$, for all $0<i<r$, is the singular value of $\mathbb{H}_d(f)$ and $\Sigma_{ij}=0$ otherwise. 
%Here, $r$ denotes the rank of $\mathbb{H}_d(f)$.
% $\Sigma \in \Rd^{r\times r}$ is the diagonal matrix whose diagonal components contain the singular values, and $r$ denotes its rank.
Then,  we can find two matrices pairs $\Phi, \tilde \Phi \in \Rd^{n\times n}$ and $\Psi$, $\tilde \Psi\in \Rd^{d\times r}$ satisfying the conditions
\begin{eqnarray}
%\sum_{k=1}^p T^{(k)\top}T^{(k)} &=& I_{n\times n},~ \label{eq:T} \\
\tilde \Phi \Phi^\top = I_{n\times n}, \qquad \Psi \tilde \Psi^{\top} = P_{R(V)}, \label{eq:noframe}
\end{eqnarray}
%where $R(V)$ denote the range space of $\Vb$ and $P_{R(\Vb)}$ represents a projection onto $R(V)$, 
such that
$$\hank_d(f) = \tilde \Phi \Phi^\top  \hank_d(f)  \Psi \tilde \Psi^{\top},$$
which leads to the decomposition of $f$ using a
single layer deep convolutional framelet expansion:
\begin{eqnarray}
f = % \hank_d^\dag\left( \hank_d(f) \right) %&=& \frac{1}{d} \sum_{i=1}^{m} \left( \tilde \Phi c_j\right) \circledast \tilde \psi_j \notag\\
\left(\tilde\Phi C\right) \circledast \nu(\tilde \Psi),  && \quad \mbox{where}\quad
C = \Phi^\top \left( f \circledast \overline \Psi  \right) \label{eq:finsuf}
\end{eqnarray}
%where $b_{enc} \in \Rd^{d}, b_{dec} \in \Rd$ denotes the encoder and decoder biases, respectively.
Note that 
%the bases matrix $\tilde \Psi\in \Rd^{d\times r}$ in \eqref{eq:noframe} does not satisfy the frame condition due to the insufficient number channels,
%i.e. $r< d$, and 
\eqref{eq:finsuf} is the general form of the signals that are associated with a rank-$r$ Hankel structured matrix.

Note that the basis ($\Phi, \tilde \Phi$) and ($\Psi$, $\tilde \Psi$) are not specified in  \eqref{eq:newcost}.
In our deep convolutional framelets,  $\Phi$ and $\tilde \Phi$ correspond to the generalized pooling and unpooling which are chosen based on
domain knowledges, so we are interested in only estimating the filters $\Psi$, $\tilde \Psi$.
To restrict the search space furthermore, 
let $\Hc_r$ denote the space of such signals  that have positive framelet coefficient, i.e. $C \geq 0$: % under bias:
\begin{eqnarray}
\Hc_r &=& \left\{ f\in \Rd^n~|~ f = \left(\tilde\Phi C\right) \circledast \nu(\tilde \Psi)+ 1_n b_{dec}, \quad
C = \Phi^\top \left( f \circledast \overline \Psi  + 1_n b_{enc}^\top\right)\geq 0 \right\} %\label{eq:finsuf}
\end{eqnarray}
where $b_{enc} \in \Rd^{d}, b_{dec} \in \Rd$ denotes the encoder and decoder biases, respectively.
Then, the main goal of the neural network training  is to learn   ($\Psi$, $\tilde \Psi$, $b_{enc}, b_{dec}$) from training data  $\{(f_{(i)}, f_{(i)}^*)\}_{i=1}^N$ assuming that
$\{f_{(i)}^*\}$ are associated with rank-$r$ Hankel matrices.
More specifically,  our regression problem under low-rank Hankel matrix constraint in \eqref{eq:fcost} for the training data can be equivalently represented by
\begin{eqnarray}\label{eq:newcost}
\min_{\{f_{(i)}\}\in \Hc_r}   \sum_{i=1}^N\|f_{(i)}^* - f_{(i)}\|^2  &=& \min_{( \Psi, \tilde\Psi,b_{enc},b_{dec})}  \sum_{i=1}^N\left\|f_{(i)}^* - \Qc(f_{(i)};\Psi,\tilde\Psi,b_{enc},b_{dec})\right\|^2
\end{eqnarray}
where
$$\Qc(f_{(i)};\Psi,\tilde\Psi,b_{enc},b_{dec})= \left(\tilde\Phi C[f_{(i)}]\right) \circledast \nu(\tilde \Psi)+ 1_n b_{dec}$$
$$\quad C[f_{(i)}] =\rho\left(
\Phi^\top \left( f_{(i)} \circledast \overline \Psi  + 1_n b_{enc}^\top\right)\right),$$
where $\rho(\cdot)$ is the ReLU to impose the positivity.
After the network is fully trained, the inference for a given noisy input $f$ is simply done by
$\Qc(f;\Psi,\tilde\Psi,b_{enc},b_{dec})$, which is equivalent to find a denoised solution that has the rank-$r$ Hankel structured matrix.
Therefore, using deep convolutional framelets with insufficient channels, we do not need an explicit shrinkage operation.
The idea can be further extended to the multi-layer deep convolutional framelet expansion as follows:
\begin{definition}[\bf Deep Convolutional Framelets Training]
Let $\{f_{(i)},f_{(i)}^*\}_{i=1}^N$ denote the input and target sample pairs. Then, the deep convolutional framelets training problem is given by
\begin{eqnarray}\label{eq:training}
\min_{\{\Psi^{(j)},\tilde \Psi^{(j)},b_{enc}^{(j)}, b_{(dec)}^{(j)}\}_{j=1}^L}\sum_{i=1}^N \left\|f_{(i)}^* - \Qc\left(f_{(i)};\{\Psi^{(j)},\tilde \Psi^{(j)},b_{enc}^{(j)}, b_{(dec)}^{(j)}\}_{j=1}^L\right)\right\|^2
\end{eqnarray}
where $\Qc$ is defined by \eqref{eq:G}. % with ReLU for each layer.
%\begin{eqnarray}\label{eq:G}
%\mathbbm{G}\left(f_i;\{\Phi^{(j)},\tilde \Phi^{(j)}\}_{j=1}^L\right)
%\end{eqnarray}
%\begin{eqnarray*}
%=  \hank_{d_{(1)}}^{\dag} \left(\tilde\Phi^{(1)}  \cdots \rho\left(  \hank_{d_{(L)}|p_{(L)}}^{\dag} \left(\tilde\Phi^{(L)} \rho\left(\Phi^{(L)\top}  \hank_{d_{(L)}|p_{(L)}}\cdots \rho(\Phi^{(1)\top}  \hank_{d_{(1)}}(f) \Psi^{(1)} ) \cdots \Psi^{(L)}\right)  \tilde\Psi^{(L)\top} \right) \right) \cdots \tilde\Psi^{(1)\top} \right) 
%\end{eqnarray*}
\end{definition}

Thanks to the inherent shrinkage behaviour from the neural network training,   \eqref{eq:denoising} and \eqref{eq:inpainting} can be written as:
 \begin{eqnarray}
 \hat f &=& \Qc(g)  \label{eq:deepdenoising} \\
 f_{n+1} &=&  \mu P_\Lambda g + (I-\mu P_\Lambda) \Qc(f_n) \label{eq:deepinpainting}
 \end{eqnarray}
where $\Qc(\cdot)$ denotes the deep convolutional framelet output.
Note that \eqref{eq:deepdenoising} is in fact the existing deep learning denoising algorithm \cite{xie2012image,burger2012image,bae2016beyond},
whereas the first iteration of \eqref{eq:deepinpainting} corresponds to the existing deep learning-based inpainting algorithm \cite{xie2012image,yeh2016semantic}.
This again confirms that deep convolutional framelets is a general deep learning framework. Later, we will provide numerical experiments using \eqref{eq:deepdenoising} and \eqref{eq:deepinpainting} for image denosing
and inpainting applications, respectively.

% We conjecture that the deep network training in \eqref{eq:training} makes the convolutional
%  framelets expansion  to have the best energy compaction  for a given task by finding the optimal local bases for given non-local bases, ReLU and bias.
%
%
%
%
%
%Thus, if we choose the nonlocal and local bases to have energy compaction,   the most dominant convolutional framelet coefficients are from the underlying intact signals, so we can apply some operations to remove the remaining framelet coefficients and to obtain the noise reduced image.
%   which may be  the origin of the power of deep learning in the context of inverse problem.

\subsection{Bypass-Connection versus  No Bypass-Connection}

The  by-pass connection in Fig.~\ref{fig:resnet}(b) is  closely related to  ResNet in Fig.~\ref{fig:resnet}(a), except that the final level of nonlinearity is not commonly used in by-pass connection and
the by-pass connection is normally placed between the  input and the output of the network.
In practice, researchers empirically determine whether to use bypass connections or not by trial and error.

In order to understand the role of the by-pass connection,  we need to revisit the single layer training procedure in \eqref{eq:newcost}.
Then,   the network with the bypass connection can be written by
$$\Qc(f_{(i)};\Psi,\tilde\Psi,b_{enc},b_{dec})  = \tilde\Qc(f_{(i)};\Psi,\tilde\Psi, b_{enc},b_{dec}) + f_{(i)},$$
where 
%\begin{eqnarray*}%\label{eq:tildeQ}
$$\tilde\Qc(f_{(i)};\Psi,\tilde\Psi,b_{enc},b_{dec}) =\left(\tilde\Phi C[f_{(i)}]\right) \circledast \nu(\tilde \Psi) + 1_n b_{dec},$$
%\end{eqnarray*}
$$C[f_{(i)}]=
\Phi^\top \left( f_{(i)}\circledast \overline \Psi  +1_n b_{enc}^\top \right).$$
Next,  note that $f_{(i)}$ is contaminated with artifacts so that  it can be written by
$$f_{(i)} = f_{(i)}^*+h_{(i)}, $$
where $h_{(i)}$ denotes the noise components and $f_{(i)}^*$ refers to the noise-free ground-truth.
Therefore, 
the network training \eqref{eq:newcost} using the by-pass connection can be equivalently written by
\begin{eqnarray}\label{eq:training2}
 \min_{( \Psi, \tilde\Psi,b_{enc},b_{dec})}  \sum_{i=1}^N\left\|h_{(i)}+ \tilde\Qc(f_{(i)}^*+h_{(i)};\Psi,\tilde\Psi,b_{enc},b_{dec})\right\|^2
 \end{eqnarray}
Therefore, if we can find an encoder filter $\overline \Psi $ such that
it approximately annihilates the true signal $f_{(i)}^*$, i.e. 
\begin{eqnarray}\label{eq:fanal}
 f_{(i)}^*\circledast \overline \Psi \simeq 0   \  , 
%R(\overline\Psi) = R\left((\hank_d(h_{(i)}))^\top\right)
\end{eqnarray}
then we have
\begin{eqnarray*}
 C[f_{(i)}^*+h_{(i)}] &=& \Phi^\top \left( (f_{(i)}^*+h_{(i)})\circledast \overline \Psi  +1_n b_{enc}^\top \right) \\
 &\simeq & \Phi^\top \left( h_{(i)}\circledast \overline \Psi  +1_n b_{enc}^\top \right) \\
 &=& C[h_{(i)}]
 \end{eqnarray*}
With this filter $\overline\Psi$,  the decoder filter $\tilde\Psi$ can be designed such that
\begin{eqnarray*}
\left(\tilde\Phi C[h_{(i)}]\right) \circledast \nu(\tilde \Psi) 
&=&\left(\tilde\Phi   \Phi^\top \left( (h_{(i)})\circledast \overline \Psi  \right)\right) \circledast \nu(\tilde \Psi) \\
&=& h_{(i)}\circledast \overline \Psi  \circledast \nu(\tilde \Psi) \\
&\simeq& - h_{(i)} \  ,
\end{eqnarray*}
by minimizing the cost \eqref{eq:training2}.
%by choosing $\Psi \tilde\Psi = - R(V)$.
Thus, our deep convolutional framelet with a by-pass connection can  recover the artifact signal $h_{(i)}$ (hence, it recovers $f_{(i)}^*$ by subtracting it from $f_{(i)}$).
 In fact, \eqref{eq:fanal} is equivalent to the filter condition in \eqref{eq:resnet_insuf} in ResNet which spans the minimum singular vector 
 subspace.

Using the similar argument, we can see that if the encoder filter annihilates the artifacts, i.,e.
\begin{eqnarray}
 h_{(i)}\circledast \overline \Psi \simeq 0  %R(\overline\Psi) = R\left((\hank_d(f_{(i)}^*))^\top\right)
\end{eqnarray}
then $C[f_{(i)}^*+h_{(i)}] \simeq C[f_{(i)}^*]$
and our deep convolutional framelet {\em without} the by-connection can recover the ground-truth signal $f_{(i)}^*$ 
i.e.
$\Qc(f_{(i)};\Psi,\tilde\Psi) \simeq f_{(i)}^*$
by minimizing the cost \eqref{eq:training2}.

In brief,   if the true underlying signal has lower dimensional structure than the artifact,
then  the annihlating filter relationship in \eqref{eq:fanal} is more easier to achieve \cite{vetterli2002sampling}; thus,  the neural network with by-pass connection achieves better performance.
On the other hand, if the artifact has lower dimensional structure,  a neural network without by-pass connection is better.
This coincides with many empirical findings. For example, in image denoising \cite{kang2017wavelet,zhang2016beyond}  or streaking artifact problems \cite{han2017framing,jin2016deep},
a residual network works better, since the artifacts are random and have complicated distribution.
In contrast,  to remove the cupping artifacts 
 in the X-ray interior tomography problem,  a CNN without skipped
connection is better, since the cupping artifacts are usually smoothly varying \cite{han2017interior}.

\section{Multi-Resolution Analysis via Deep Convolutional Framelets}

In deep convolutional framelets,  for a given non-local basis, 
the local  convolution filters are learnt to give the best shrinkage behaviour.
Thus,
non-local basis $\Phi$ is an important design parameter that controls the performance.
In particular, the energy compaction
property for the deep convolutional framelets is significantly affected by $\Phi$.
%since the local-basis can be optimized with respect to the non-local basis  using training data.
Recall that   the SVD  basis for the Hankel matrix results in the best energy compaction property; however, 
the SVD basis varies depending on input signal type so that we cannot use the same basis for various input data.

Therefore, we should choose an analytic non-local  basis $\Phi$ such that it can approximate the SVD basis and result in good energy compaction property.
Thus,   wavelet is one of the preferable choices  for piecewise continuous signals and images \cite{daubechies1992ten}.
%, which
%is another important motivation for our multi-resolution analysis using deep convolutional framelets in this section. 
Specifically, in wavelet basis, the standard  pooling and unpooling networks are used  as low-frequency path of wavelet transform, but there exists additional high-frequency paths from wavelet transform.
Another important motivation for multi-resolution analysis of convolutional framelets  is the exponentially large receptive field.
For example,
Fig. \ref{fig:receptive_field} compares  the network depth-wise effective receptive field of a multi-resolutional
network with pooling  against that of  a baseline network without pooling layers.
With the same size convolutional filters, the effective receptive field is enlarged in the network with pooling layers.
Therefore, our multi-resolution analysis (MRA) is indeed derived to supplement the enlarged receptive field from
pooling layers with the fine detailed processing using high-pass band convolutional framelets.

   \begin{figure}[!hbt]
    \centerline{\includegraphics[width=7cm]{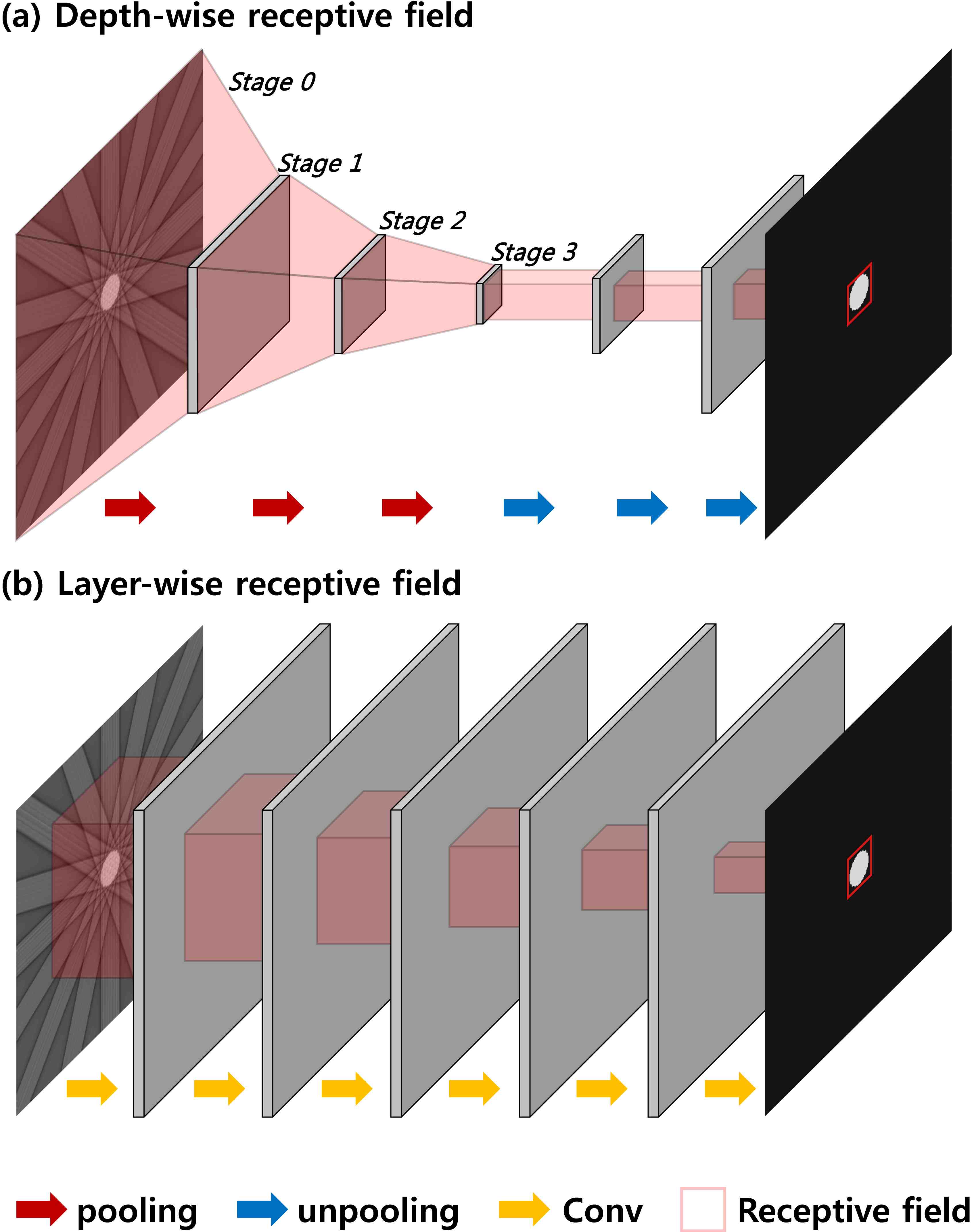}}
    \caption{Effective receptive field comparison. (a)  Multi-resolution network,  and (b) CNN without pooling.}
    \label{fig:receptive_field}
\end{figure}

\subsection{Limitation of U-Net}
\label{sec:limitation}

%This section discuses  the  multi-resolution analysis (MRA) using deep  convolutional framelets in detail due to its importance.
%In the following, we therefore propose a new class of deep learning network using multi-resolution deep convolutional framelets.
%, which is believed to be
%quite universal for various inverse problems. 

Before we explain our multi-resolution deep convolutional framelets,  we first discuss the limitations of the popular multi-resolution
deep learning architecture called U-net \cite{ronneberger2015u}, which
 is a composed of encoder and decoder network with a skipped connection.
The U-Net utilizes the pooling and unpooling as shown in Fig.~\ref{fig:proposed}(a) to obtain the exponentially large receptive field. % (see Fig.~\ref{fig:proposed}(a)).
%
%
%So far, we have claimed that the non-local basis $\Phi \in \Rd^{n\times m}$ and its dual $\tilde\Phi$ %and local basis $\Psi$ 
%should satisfy \eqref{eq:phi0} for the PR condition. 
%However, there are some popular non-local basis structures that do not satisfy the requirement.
%For example,  pooling and unpooling layers in the deep neural
%network corresponds to a non-local transform that convert signals to a different resolution.
%These
% are quite often used in classifier design as well as segmentation network \cite{noh2015learning}.
Specifically,   the  average and max pooling operators $\Phi_{ave}, \Phi_{max} \in \Rd^{n\times \frac{n}{2}}$  for $f\in \Rd^n$ used in U-net is defined as follows:
\begin{eqnarray}
\Phi_{ave}= \frac{1}{\sqrt{2}} \begin{bmatrix} 1 & 0 & \cdots & 0 \\ 1 & 0 & \cdots & 0 \\  0 & 1 & \cdots & 0 \\ 
 0 & 1 & \cdots & 0 \\  \vdots & \vdots & \ddots & \vdots \\ 0 & 0 & \cdots & 1  \\ 0 & 0 & \vdots & 1 \end{bmatrix}  \label{eq:apool}
 &,&
 \Phi_{max}=  \begin{bmatrix} b_1 & 0 & \cdots & 0 \\ 1-b_1 & 0 & \cdots & 0 \\  0 & b_2 & \cdots & 0 \\ 
 0 & 1-b_2 & \cdots & 0 \\  \vdots & \vdots & \ddots & \vdots \\ 0 & 0 & \cdots & b_{\frac{n}{2}}  \\ 0 & 0 & \vdots & 1-b_{\frac{n}{2}} \end{bmatrix}
 \end{eqnarray}
where $\{b_i\}_i$ in max pooling are random $(0,1)$ binary numbers that are determined by the signal statistics.
%\end{example}
We can easily see that the columns of  max pooling or average pooling are orthogonal to each other; however, 
it does not constitute a basis because it  does not span $\Rd^n$.
%Indeed, this neural network can be considered of the multi-resolutional deep convolutional framelets with $\Phi_{low}$ only, and
%the perfect recovery is not guaranteed in this architecture. 
Then, what does this network perform?

Recall that for the case of average pooling,  the unpooling layer $\tilde \Phi$ has the same form as the pooling, i.e. $\tilde \Phi =\Phi$.
In this case, under the frame condition for the local bases $\Psi\tilde \Psi^{\top}=I_{d\times d}$,
the signal after pooling and unpooling becomes:
\begin{eqnarray*}
\hat f &=& \hank_d^\dag\left( \Phi ( \Phi^{\top}\hank_d(f)\right) = \Phi  \Phi^{\top}f
\end{eqnarray*}
which is basically a low-pass filtered signal and the detail signals are lost.
To address this limitation and retain the fine detail,
U-net has by-pass connections and concatenation layers as shown in Fig.~\ref{fig:proposed}(a).
%
%This implies that the pooling and unpooling removes the details of the input signals by projecting it to the low-resolution
%signals. 
%This works fine for the segmentation problem; but, a care needs to be taken for the case of image restoration, because
%the goal is to keep the details of the signal.
%In fact,  this limitation of pooling and unpooling can be easily addressed by the multi-resolution analysis of convolutional
%framelets, which will be discussed in detail later.
Specifically, combining the low-pass and by-pass connection, the augmented convolutional framelet coefficients $C_{aug}$ can be  represented by
\begin{eqnarray}\label{eq:Y}
C_{aug} &=& \Phi_{aug}^\top \hank_d(f)\Psi = \Phi_{aug}^\top  (f\circledast \overline\Psi)  = \begin{bmatrix} C \\ S \end{bmatrix} %\begin{bmatrix} I   \\ \Phi^\top \end{bmatrix} \hank_d(f) \Psi 
\end{eqnarray}
where 
\begin{eqnarray}\label{eq:A}
\Phi_{aug}^\top := \begin{bmatrix} I   \\ \Phi^\top \end{bmatrix}, \quad& C := f\circledast \Psi, & \quad S := \Phi^\top  (f\circledast  \overline \Psi)
\end{eqnarray}
% Note that
%$$A^\top A  = I + \Phi \Phi^\top ,$$
%where $\Phi\Phi^\top$ is a rank deficient matrix.
%In particular, for the case of average pooling, we have $\Phi\Phi^\top=  P_{R(\Phi)}$, so
%this cannot become a scaled identity matrix as in \eqref{eq:tight}.
%Therefore, the U-Net frame is not tight and  prone to noise boosting.
%
%
%
%Another important limitation of U-net is that it does not satisfy the recovery condition.
%While this can be readily seen from that U-net does not satisfy the condition \eqref{eq:id},
%we can explicitly see this  from the unique concatenation step in U-net (see Fig.~\ref{fig:Unet}(a)).
%Another important uniqueness of the U-net is that
%the filtered input signals are concatenated together  before a local filtering is applied.
%Then, the recovery step becomes
%\begin{eqnarray}
%\hat f  &=& (\Phi_{aug}C_{aug})\circledast \nu(\tilde \Psi)
%\end{eqnarray}
After unpooling, the low-pass branch signal becomes $\Phi S = \Phi\Phi^\top (f\circledast  \overline \Psi)$, so the
signals at the concatenation layer is then given by
\begin{eqnarray}
W  &=& 
 \begin{bmatrix} \hank_d(f)\Psi  &  \Phi\Phi^\top \hank_d(f)\Psi \end{bmatrix} = \begin{bmatrix} f\circledast \overline \Psi   & \Phi\Phi^\top( f\circledast \overline \Psi   ) \end{bmatrix}  
\end{eqnarray}
where the first element in $W$ comes from the by-pass connection.
%where we use $ \Phi \Phi^\top=P_{R(\Phi)}$ for the case of average pooling.
The final step of recovery can be then represented by:
\begin{eqnarray}
\hat f &=& \hank_d^\dag \left(W \begin{bmatrix} \tilde\Psi_1^\top \\ \tilde\Psi_2^\top \end{bmatrix} \right)  \notag \\
&=&  \hank_d^\dag (\hank_d(f)\Psi\tilde\Psi_1^\top) + \hank_d^\dag\left(  \Phi\Phi^\top \hank_d(f)\Psi\tilde\Psi_2^\top\right)  \notag \\
&=& \frac{1}{d} \sum_{i=1}^q \left( f \circledast \overline\psi_i \circledast \tilde\psi_i^1 +
  \Phi\Phi^\top(f \circledast \overline\psi_i) \circledast \tilde\psi_i^2\right) \label{eq:fright}
\end{eqnarray}
where  
the last equality comes from \eqref{eq:lowrank1} and \eqref{eq:recon1} in  Lemma~\ref{lem:calculus}.
Note that  this does not guarantee the PR because the low frequency component
$ \Phi\Phi^\top(f\circledast \psi_i)$ is contained  in both terms of \eqref{eq:fright}; so the low-frequency component is overly emphasized, which is believed to be the main source
of smoothing.
%Accordingly,
% for  a given $f$, 
%it is difficult to find the local filter $\psi_i,\theta_i$ and $\xi_i$ such that  \eqref{eq:fright} becomes equal to $f$.
%Thus, the signal recovery condition by a frame expansion is not satisfied.
%We conjecture that this is the main source of spurious artifacts in U-net reconstruction.

\subsection{Proposed multi-resolution analysis}

To address the limitation of U-net, here we propose a novel multi-resolution analysis using wavelet non-local basis.
As discussed before,  at the first layer, we are   interested in learning  $\Psi^{(1)}$  and $\tilde\Psi^{(1)}$  such that 
\begin{eqnarray*}
\hank_{d_{(1)}}(f)= \Phi^{(1)}C^{(1)}\tilde \Psi^{(1)\top} ,\quad \mbox{where} \quad C^{(1)}:=\Phi^{(1)\top}\hank_{d_{(1)}}(f) \Psi^{(1)}
\end{eqnarray*}
%where $d_{(1)}$ denotes the first layer convolutional filter length.
%For simplicity we assume that the nonlocal basis $\Phi$ is orthonormal. %, but biorthogonal transform can be used.
%However, to guarantee the PR condition,  $\Phi^{(1)\top}\hank(f)$ needs to be sufficiently sparse, which is not the case in practice.
For MRA,  we  decompose the nonlocal
orthonormal basis $\Phi^{(1)}$ into the low and high frequency subbands, i.e.
$$\Phi^{(1)} = \begin{bmatrix} \Phi_{low}^{(1)}  & \Phi_{high}^{(1)} \end{bmatrix} \ .$$
For example, if we use Haar wavelet,  
the first layer operator $\Phi_{low}^{(1)},\Phi_{high}^{(1)} \in \Rd^{n\times \frac{n}{2}}$  are given by
$$\Phi_{low}^{(1)}= \frac{1}{\sqrt{2}} \begin{bmatrix} 1 & 0 & \cdots & 0 \\ 1 & 0 & \cdots & 0 \\  0 & 1 & \cdots & 0 \\ 
 0 & 1 & \cdots & 0 \\  \vdots & \vdots & \ddots & \vdots \\ 0 & 0 & \cdots & 1  \\ 0 & 0 & \vdots & 1 \end{bmatrix} ,\quad
 \Phi_{high}^{(1)}= \frac{1}{\sqrt{2}} \begin{bmatrix} 1 & 0 & \cdots & 0 \\ -1 & 0 & \cdots & 0 \\  0 & 1 & \cdots & 0 \\ 
 0 & -1 & \cdots & 0 \\  \vdots & \vdots & \ddots & \vdots \\ 0 & 0 & \cdots & 1  \\ 0 & 0 & \vdots & -1 \end{bmatrix}
  $$
  Note that $\Phi_{low}^{(1)}$ is exactly the same as the average pooling operation in \eqref{eq:apool};
  however, unlike the pooling in U-net,  $\Phi^{(1)} = \begin{bmatrix} \Phi_{low}^{(1)}  & \Phi_{high}^{(1)} \end{bmatrix} $
  now constitutes an orthonormal basis in $\Rd^n$ thanks to $\Phi_{high}^{(1)} $. % which satisfy the PR condition.
  We also define the approximate signal $C_{low}^{(1)}$ and the detail signal $ C_{high}^{(1)}$:
 \begin{eqnarray*}
 C_{low}^{(1)} &:=&  \Phi_{low}^{(1)\top}\hank_{d_{(1)}}(f) \Psi^{(1)} = \Phi_{low}^{(1)\top} ( f \circledast \overline \Psi^{(1)}) \\
  C_{high}^{(1)} &:=&  \Phi_{high}^{(1)\top}\hank_{d_{(1)}}(f) \Psi^{(1)} = \Phi_{high}^{(1)\top} ( f \circledast \overline \Psi^{(1)}) 
 \end{eqnarray*}
  such that
 \begin{eqnarray*}
 C^{(1)} = \Phi^{(1)\top}\hank_{d_{(1)}}(f) \Psi^{(1)} =  %\begin{bmatrix} \Phi_{low}^{(1)\top}\hank(f) \Psi^{(1)} \\  \rho(\Phi_{high}^{(1)\top}\hank(f) \Psi^{(1)}) \end{bmatrix} =
  \begin{bmatrix} C_{low}^{(1)} \\ C_{high}^{(1)} \end{bmatrix}
 \end{eqnarray*}
 Note that this operation corresponds to the local filtering followed by non-local basis matrix multiplication as shown in the red block of Fig.~\ref{fig:structure}(a).
Then, at the first layer, we have the following decomposition:
\begin{eqnarray*}
\hank_{d_{(1)}}(f)=
\Phi^{(1)}C^{(1)}\tilde \Psi^{(1)\top} = \Phi_{low}^{(1)}C_{low}^{(1)}\tilde \Psi^{(1)\top} +  \Phi_{high}^{(1)}C_{high}^{(1)}\tilde \Psi^{(1)\top} 
\end{eqnarray*}

  \begin{figure}[!tb] 
\center{\includegraphics[width=16cm]{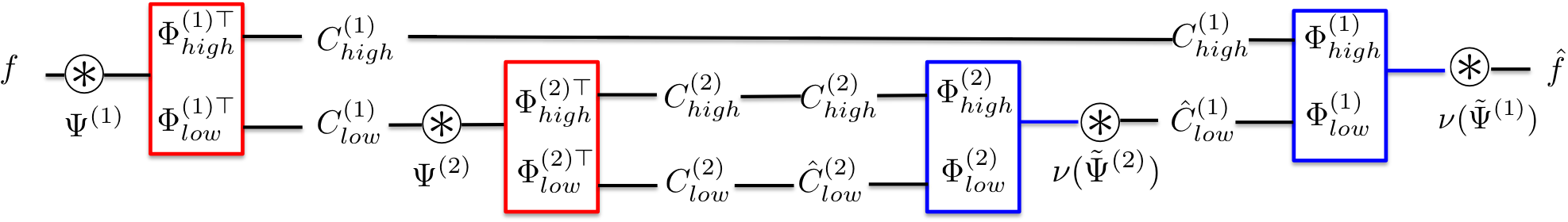}}
\centerline{\mbox{(a)}}
\center{\includegraphics[width=12cm]{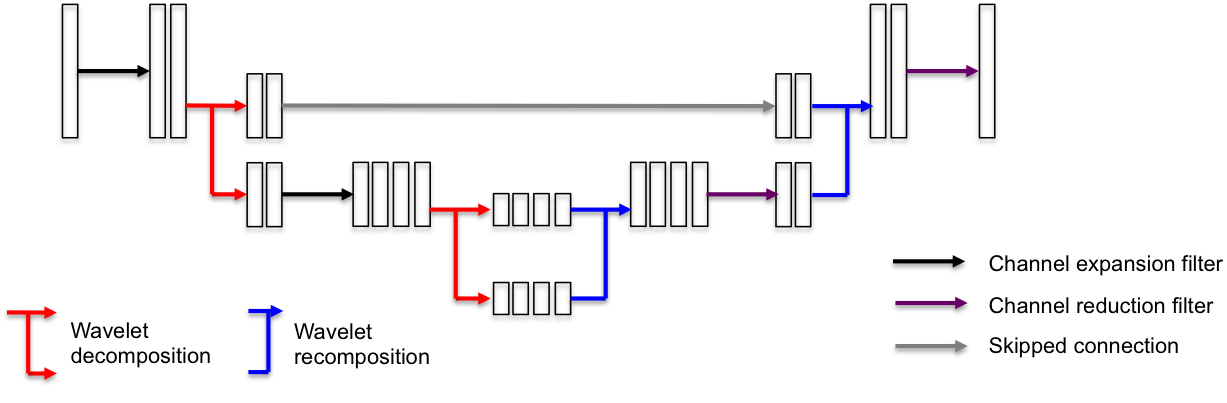}}
\centerline{\mbox{(b)}}
\caption{(a) Proposed multi-resolution analysis of  deep convolutional framelets. Here, $\circledast$ corresponds to the convolution operation;
the red and blue blocks correponds to the encoder and decoder blocks, respectively. (b) An example of multi-resolution deep convolutional framelet decomposition with a length-2 local filters. 
}
\label{fig:structure}
\end{figure}

At the second layer,  we proceed similarly using the approximate signal $C_{low}^{(1)}$. More specifically, we are interested in using orthonormal non-local bases:
$\Phi^{(2)}=\begin{bmatrix} \Phi_{low}^{(2)} & \Phi_{high}^{(2)} \end{bmatrix}$,
where $\Phi_{low}^{(2)}$ and $\Phi_{high}^{(2)}$ transforms the approximate signal $C_{low}^{(1)} \in \Rd^{n/2\times d_{(1)}}$ to  low and high bands, respectively (see Fig.~\ref{fig:structure}(a)):
\begin{eqnarray*}
\hank_{d_{(2)}|p_{(2)}}(C_{low}^{(1)})= \Phi_{low}^{(2)}C_{low}^{(2)}\tilde \Psi^{(2)\top} + \Phi_{high}^{(2)}C_{high}^{(2)}\tilde \Psi^{(2)\top},
\end{eqnarray*}
where $p_{(2)}=d_{(1)}$ denotes the number of Hankel blocks in \eqref{eq:ehank},  $d_{(2)}$ is the second layer convolution filter length, and 
\begin{eqnarray}\label{eq:2nd}
C_{low}^{(2)}  &:=& \Phi_{low}^{(2)\top}\hank_{d_{(2)}|p_{(2)}}(C_{low}^{(1)}) \Psi^{(2)} =  
\Phi_{low}^{(2)\top} (C_{low}^{(1)}\circledast \overline \Psi^{(2)} ) \notag\\
 C_{high}^{(2)} &:=& \Phi_{high}^{(2)\top}\hank_{d_{(2)}|p_{(2)}}(C_{low}^{(1)}) \Psi^{(2)} = \Phi_{high}^{(2)\top} (C_{low}^{(1)}\circledast \overline\Psi^{(2)} )
\end{eqnarray}
Again,  $\Phi_{low}^{(2)}$ corresponds to the standard average pooling operation.
Note that we need a lifting operation to an extended Hankel matrix with $p_{(2)}=d_{(1)}$ Hankel blocks in \eqref{eq:2nd}, because the first layers generates
$p_{(2)}$ filtered output which needs to be convolved with $d_{(2)}$-length filters in the second layer.

Similarly, the approximate signal needs further processing from the following layers.
In general, for $l=1,\cdots,L$,  we have
\begin{eqnarray*}
\hank_{d_{(l)}|p_{(l)}}(C_{low}^{(l-1)})= \Phi_{low}^{(l)}C_{low}^{(l)}\tilde \Psi^{(l)\top} +  \Phi_{high}^{(l)}C_{high}^{(l)}\tilde \Psi^{(l)\top}  ,  
\end{eqnarray*}
where
\begin{eqnarray*}
 C_{low}^{(l)} &:=& \Phi_{low}^{(l)\top}\hank_{d_{(l)}|p_{(l)}}(C_{low}^{(l-1)}) \Psi^{(l)} =  \Phi_{low}^{(l)\top} (C_{low}^{(l-1)} \circledast \overline\Psi^{(l)}) \\
   C_{high}^{(l)}&:=& \Phi_{high}^{(l)\top}\hank_{d_{(l)}|p_{(l)}}(C_{low}^{(l-1)}) \Psi^{(l)} =  \Phi_{high}^{(l)\top} (C_{low}^{(l-1)} \circledast \overline\Psi^{(l)}) 
   \end{eqnarray*}
where  $p_{(l)}$ denotes the dimension of local basis at $l$-th layer.
This results in $L$-layer deep  convolutional framelets using Haar wavelet.  %We can even split the framelets coefficients

The multilayer implementation of convolution framelets now results in an interesting encoder-decoder deep network structure
as shown in Fig.~\ref{fig:structure}(a), where the red and blue blocks represent encoder and decoder blocks, respectively.
In addition,  Table~\ref{tbl:dimension} summarizes the dimension of the $l$-th layer matrices.  
More specifically, with the ReLU,  the encoder parts are given as follows:
\begin{eqnarray*}
\begin{cases}
 \quad C_{low}^{(1)} = \rho\left(\Phi_{low}^{(1)\top}	\left( f\circledast \overline \Psi^{(1)}\right)\right), &   C_{high}^{(1)} =\Phi_{high}^{(1)\top}\left( f\circledast \overline \Psi^{(1)}\right) \\
 \quad C_{low}^{(2)} =  \rho\left(\Phi_{low}^{(2)\top}\left(C_{low}^{(1)}\circledast \overline \Psi^{(2)}\right)\right), &   C_{high}^{(2)} =\Phi_{high}^{(2)\top}\left(C_{low}^{(1)}\circledast \overline \Psi^{(2)}\right) \\
 &\vdots  \\
   \quad C_{low}^{(L)} = \rho\left(\Phi_{low}^{(L)\top}\left(C_{low}^{(L-1)} \circledast \overline \Psi^{(L)}\right)\right), &   C_{high}^{(L)} =\Phi_{high}^{(L)\top}\left(C_{low}^{(L-1)} \circledast \overline \Psi^{(L)}\right)
% C^{(L)}= &\rho(\Phi^{(L)\top}\hank_{p^{(L-1)}}(C^{(L-1)}) \Psi^{(L)})
\end{cases}
\end{eqnarray*}
%which corresponds to multi-channel filtering  with local filter $\Psi^{(l)} \in \Rd^{p_{(l)}d_{(l)} \times p_{(l)}d_{(l)}}$ followed by the
%non-local transform $\Phi_{high}^{(l)\top}$ or $\Phi_{low}^{(l)\top}$.
%For the case of high band signal,
%a ReLU operation follows as  in neural network (see Fig.~\ref{fig:encoder}).
%Note that  the last layer needs a ReLU for low frequency band as well to impose the matrix PR condition.
On the other hand,  the decoder part is given by
\begin{eqnarray}\label{eq:decoder}
\begin{cases}
%\hat f &=&\hank^\dag\left(\Phi^{(L)}C^{(L)}\tilde\Psi^{(L)\top}\right)\\
  \hat C_{low}^{(L-1)}
 &=
\rho\left(\hank_{d_{(L)}|p_{(L)}}^{\dag}\left(\Phi^{(L)}\hat C^{(L)}\tilde\Psi^{(L)\top}\right)\right) =
 \rho \left( \left(\Phi^{(L)}\hat C^{(L)}\right) \circledast \nu(\tilde\Psi^{(L)\top})\right)\\
% \hat C_{low}^{(L-2)} %&= \rho\left(\hank_{p^{(L-2)}}^\dag\left(\Phi_{low}^{(L-1)}\hat C_{low}^{(L-1)}\tilde\Psi^{(L-1)\top}+ \Phi_{high}^{(L-1)} C_{high}^{(L-1)}\tilde\Psi^{(L-1)\top}\right)\right)\\
%&=\rho\left(\hank_{d_{(L-1)}|p_{(L-1)}}^{\dag}\left(\Phi^{(L-1)}\hat C^{(L-1)}\tilde\Psi^{(L-1)\top}\right)\right)  =
% \rho \left( \left(\Phi^{(L-1)}\hat C^{(L-1)}\right) \circledast \nu(\tilde\Psi^{(L-1)\top})\right) \\
&\vdots\\
\hat C_{low}^{(1)}&=  
\rho\left(\hank_{d_{(2)}|p_{(2)}}^{\dag}\left( \Phi^{(2)}\hat C^{(2)}\tilde\Psi^{(2)\top}\right)\right) =
 \rho \left( \left(\Phi^{(2)}\hat C^{(2)}\right) \circledast \nu(\tilde\Psi^{(2)\top})\right)
\\
% \rho\left(\hank_{p^{(1)}}^\dag\left(\Phi_{low}^{(2)}\hat C_{low}^{(2)}\tilde\Psi^{(2)\top}+\Phi_{high}^{(2)} C_{high}^{(2)}\tilde\Psi^{(2)\top}\right)\right)\\
%\quad \hat f&=  \hank^\dag\left(\Phi_{low}^{(1)}\hat C^{(1)}\tilde\Psi^{(1)\top}+ \Phi_{high}^{(1)} C_{high}^{(1)}\tilde\Psi^{(1)\top}\right)
 \hat f&=  \hank_{d_{(1)}}^{\dag}\left( \Phi^{(1)}\hat C^{(1)}\tilde\Psi^{(1)\top}\right) =
 \rho \left( \left(\Phi^{(1)}\hat C^{(1)}\right) \circledast \nu(\tilde\Psi^{(1)\top})\right)
\end{cases}
\end{eqnarray}
 where  $\nu(\Psi)$ is defined in \eqref{eq:tau} and we use
 \begin{eqnarray*}\label{eq:Chat}
 \Phi^{(l)}\hat C^{(l)}\tilde\Psi^{(l)\top} &= \Phi_{low}^{(l)}\hat C_{low}^{(l)}\tilde\Psi^{(l)\top}+ \Phi_{high}^{(l)} \hat C_{high}^{(l)}\tilde\Psi^{(l)\top}   \ .
 \end{eqnarray*}
% since  $\hat C^{(l)} = \begin{bmatrix} \hat C_{low}^{(l)} \\ C_{high}^{(l)} \end{bmatrix}$.
where  we could further process high frequency components as
\begin{eqnarray}
 \hat C_{high}^{(L)} =   C_{high}^{(L)} \circledast  H^{(L)}
\end{eqnarray}
for some filter $H^{(L)}$, and 
$\hat C_{low}^{(L)}$ is the decoded low frequency band
from $(L-1)$-th resolution layer, which can be further processed with additional filters. % (except that $\hat C_{low}^{(L)}= C_{low}^{(L)}$.
%Note that ReLU is only used for the low-frequency branch to enforce the consistency between the encoder and decoder.

\begin{table}[!hbt]
\begin{center}
\caption{The  nomenclature and dimensions of the matrices at the $l$-th layer MRA using deep convolutional framelet.
Here, $d_{(l)}$ denotes the filter length, and $p_{(l)}=p_{(l-1)}d_{(l-1)}$  with $p_{(0)}=d_{(0)}=1$ refers to the number of Hankel block.}
\begin{tabular}{c|c|c}
\hline
Name &  Symbol & Dimension\\
\hline
Non-local basis & $\Phi^{(l)}$  &  $\frac{n}{2^{l-1}}\times \frac{n}{2^{l-1}}$\\
Low-band non-local basis & $\Phi_{low}^{(l)}$  &  $\frac{n}{2^{l-1}}\times \frac{n}{2^{l}}$\\
High-band non-local basis & $\Phi_{high}^{(l)}$  &  $\frac{n}{2^{l-1}}\times \frac{n}{2^{l}}$\\
Local basis & $\Psi^{(l)}$ & $p_{(l)}d_{(l)}\times p_{(l)}d_{(l)}$ \\
Dual local basis & $\tilde\Psi^{(l)}$ & $p_{(l)}d_{(l)}\times p_{(l)}d_{(l)}$  \\
Signal and its estimate & $C^{(l)}, \hat C^{(l)} $   & $\frac{n}{2^{l-1}} \times p_{(l)}d_{(l)}$ \\
Approximate signal and its estimate & $C_{low}^{(l)},\hat C_{low}^{(l)}$   & $\frac{n}{2^{l}} \times p_{(l)}d_{(l)}$ \\
Detail signal & $C_{high}^{(l)}$   & $\frac{n}{2^{l}} \times p_{(l)}d_{(l)}$ \\
Hankel lifting of low-band signals  &  $\hank_{d_{(l)}|p_{(l)}}(C_{low}^{(l-1)})$  &  $\frac{n}{2^{l-1}} \times p_{(l)}d_{(l)}$ \\
\hline
\end{tabular}
\label{tbl:dimension}
\end{center}
\end{table}

Fig.~\ref{fig:structure}(b) shows the overall structure of multi-resolution analysis with convolutional framelets when length-2 local filters are used.
Note that the structure is quite similar to U-net structure \cite{ronneberger2015u}, except for the high pass filter pass.
This again confirms a close relationship between the deep convolutional framelets and deep  neural networks.
%Accordingly,  based on the analysis \eqref{eq:fright} for U-net, we can also add concatenation layer in our convolutional framelets to further improve
%the performance.
%This will be explained later in our 2-D implementation.
%In the next section, we will further reveal the close relationship between them. % by introducing nonlinearity.

%\subsection{Extension to 2-D Problems}

%-------------------------------------------------------------------------
\section{Experimental Results}

In this section, we investigate  various inverse problem applications  of deep convolutional framelets, including image denoising,
 sparse view CT reconstruction, and inpainting. In particular,
we will focus on our novel multi-resolution deep convolutional framelets using Haar wavelets.
For these applications, our multi-resolution deep convolutional framelets should be extended to 2-D structure,  so the architecture
in Fig.~\ref{fig:structure}(b) should be modified.  The resulting  architectures are illustrated in Figs.~\ref{fig:proposed}(b)(c).
The 2-D U-Net, which is used as one of  our reference networks,  is shown in  Fig.~\ref{fig:proposed}(a).
Note that the U-Net in Fig.~\ref{fig:proposed}(a)  is exactly the same as the reconstruction network for inverse problems in
  Jin et al \cite{jin2016deep}  and Han et al \cite{han2016deep}.

More specifically, the networks are composed of convolution layer, ReLU, and skip connection with concatenation. Each stage consists of four $3\times3$ convolution layers  followed by ReLU, except for the final stage and the last layer. The final stage is composed of two $3\times3$ convolution layers for each low-pass and high-pass branch, and the last layer is $1\times1$ convolution layer. In Figs.~\ref{fig:proposed}(b)(c), we use  
the 2-D Haar wavelet decomposition and recomposition instead of the standard pooling and unpooling. 
Thus, at each pooling layer,  the wavelet transform generates four subbands: HH, HL, LH, and LL bands.
Then, the LL band is processed using convolutional layers followed by another wavelet-based pooling layer.
As shown in  Figs.~\ref{fig:proposed}(b)(c), the channels are doubled after wavelet decomposition. Therefore, the number of convolution kernels increases from 64 in the first layer to 1024 in the final stage. 
In Fig.~\ref{fig:proposed}(b)
we have additional filters for the high-pass branches of the highest  layer,
whereas  Fig.~\ref{fig:proposed}(c) only has skipped connection.
The reason why we do not include an additional  filter in Fig.~\ref{fig:proposed}(c) is to verify that the improvement  over U-net
is not from the additional filters in the high-pass bands but rather comes from wavelet-based non-local basis.
For a fair comparison with  U-net structure  in Fig.~\ref{fig:proposed}(a),  our proposed network in Figs.~\ref{fig:proposed}(b)(c) also have concatenation layers that stack all
subband signals before applying filtering.

%In the following, we will provide comparative studies of our architectures with U-Net architecture in Fig.~\ref{fig:proposed}(a) which was used in recent work by Jin et al \cite{jin2016deep} for inverse problems.
Although Figs.~\ref{fig:proposed}(b)(c) and Fig.~\ref{fig:proposed}(a) appear similar, there exists fundamental differences due to the
additional high-pass connections. In particular,  in Figs.~\ref{fig:proposed}(b)(c), there exist skipped connections at the
HH, HL, and LH subbands, whereas U-Net structure in Fig.~\ref{fig:proposed}(a) does not have high pass filtering before the bypass connection.
Accordingly,  as explained in Section~\ref{sec:limitation}, the U-Net does not satisfy the frame condition, emphasizing the low-pass components.
On the hand, our networks in Figs.~\ref{fig:proposed}(b)(c) do satisfy the frame condition so we conjecture that high pass signals can be better
recovered by the proposed network. Our numerical results in the following indeed provide empirical evidence of our  claim.

  \begin{figure}[!hbt] 
\center{\includegraphics[width=14cm]{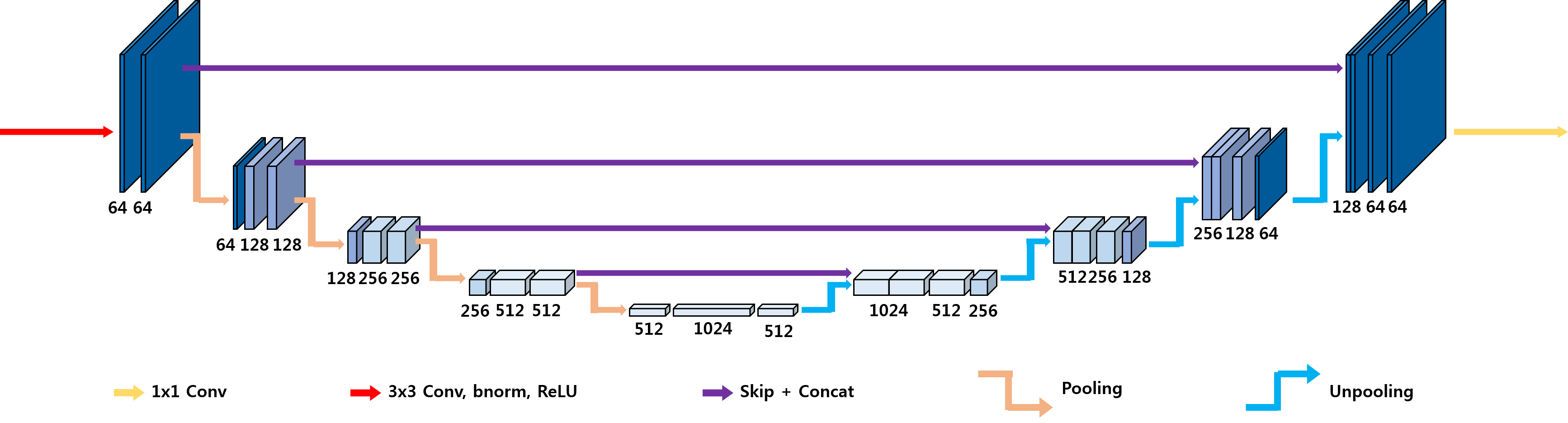}}
\centerline{\mbox{(a)}}
\center{\includegraphics[width=14cm]{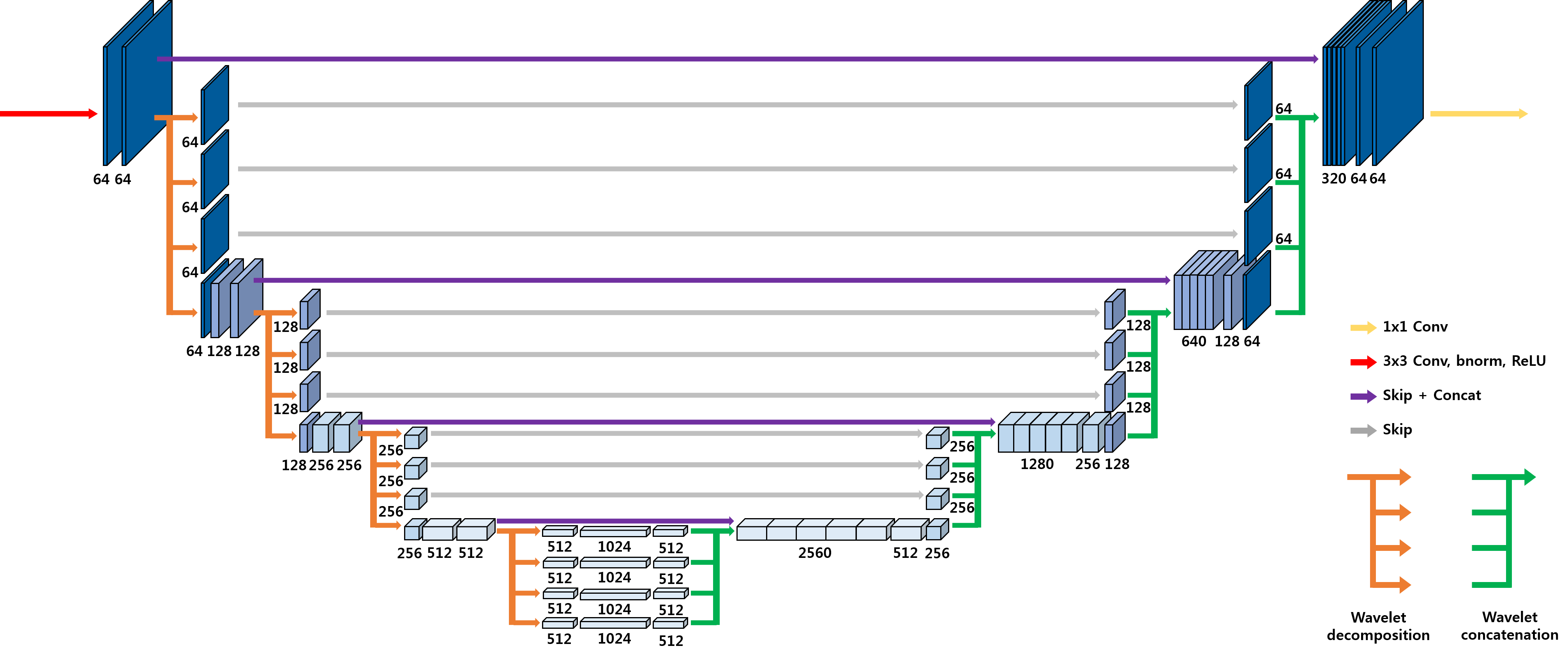}}
\centerline{\mbox{(b)}}
\center{\includegraphics[width=14cm]{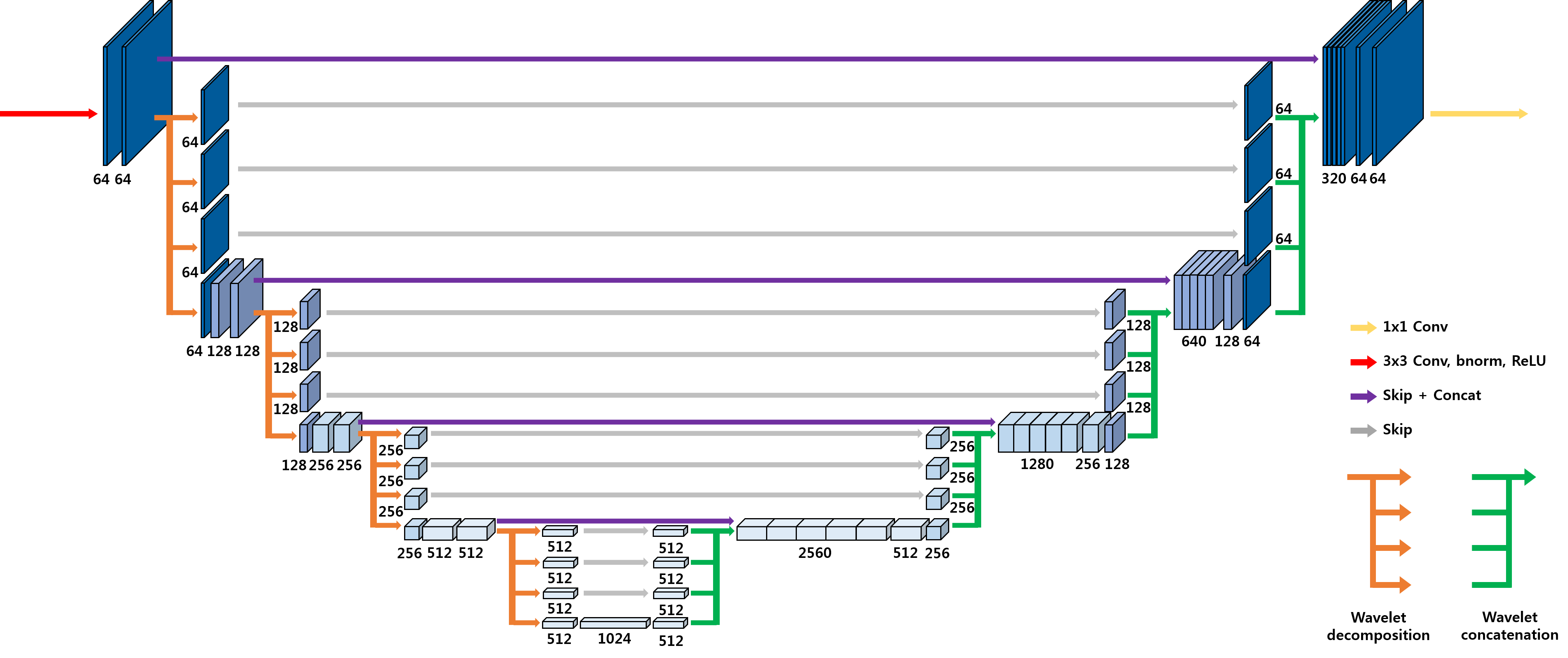}}
\centerline{\mbox{(c)}}
\caption{(a) U-Net structure used in \cite{jin2016deep}, which is used for our comparative studies. Proposed multi-resolution deep convolutional framelets structures
  for (b) our denoising, and in-painting experiments, and  for (c) sparse-view CT reconstruction, respectively.
}
\label{fig:proposed}
\end{figure}

\subsection{Image denoising}
\label{subsec:denoising}
%{T}{o} evaulate the performance of the proposed deep convolutional framelet, we first
%apply the network for denoising task.  
Nowadays, deep CNN-based algorithms have achieved great performance in image denoising \cite{bae2016beyond, zhang2017beyond}. In this section, we will therefore show that the proposed multi-resolution deep convolutional framelet outperforms the standard U-Net  in denoising task.
Specifically, 
 the proposed network and other networks were trained to learn the noise pattern similar to the existing work \cite{zhang2017beyond}.  Then, the noise-free image can be obtained by subtracting the estimated noises.

For training and validation, DIVerse 2K resolution images (DIV2K) dataset \cite{Agustsson_2017_CVPR_Workshops} was used to train the proposed network
in Fig.~\ref{fig:proposed}(b). Specifically, 800 and 200 images from the dataset were used for training and validation, respectively.  The noisy input images were generated by adding Gaussian noise of $\sigma = 30$. To train the network with various noise patterns, the Gaussian noise was re-generated in every epoch during training. 
%The corresponding label data was set to the difference between the noise-free image and the input data. Furthermore, the training data were augmented by flippint and rotating the image. The overall images were rescaled so that the pixel values of images were within $[0, 1]$. 

The proposed network was trained by Adam optimization \cite{kingma2014adam} with the momentum $\beta_1 = 0.5$. The initial learning rate was set to 0.0001, and it was divided in half at every 25 iterations, until it reached around 0.00001. The size of patch was $256\times256$, and 8 mini-batch size was used. The network was trained using 249 epochs. The proposed network was implemented in Python using TensorFlow library \cite{abadi2016tensorflow} and trained using a GeForce GTX 1080. The Gaussian denoising network took about two days for training.

The standard U-net structure in Fig.~\ref{fig:proposed}(a) was used for the baseline network for comparison.
In addition, RED-Net \cite{mao2016image} was used as another baseline network for comparison. For a
fair comparison, RED-Net was implemented using the identical hyperparameters.
More specifically, the number of filter channels at the finest scale was 64 and we use $3\times 3$ filters.
Furthermore, we used 8 mini-batsh size. 
These networks were trained under the same conditions. To evaluate the trained network, we used Set12, Set14, and BSD68, and the peak signal-to-noise ratio (PSNR) and structural similarity (SSIM) index \cite{wang2004image} were calculated for a quantitative evaluation. The PSNR is used to measure the quality of the reconstructed image, which is defined as
\begin{eqnarray}
	PSNR %&=& 10 \cdot \log_{10} \left(\dfrac{MAX_{Y}^2}{MSE(\widehat{X}), Y)})\right) \\
		 &=& 20 \cdot \log_{10} \left(\dfrac{MAX_{Y}}{\sqrt{MSE(\widehat{X}), Y)}}\right), 
\label{eq:psnr}		 
\end{eqnarray}
where $\widehat{X}$ and $Y$ denote the reconstructed image and noise-free image (ground truth), respectively. $MAX_{Y}$ is the maximum value of noise-free image.
SSIM is used to measure the similarity between original image and distorted image due to deformation, and it is defined as
\begin{equation}
	SSIM = \dfrac{(2\mu_{\widehat{X}}\mu_{Y}+c_1)(2\sigma_{\widehat{X}Y}+c_2)}{(\mu_{\widehat{X}}^2+\mu_{Y}^2+c_1)(\sigma_{\widehat{X}}^2+\sigma_{Y}^2+c_2)},
\end{equation}
where $\mu_{M}$ is a average of $M$, $\sigma_{M}^2$ is a variance of $M$ and $\sigma_{MN}$ is a covariance of $M$ and $N$. 
To stabilize the division, $c_1=(k_1R)^2$ and $c_2=(k_2R)^2$ are defined in terms of $R$, which is the dynamic range of the pixel values. We followed the default values of $k_1 = 0.01$ and $k_2 = 0.03$.

Table \ref{table:table_denoising} shows the quantitative comparison of denoising performance.
%Although U-Net shows slightly better performance for Set14 and BSD68 with Gaussian noise $\sigma =15$, the proposed networks can provide better performance in the edge components of images due to the high-pass branch of the network. As shown in Fig. \ref{fig:noise_15}, it is possible to easily capture the detail of hat in Lena image from the reconstructed image using the proposed network, whereas that part of image was smudged in the result from U-Net.
The proposed network was superior to U-Net and RED-Net in terms of PSNR for all test datasets with Gaussian noise $\sigma = 30$. Specifically, edge
structures were smoothed out  by the standard U-Net and RED-Net, whereas
 edge structures were quite accurately recovered by the proposed network, as shown in Fig.~\ref{fig:noise_30}. This is because of the additional high-pass branches in the proposed network,  which make the image detail  well recovered. These results confirm that imposing the frame condition for the non-local basis
is useful in recovering high resolution image as predicted by our theory.
%proposed network can provide better or comparable performance for the image denoising task, which is one of the general and major inverse problems.

\begin{table}[h]
\renewcommand{\arraystretch}{1.3}
\label{table:table_denoising}
\caption{Performance comparison in the PSNR/SSIM index for different data sets in the noise removal tasks from Gaussian noise with $\sigma=30$.}
\centering
\resizebox{.5\paperwidth}{!}{
\begin{tabular}{|c|c|c|c|c|}
\hline
%Dataset($\sigma$) & set12 (15) & set14 (15) & BSD68 (15) & Average (15) & set12 (30) & set14 (30) & BSD68 (30) & Average (30) & \\ \hline 
%U-Net & 32.5816/0.8887 & \textbf{31.9751/0.8791} &  \textbf{32.0019/0.8847} & \textbf{32.0019/0.8847} & 28.9395/0.8118 & 27.9911/0.7764 & 27.7314/0.7749 & 28.2207/0.7877 \\ \hline
	Dataset ($\sigma$)	&  Input	& RED-Net \cite{mao2016image} & U-Net & Proposed \\  \hline \hline

Set12 (30)	& 18.7805/0.2942	& 28.7188/0.8194 & 28.9395/0.8118   & \textbf{29.5126}/\textbf{0.8280}  \\ \hline 
Set14 (30)	&	18.8264/0.3299 & 28.4994/\textbf{0.7966} & 27.9911/0.7764   & \textbf{28.5978}/0.7866  \\ \hline
BSD68 (30)	& 18.8082/0.3267 & 27.8420/\textbf{0.7787}	& 27.7314/0.7749  & \textbf{27.8836}/0.7761  \\ \hline

\end{tabular}}
\end{table}

\begin{figure}[t] 
\center{\includegraphics[width=16cm]{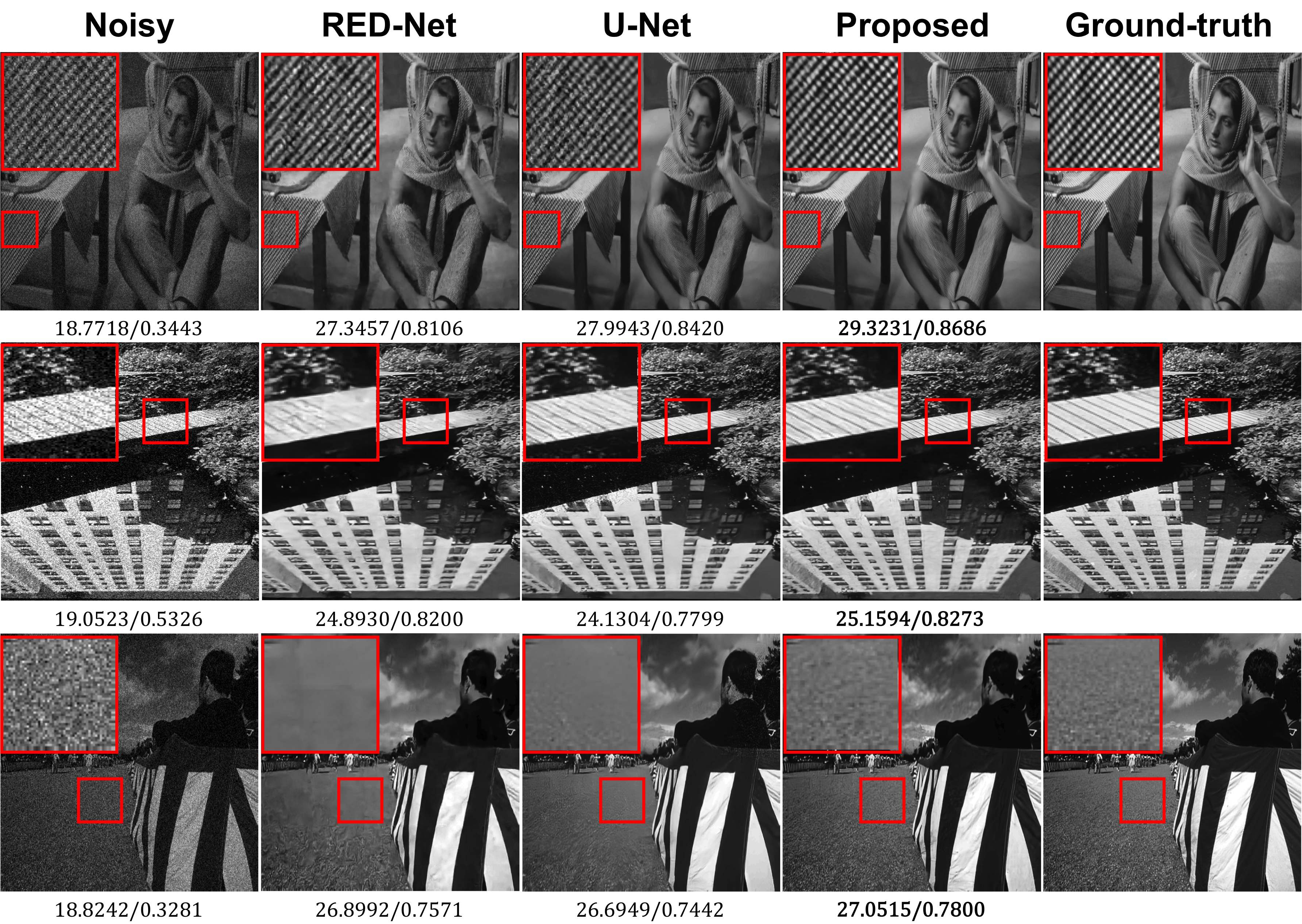}}
\caption{Denoising results for U-Net, RED-Net \cite{mao2016image}, and the proposed deep convolutional framelets from Gaussian noise with $\sigma = 30.$ 
}
\label{fig:noise_30}
\end{figure}

\subsection{Sparse-view CT Reconstruction}

%\subsubsection{Methods}
In X-ray CT,
due to the potential risk of radiation exposure, 
the main research thrust is  to reduce the  radiation dose.  Among various approaches for  low-dose CT, sparse-view CT is a recent proposal that reduces the
radiation dose by reducing the number of projection views \cite{sidky2008image}. 
However, due to the insufficient 
projection views,  standard reconstruction using the
 filtered back-projection (FBP) algorithm exhibits severe streaking artifacts that are globally distributed.  Accordingly,  researchers have extensively employed 
compressed sensing approaches   \cite{donoho2006compressed} that minimize  the total variation (TV) or other sparsity-inducing penalties under the data fidelity  \cite{sidky2008image}. 
These approaches are, however, computationally very expensive due to the repeated applications of projection and back-projection during iterative update steps.

Therefore,   the main goal of this experiment is to apply the proposed network for sparse view CT reconstruction such that  it outperforms the
existing approaches in its computational speed as well as reconstruction quality. 
%However, a direct application of conventional CNN architecture turns out inferior, because  x-ray CT images  have high texture details that are often difficult to estimate from sparse view reconstructions.
%%In another word, the manifold of  artifact-free high resolution CT image data is so complex that it is often beyond the capacity of representation power of existing deep learning architectures.
To address this,  our network is trained to learn streaking artifacts as suggested in \cite{han2016deep,han2017framing,jin2016deep}, 
using the new network architecture in Fig.~\ref{fig:proposed}(c).
%Once the streaking artifacts are estimated, an artifact-free image is then obtained by subtracting the estimated streaking artifacts.
As a training data, we used the ten patient data provided by AAPM Low Dose CT Grand Challenge (http://www.aapm.org/GrandChallenge/LowDoseCT/). 
The initial images were reconstructed by 3-D CT 2304 projection data. To generate several sparse view images, the measurements were re-generated by ${radon}$ operator in MATLAB. 
The data is composed of 2-D CT projection data from 720 views.
Artifact-free original images were generated by $iradon$ operator in MATLAB using all 720 projection views. 
The input images with streaking artifacts were generated using $iradon$ operator from 60, 120, 240, and 360 projection views, respectively. These sparse view images correspond to each donwsampling factor x12, x6,  x3, and x2.
Then, the network was trained to remove the artifacts.
%For the proposed residual learning,  the label data $Y$  were defined as the difference between the sparse view reconstruction and the full view reconstruction.

Among the ten patient data,  eight patient data were used for training, one patient data was used for validation, and a test was conducted using the remaining another patient data. 
This corresponds to  3720 and 254 slices of $512\times 512$ images for the training and validation data, respectively, and 486 slices of $512\times 512$ images for the test data.
The training data was augmented by conducting horizontal and vertical flipping.
%
%For the training data set,  we used the FBP reconstruction using 60, 90, 120, 180, 240 and 360 projection views simultaneously as input $X$ and the difference between the full view (720 views) reconstruction and the sparse view reconstructions were used as label $Y$.
%%The network was trained using
%sparse view reconstruction results both from 48 and 96 views to  address the wide range of the view numbers.
%For the validation phase, we use the FBP reconstruction using 48, 64, 96 and 192 projection views as input.

As for the baseline network for comparison, we use the U-net structure in Fig.~\ref{fig:proposed}(a)   and a single resolution CNN similar to the experimental set-up in  Jin et al \cite{jin2016deep}  and Han et al \cite{han2016deep}. 
The single resolution CNN  has the same architecture with U-net in Fig.~\ref{fig:proposed}(a), except that pooling and unpooling were not used.
All these networks were trained  similarly using the same data set.
For quantitative evaluation, we use  
the normalized mean square error (NMSE) and the peak signal-to-noise ratio (PSNR).

 The proposed network was trained by stochastic gradient descent (SGD).  The regularization parameter was $\lambda = 10^{-4}$. The learning rate was set from $10^{-3}$ to $10^{-5}$ which was gradually reduced at each epoch. The number of epoch was 150. A mini-batch data using image patch was used,  and the size of image patch was $256\times256$. 
The network was implemented using MatConvNet toolbox (ver.24) \cite{vedaldi2015matconvnet} in MATLAB 2015a environment (Mathwork, Natick). We used a GTX 1080 Ti graphic processor and i7-7770 CPU (3.60GHz). The network takes about 4 days for training.
Baseline networks were trained similarly.
%In addition, single resolution CNN without pooling and channel double was used as another reference network (see \cite{han2016deep} for detailed
%architecture). The single resolution network was trained similarly.

\begin{figure*}[!bt]
    \centerline{\includegraphics[width=0.95\linewidth]{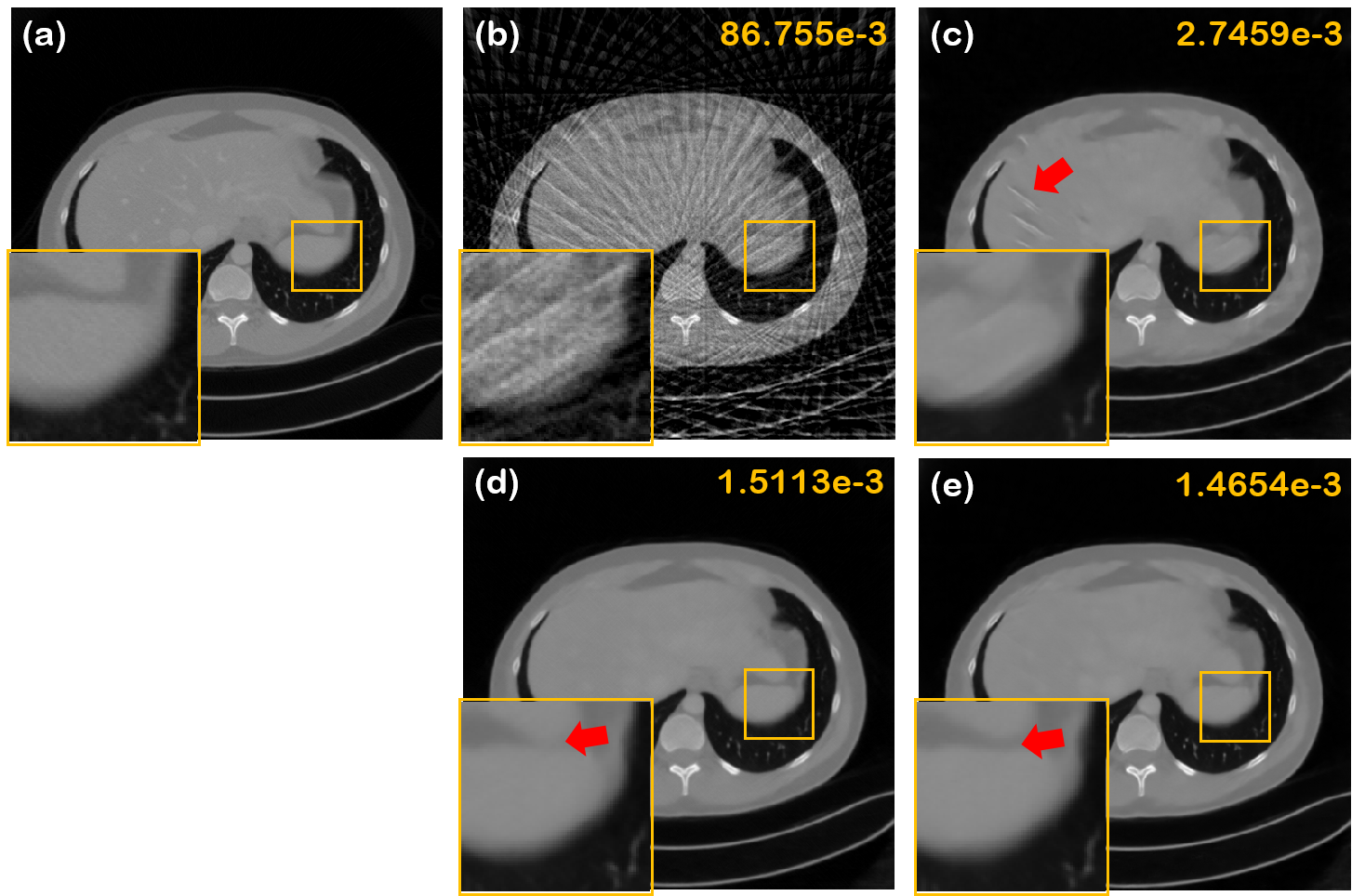}}
    \caption{CT reconstruction result comparison   from 60 views.  The number on the top
    right corner represents the NMSE values, and the red arrow refers to the area of noticeable differences. The yellow boxes denote the zoomed area.  
    FBP reconstruction results from (a) full projection views, and (b) 60 views. Reconstruction results by (c) CNN, (d) U-net, and (e) the proposed multi-resolution deep convolutional framelets.}
    \label{fig:60view}
\end{figure*}

\begin{figure*}[!bt]
    \centerline{\includegraphics[width=0.95\linewidth]{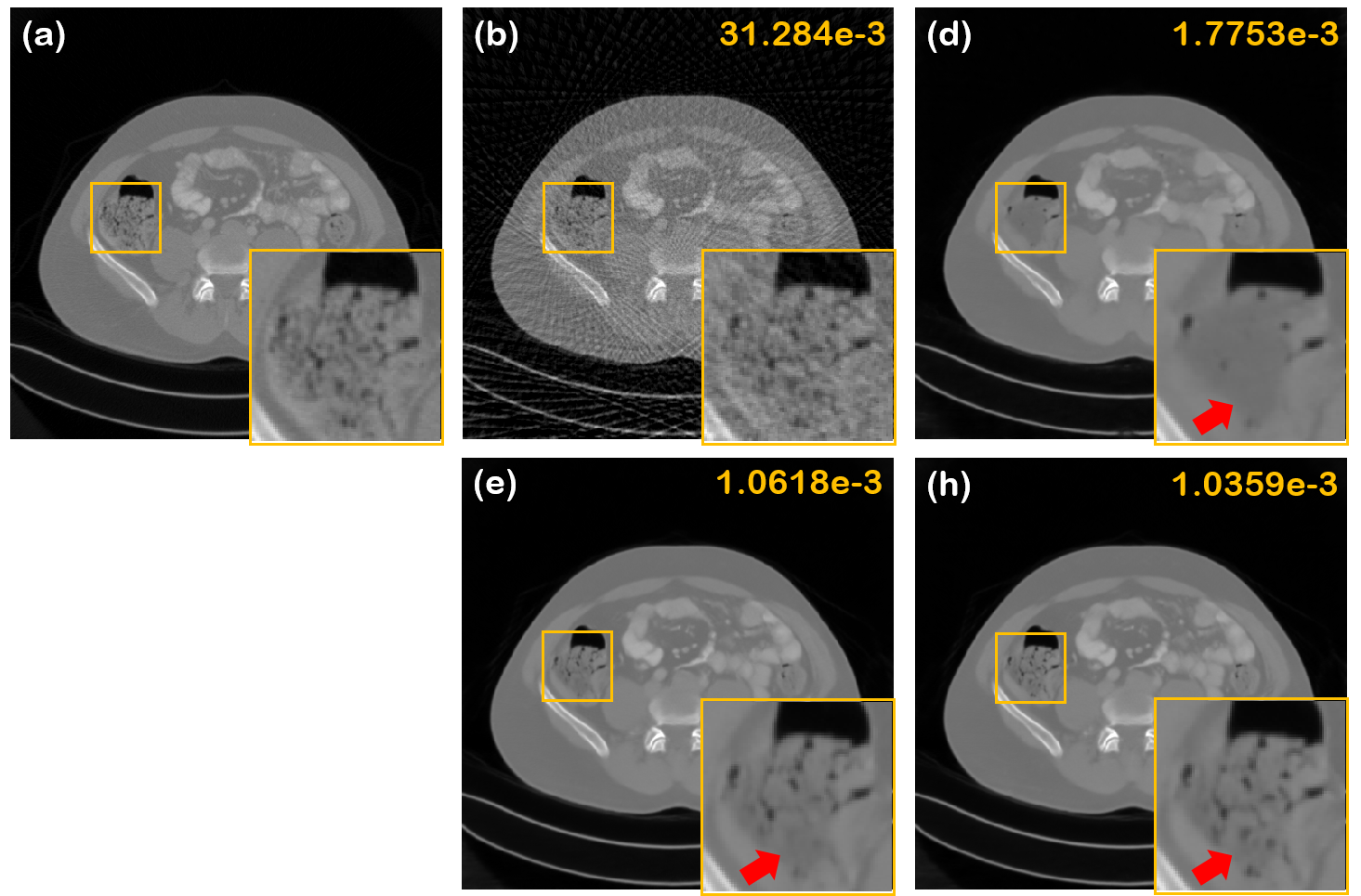}}
    \caption{CT reconstruction result comparison   from 90 views. The number on the top
    right corner represents the NMSE values, and the red arrow refers to the area of noticeable differences. The yellow boxes denote the zoomed area.  
		FBP reconstruction results from (a) full projection views, and (b) 90 views. Reconstruction results by (c) CNN, (d) U-net, and (e) the proposed multi-resolution deep  convolutional framelets.}
    \label{fig:90view}
\end{figure*}

\begin{table}[!b]
\caption{Average PSNR results comparison for  reconstruction results from various projection views and algorithms.
Here, CNN refers to the single resolution network. }
\begin{center}
\begin{tabular}{c|c|c|c|c|c|c}
\hline
\multirow{2}{*}{ PSNR [dB]} & 60 views	& 90 views 	& 120 views & 180 views & 240 views & 360 views\\
					& (x12)			& (x8)			& (x6)			& (x4)			& (x3) 			& (x2)		 \\
\hline\hline
FBP				& 22.2787		&	25.3070		& 27.4840		& 31.8291 	& 35.0178 	& 40.6892	\\
CNN			&	36.7422		& 38.5736		& 40.8814		& 42.1607		& 43.7930		& 44.8450 \\
U-net 		&	38.8122		& 40.4124		& 41.9699		& 43.0939		& 44.3413		& 45.2366	\\
Proposed 	& \textbf{38.9218}		& \textbf{40.5091}		& \textbf{42.0457}		& \textbf{43.1800}		& \textbf{44.3952}		& \textbf{45.2552 } \\
\hline
\end{tabular}
\label{tbl:ct}
\end{center}
\end{table}

%\subsubsection{Reconstruction results}

Table \ref{tbl:ct} illustrates the average PSNR values for reconstruction results from various number of projection views.
%The proposed network consistently outperforms U-net and single resolution network (CNN).
Due to the high-pass branch of the network, the deep convolutional framelets produced consistently improved images quantitatively
across all view downsampling factors.  Moreover, visual improvements from the proposed network are more remarkable.
For example, Fig. \ref{fig:60view}(a)-(e) shows reconstruction results  from 60 projection views.  Due to the severe view downsampling,
the FBP reconstruction result in Fig. \ref{fig:60view}(b) provides severely corrupted images with significant streaking artifacts.
Accordingly,  all the reconstruction results in Fig. \ref{fig:60view}(c)-(e) were not compatible to the full view reconstruction results in Fig. \ref{fig:60view}(a).
In particular, there are significant remaining streaking artifacts for the conventional CNN architecture (Fig. \ref{fig:60view}(c)), which were reduced using U-net as shown in Fig. \ref{fig:60view}(d).  However, as indicated by the arrow, some blurring artifacts were visible in Fig. \ref{fig:60view}(d).
On the other hand, the proposed network removes the streaking and blurring artifact as shown in Fig. \ref{fig:60view}(e).  Quantitative evaluation
also showed that the proposed deep convolutional framelets has the minimum NMSE values.

As for reconstruction results from larger number of projection views,
 Fig. \ref{fig:90view}(a)-(e) show reconstruction results  from 90 projection views.  All the algorithms significantly improved compared to the 60 view reconstruction.
 However, in the reconstruction results by single resolution CNN  in Fig. \ref{fig:90view}(b) and U-net in  Fig. \ref{fig:90view}(c),  the details have disappeared.
 On the other hand, most of the detailed structures were well reconstructed by the proposed deep convolutional
 framelets  as shown in  Fig. \ref{fig:90view}(e). Quantitative evaluation
also showed that the proposed deep convolutional framelets has the minimum NMSE values.
The zoomed area in Fig. \ref{fig:90view}(a)-(e) also  confirmed the findings. The reconstruction result by the deep convolutional
framelets provided very realistic image, whereas the other results are somewhat blurry.

These experimental results clearly confirmed that the proposed network is quite universal in the sense it can be used for various artifact patterns.
This is due to the network structure retaining the high-pass subbands, which automatically adapts the resolutions even though
various scale of image artifacts are present.

\subsection{Image inpainting}

{I}{m}age inpainting is a classical image processing problem whose goal is to estimate the missing pixels in an image.
Image inpainting has
 many scientific and engineering applications.
 Recently, the field of image inpainting has been dramatically changed due to the advances of
CNN-based inpainting algorithms \cite{xie2012image,pathak2016context,yeh2016semantic}.
One of the remarkable aspects of these approaches is the superior performance improvement over the existing methods
in spite of its ultra-fast run time speed.
%%Another very recent approach in this field is using the deep neural networks \cite{?}.
%For example, in \cite{?} convolutional neural network (CNN) was trained from large data sets such that
%the missing pixels can be instantaneously estimated as the CNN output. 
Despite this stellar performance,  
%  it remains unclear {\em why} these deep learning architectures work for inpainting problems.
%Moreover, unlike the usual evolution of signal processing theory  around  classical theories, 
 the link between  deep learning  and  classical inpainting approaches 
  remains poorly understood.
In this section, 
inspired by the classical frame-based inpainting algorithms  \cite{cai2008framelet,cai2012image,cai2009convergence},
we will show that the CNN-based image inpainting algorithm is indeed the first iteration of 
 { deep convolutional framelet  inpainting}, so  the inpainting performance 
 can be  improved with multiple iterations of inpainting and image update steps using CNN.

%
%To extend this idea to our deep convolution framelet based inpainting with performance guarantee,  we should define a shrinkage operator.
%Interesting, the use of the ReLU and other nonlinearities works as indeed a shrinkage operator as proven in  Papyan et al  \cite{papyan2017convolutional}.
%This implies that our deep convolutional framelet inpainting algorithm can be simply represented by
More specifically, similar to  the classical frame-based inpainting algorithms  \cite{cai2008framelet,cai2012image,cai2009convergence},
we use the update algorithm  in Eq.~\eqref{eq:deepinpainting}  derived from PR condition, which is written again as follows:
\begin{eqnarray}
f_{n+1} = \mu P_\Lambda g + (I-\mu P_\Lambda)  \Qc(f_n) \ , %\Wc^\top \Wc f_n \  .
\end{eqnarray}
where $\Qc$  is our deep convolutional framelet output.
%  is given by
%\begin{eqnarray}\label{eq:Q} 
%\Qc (f_n) = \Wc^\top  \mathrm{Prox}_{\iota_{S_\alpha}}(\Wc f_n) \  .
%\end{eqnarray}
However, unlike the existing works using tight frames \cite{cai2008framelet,cai2009convergence},  our deep convolutional
framelet does not satisfy the tight frame condition; so we relax the iteration using Krasnoselskii-Mann (KM) method \cite{bauschke2011convex}  as
summarized in Algorithm 1. 

%In Supplement, we provide the convergence proof for Algorithm 1.

\begin{algorithm}
%\footnotesize
\caption{Pseudocode implementation.}
\label{alg:Pseudocode}
\begin{algorithmic}[1]
\State Train a deep network $\Qc$ using training data set.
\State  Set $0\leq\mu< 1$ and  $0<\lambda_n<1, \forall n$.
\State Set initial guess of $f_0$ and $f_1$.
\For{$n=1,2, \dots,$ until convergence}
    \State  $q_n:= \Qc(f_n)$.
    \State $\bar f_{n+1} := \mu P_\Lambda g + (I-\mu P_\Lambda) q_n$
    \State $f_{n+1} := f_n + \lambda_n (\bar f_{n+1}- f_n )$
%    \State 
    \EndFor
\end{algorithmic}
\end{algorithm}

\begin{figure*}[!hbt]
\centering
\includegraphics[width=7cm]{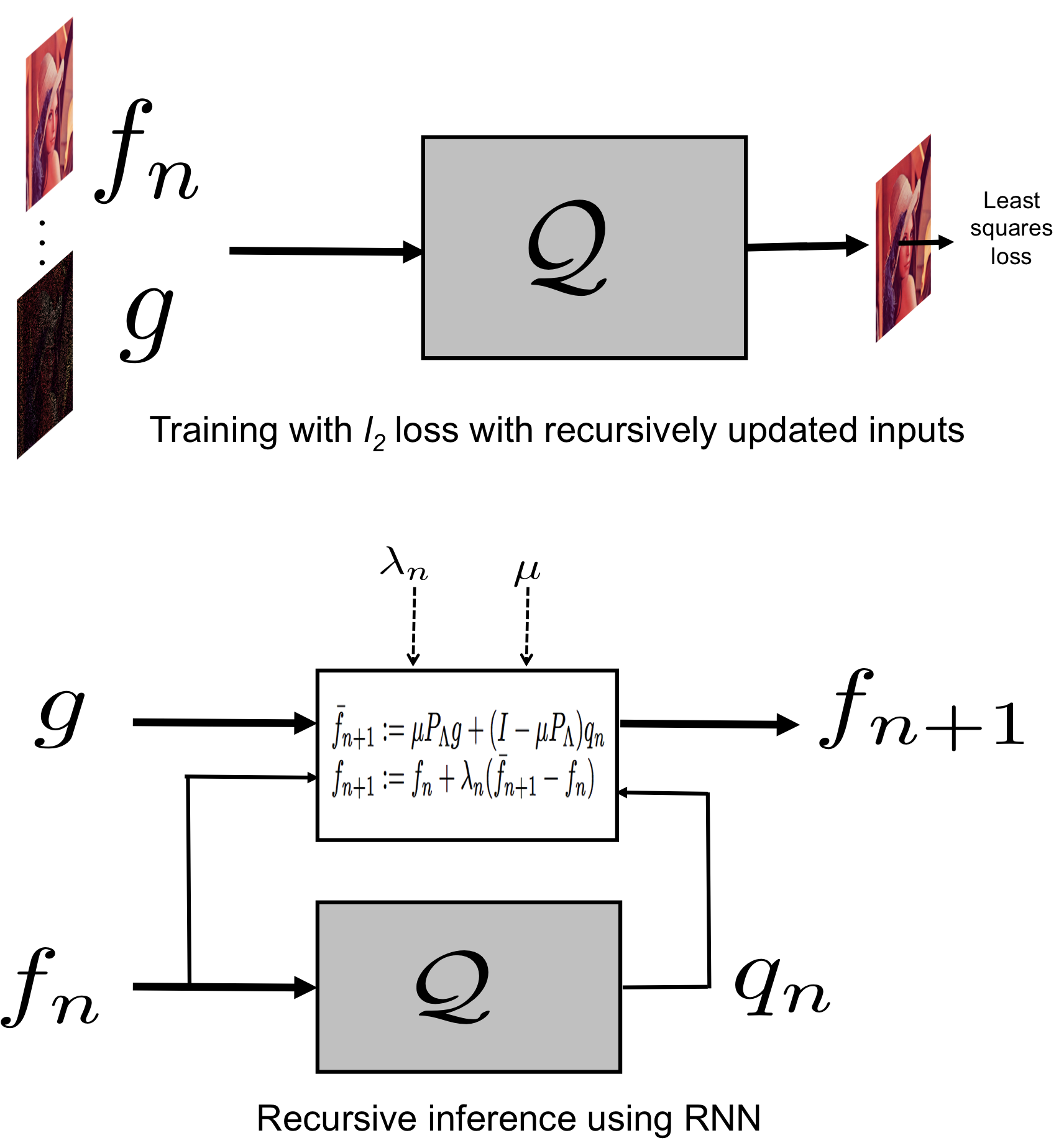}
%\centerline{\mbox{(a)}\hspace{8.5cm}\mbox{(b)}}
\caption{Proposed RNN architecture using deep convolutional framelets for training and inference steps.}%The sampled spectrum of a patch can be represented in (c).}
\label{fig:RNN}
\end{figure*}

Note that the resulting inpainting algorithm assumes the form of the recursive neural network (RNN), because the CNN
output is used as the input for the CNN for another iteration. The corresponding inference step based on Algorithm 1 is illustrated in Fig.~\ref{fig:RNN}. As for the CNN building block of the proposed RNN,
the multi-resolution deep convolution framelets in Fig.~\ref{fig:proposed}(b) is used.

\begin{figure*}[!tb]
\centering
\includegraphics[width=16cm]{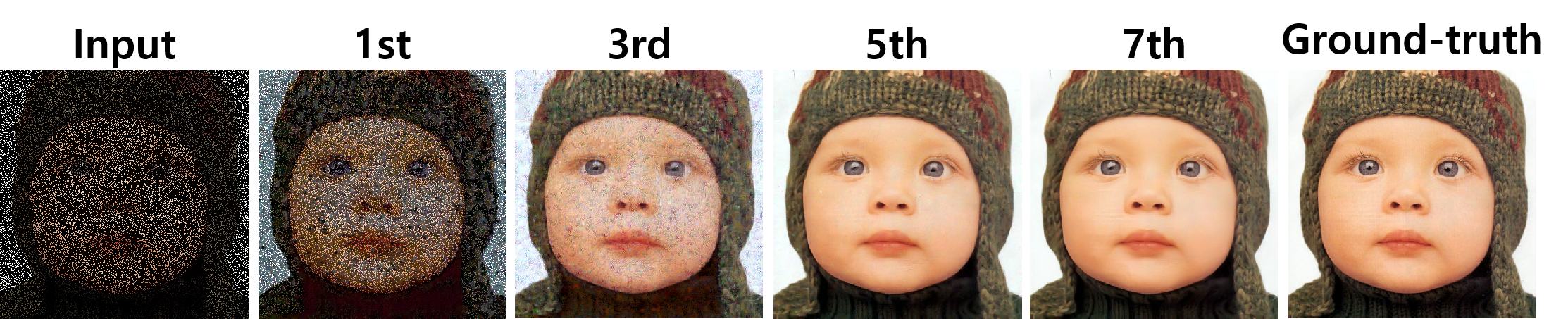}
\caption{The proposed RNN results along iteration.}
\label{fig:rnn_result}
\end{figure*}

\begin{figure}[!tb]
\centering
\includegraphics[width=5cm]{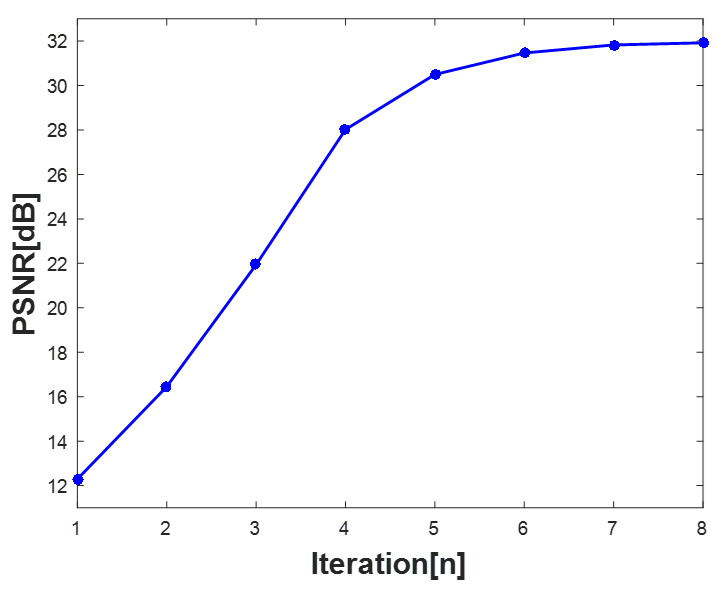}
\caption{PSNR versus iteration for the results in Fig.~\ref{fig:rnn_result}.}
\label{fig:psnr}
\end{figure}

We performed inpainting experiments using randomly sub-sampled images.
We used DIV2K dataset \cite{Agustsson_2017_CVPR_Workshops} for our experiments. 
Specifically,  800 images from the database were used for training,  and 200 images was used for the validation. In addition, the training data was augmented by conducting horizontal and vertical flipping and rotation. For the inpainting task of random sub-sampled images, 75$\%$, 80$\%$, and 85$\%$ of pixels in images were randomly removed from the images
for every other epoch during the training. %The center-missing images were generated by removing 25$\%$ of the middle part of images. 
All images were rescaled to have values between 0 and 1. %The missing region in the center-croppted images was first filled by constant mean value before 
For training, Adam optimization \cite{kingma2014adam} with the momentum $\beta_1 = 0.5$ was  used. The learning rate for the generators was set to 0.0001, and it was divided in half every 50 iterations, until it reached around 0.00001. 
%Regarding the adversarial discriminator training, the learning rate was set to 0.0001. The weights of discriminator were only clipped within the range of [-0.01, 0.01]. For the number of generator update $K_g$ and one of discriminator update $K_d$, we set them as 1, i.e. ($K_g$=1, $K_d$=1).
The size of patch was $128\times128$.  We used 32 mini-batch size for training of the random missing images. 

Since our network should perform inferences from intermediate reconstruction images,  the network should be trained with respect to the intermediate results.
Therefore, this training procedure is implemented using multiple intermediate results as inputs 
as shown in Fig.~\ref{fig:RNN}.  In particular, we trained the network according to multiple stages. In the stage 1, we trained the network using the initial dataset, $D_{init}$, which is composed of missing images and corresponding label images. After the training of network using $D_{init}$ converged, the input data for network training was replaced with the first inference result $f_1 = \Qc_k(f_0)$ and the label images, where $\Qc_k$ is $k$-th trained network and $f_k$ denotes the $k$-th inference result. That is, at the $k$-th stage, the network was trained to fit the $k$-th inference result $f_k= \Qc_l(f_{k-1})$ to the corresponding label data. 
%The iterative updatate at the inference stage follows Algorhtm \ref{alg:Pseudocode} with trained network $\Qc$.

The proposed network was implemented in Python using TensorFlow library \cite{abadi2016tensorflow} and trained using a GeForce GTX 1080ti. 
The training time of inpainting network for randomly missing data was about six days.

%\subsection{Network training under context loss}

\begin{figure*}[!bt]
\centering
\includegraphics[width=16cm]{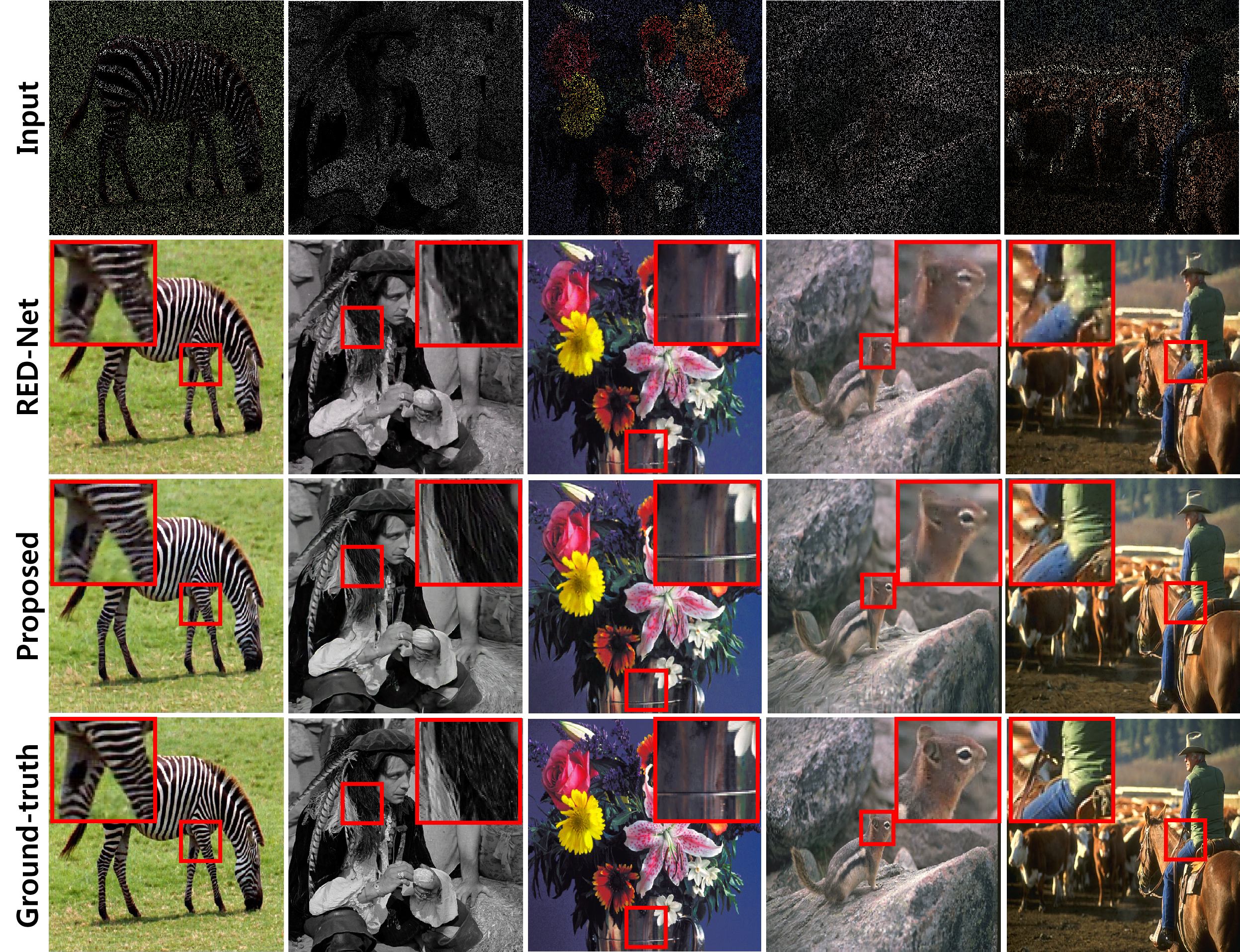}
\caption{Image inpainting results for input images with $\%80$ missing pixels in random locations. }
\label{fig:comp1}
\end{figure*}

To evaluate the trained network, we used Set5, Set14 and BSD100 dataset for testing. Fig.~\ref{fig:rnn_result} shows typical image update from the proposed RNN structure. As iteration goes on, the images are gradually improved, as the CNN works as a
image restoration network during the RNN update.  The associated PSNR graph in Fig.~\ref{fig:psnr} confirms that the algorithm
converges after six RNN update steps.
As a  reference network for comparison, we compared our algorithm with  RED-Net \cite{mao2016image}. For a fair comparison, we trained RED-Net using the same $3\times 3$ filter
 and the same number of channels of 64, which were also used for our proposed method. As shown in Fig.~\ref{fig:comp1}, the proposed method can restore the details and edges of images much better than RED-Net. 
 In addition, the PSNR results in Table~\ref{table:table_inpaint} for 80\% missing pixel images
  indicated that the proposed method outperformed RED-Net in Set5 and Set14 data and was comparable for BSD100 datasets.

\begin{table}[!bt]
\renewcommand{\arraystretch}{1.3}
\label{table:table_inpaint}
\caption{Performance comparison in terms of PSNR/SSIM index for various dataset in the inpainting task for 80$\%$ missing images.}
\centering
\resizebox{.5\paperwidth}{!}{
\begin{tabular}{|c|c|c|c|c|}
\hline
%Dataset($\sigma$) & set12 (15) & set14 (15) & BSD68 (15) & Average (15) & set12 (30) & set14 (30) & BSD68 (30) & Average (30) & \\ \hline 
%U-Net & 32.5816/0.8887 & \textbf{31.9751/0.8791} &  \textbf{32.0019/0.8847} & \textbf{32.0019/0.8847} & 28.9395/0.8118 & 27.9911/0.7764 & 27.7314/0.7749 & 28.2207/0.7877 \\ \hline
	Dataset 	&  Input	& RED-Net & Proposed \\  \hline \hline

Set5	& 7.3861/0.0992	 &	28.0853/0.9329 & \textbf{28.4792}/\textbf{0.9368}\\ \hline
Set14 	& 6.7930/0.0701  &	24.4556/0.8545 & \textbf{25.7447}/\textbf{0.8649}\\ \hline
BSD100  & 7.8462/0.0643	 &  \textbf{25.5847}/0.8510& 25.4588/\textbf{0.8530}\\ \hline

\end{tabular}}
\end{table}

%In contrast to  the existing  context-based  inpainting networks for  one-time generation of  missing image parts such as context encoder (CE) \cite{pathak2016context},
While most of the existing inpainting networks are based on
feed-forward network  \cite{xie2012image,pathak2016context,yeh2016semantic}, our theory of the deep convolutional framelets leads to a recursive neural network (RNN) that gradually improved  the image with
CNN-based image restoration, making the algorithm more accurate for various missing patterns.
These results confirmed that the theoretical framework of deep convolutional framelets is promising in designing new deep learning
algorithms for inverse problems.

\section{Discussions}

We also investigate whether our theory can answer current theoretical issues and intriguing empirical findings from machine learning community.
Amazingly, our theoretical framework gives us many useful insights.

\subsection{Low rankness of the extended Hankel matrix}

In our theory, we showed that  deep convolutional framelet is closely related to the Hankel matrix decomposition, so the
multi-layer implementation of the convolutional framelet refines the bases such that maximal energy compaction can be achieved
using a deep convolutional framelet expansion.
In addition, we  have previously shown that with insufficient filter channels, the rank structure of the extended Hankel matrix in successive layers is bounded by that of the previous layers.
This perspective suggests that the energy compaction happens across layers, and this can be
investigated by the singular value spectrum of the extended Hankel matrix.

Here, we  provide  empirical evidence that the singular
value spectrum of the extended  Hankel matrix is compressed by going through more convolutional layers. For this experiment, we used
 the single-resolution CNN with  encoder-decoder architecture as shown in Fig. \ref{fig:denoise_single_resolution} for the sake of simplicity. Hyperparameters and dataset for the training were same as those introduced in  our denoising experiments.
Since energy compaction occurs at the convolutional framelet coefficients,
we have  considered only the encoder part corresponding  to the network from the first module to the fifth module. 
%The Hankel matrix was constructed as follows:
%\begin{eqnarray}\label{eq:hank_conv}
%\hank \left(\Xb\right) = \begin{bmatrix} \hank(x_1) & \hank(x_2) & \cdots & \hank(x_{64}) \end{bmatrix} ,
%\end{eqnarray}
%where $\hank(x_i)$ denotes the Hankel matrix constituted from the feature map of the $i$-th channel. 
As shown in Fig. \ref{fig:rank}, the singular value spectrum of  the extended Hankel matrix is compressed
by going through the layer.  This confirms our conjecture  that CNN
is closely related to the low-rank Hankel matrix approximation.

\begin{figure}[t] 
\center{\includegraphics[width=14cm]{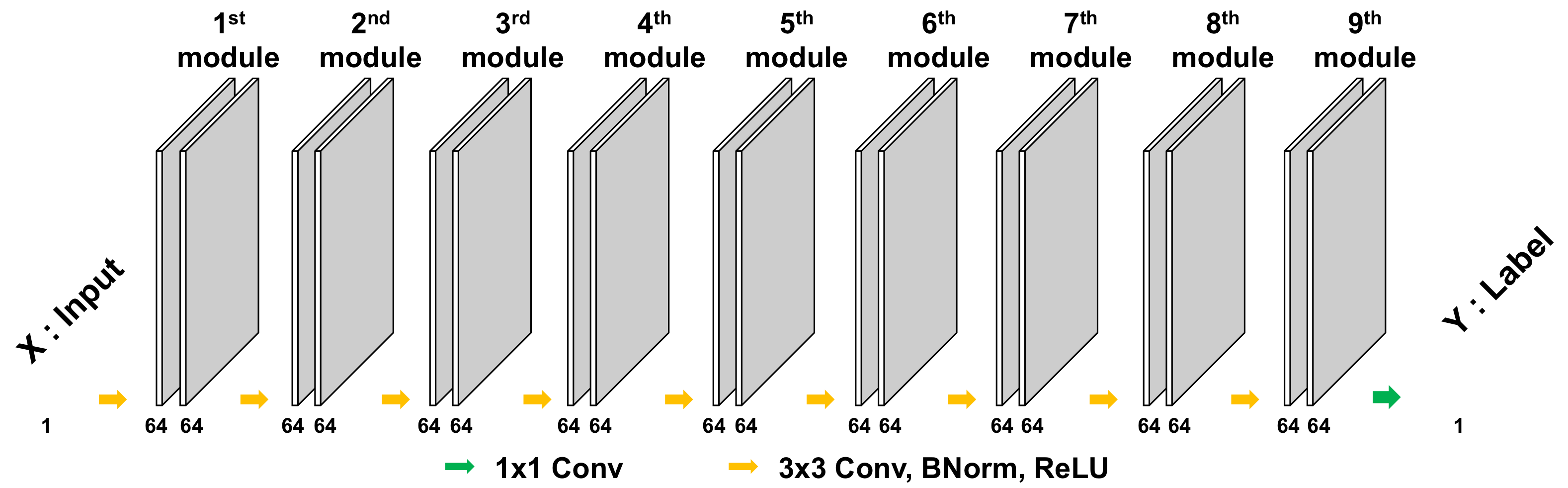}}
\caption{Single-resolution CNN architecture for image denoising.}
\label{fig:denoise_single_resolution}
\end{figure}

\begin{figure}[h] 
\center{\includegraphics[width=7cm]{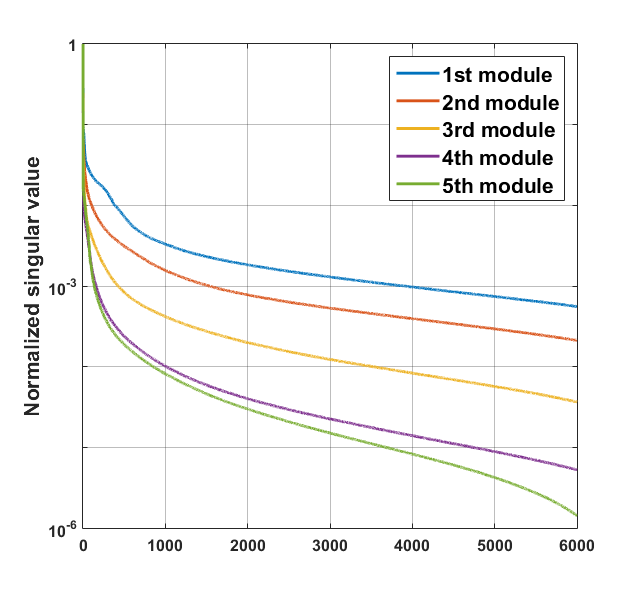}}
\caption{Normalized singular value plot of Hankel matrix constructed using the feature map from the $1st$ module to the $5th$ module.  
}
\label{fig:rank}
\end{figure}

\subsection{Insights on classification networks}

While our mathematical theory of deep convolutional framelet was derived for inverse problems, there are many important implications of our finding to general deep learning networks.
For example, we conjecture that the classification network corresponds to the encoder part of our deep convolutional framelets.  More specifically, the encoder part of the deep convolutional
framelets works for the energy compaction, so the classifier attached to the encoder can discriminate the input signals based on the  compressed energy
distributions. This is similar with the classical classifier design, where the feature vector is first obtained by a dimensionality reduction algorithm, after which support vector machine (SVM) type classifier is used.  
Accordingly, the role of residual net,  redundant channels, etc. are believed to hold for  classifier networks as well.
It is also important to note that the rank structure of Hankel matrix, which determines the energy distribution of convolutional framelets coefficients,
is translation and rotation invariant as shown in \cite{jin2015annihilating}. The invariance property was considered the most important property which gives the theoretical motivation
for Mallat's wavelet scattering  network \cite{mallat2012group,bruna2013invariant}.  Therefore,  there may be an important connection between the deep
convolutional framelets and wavelet scattering.  However, this is beyond scope 
of current paper and will be left for future research.

\subsection{Finite sample expressivity}

Another interesting observation is that the perfect reconstruction is directly related to finite sample expressivity of a neural network 
\cite{zhang2016understanding}.  Recently,  there appeared a very intriguing article providing empirical evidences that 
the  traditional statistical learning theoretical approaches fail to explain why large neural networks generalize well in practice \cite{zhang2016understanding}.  To explain
this, the authors showed that 
 simple depth-two neural networks already have perfect finite sample expressivity as soon as the number of parameters exceeds the number of data points  \cite{zhang2016understanding}. 
We conjecture that the  perfect finite sample expressivity  is closely related to the perfect reconstruction condition, saying that any finite sample size input can be reproduced perfectly using a neural network.  The intriguing link between the PR condition and finite sample expressivity needs further investigation.

\subsection{Relationship to pyramidal residual network}
Another interesting aspect of our convolutional framelet analysis is the increases of filter channels as shown in \eqref{eq:prod}.
While this does not appear to follow the conventional implementation of the convolutional filter channels,
there is a very interesting article that provides a strong empirical evidence supporting our theoretical prediction.
In the recent paper on pyramidal residual network \cite{han2016deepPy}, 
%instead of
%sharply increasing the feature map dimension at units that
%perform downsampling,
 the authors gradually increase the feature
channels across layers. This design was proven to be an effective means
of improving generalization ability. This coincides with our prediction in \eqref{eq:prod}; that is,  in order to guarantee the PR condition, the filter channel 
should increase. % regardless of the non-local transform such as pooling. 
This again suggests the theoretical potential of   the proposed deep convolutional framelets.

\subsection{Revisit to the existing deep networks for inverse problems}

Based on our theory for  deep convolutional framelets,
we now revisit the existing deep learning algorithms for inverse problems and  discuss their pros and cons. %their limitations.
 %we discuss  special cases of deep convolutional framelets, which correspond to the specific examples of  existing deep learning architectures for inverse problems.
%\subsubsection*{Wavelet residual network (WavResNet)}

By extending the work in \cite{kang2016deep},
Kang et al \cite{kang2017deep,kang2017wavelet} proposed a  {\em wavelet domain residual learning (WavResNet)} for low-dose CT reconstruction as shown in
Fig.~\ref{fig:wavresnet}.
The key idea of WavResNet is to apply the directional wavelet transform first, after which a neural network is trained such that it can learn the mapping between 
noisy input wavelet coefficients and noiseless ones  \cite{kang2017deep,kang2017wavelet}. 
%The low frequency components
%are added without processing.
In essence, this can be interpreted as a deep convolutional framelets  with the
nonlocal transform being performed first. The remaining layers are then  composed of CNN with  local filters and residual blocks.
Thanks to the global transform using directional wavelets, the signal becomes more compressed,  which is the main source of the advantages compared to the simple CNN.
Another uniqueness of the WavResNet is the concatenation layer at the ends that performs additional filters by using all the intermediate results.
This layer  performs a signal {\em boosting}  \cite{kang2017deep}. % as shown in our companion paper \cite{?}.
However, due to the lack of pooling layers,  the receptive field size is smaller than that of multi-scale network as shown in Fig.~\ref{fig:receptive_field}. Accordingly,
the architecture was better suited 
for localized artifacts from low-dose CT noise, but it is not effective 
 for removing globalized artifact patterns from sparse view CT.

%See \cite{?} for more detailed analysis of WavResNet. 

  \begin{figure}[h] 
\center{\includegraphics[width=14cm]{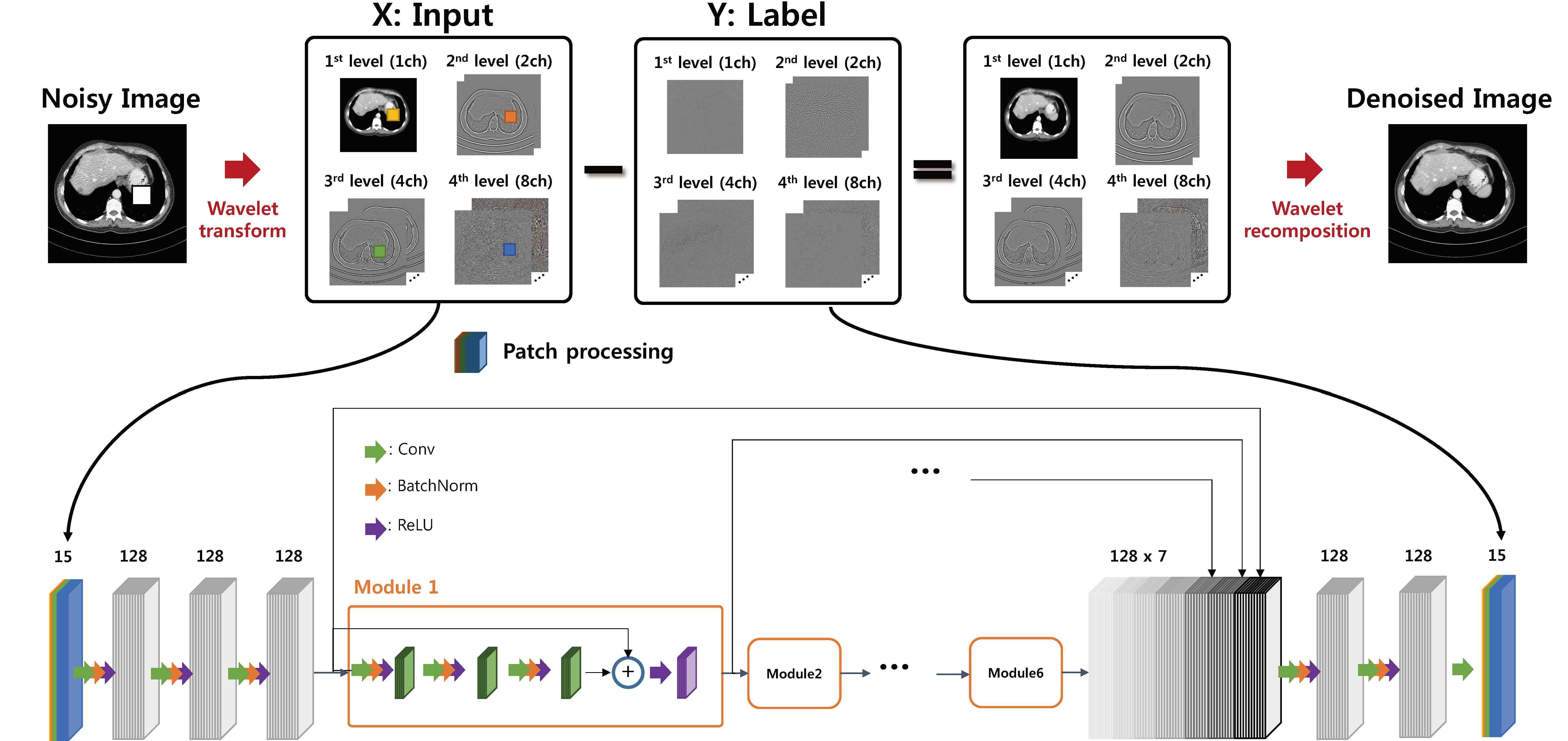}}
\caption{WavResNet architecture for low-dose CT reconstruction \cite{kang2017deep,kang2017wavelet}. 
}
\label{fig:wavresnet}
\end{figure}

%\subsubsection*{AUTOMAP}

{\em AUtomated TransfOrm by Manifold APproximation (AUTOMAP)}
 \cite{zhu2017image} is a recently proposed neural network approach for image reconstruction, which is claimed to be general for various
imaging modalities such as MRI, CT, etc.
A typical architecture is given in Fig.~\ref{fig:automap}.
This architecture  is similar to a standard CNN, except that at the first layer nonlocal basis matrix  is learned as a fully connected layer.
Moreover, the original signal domain is the measurement domain, so only local filters are followed in successive layer without additional
fully connected layer for inversion.
In theory, learning based non-local transform can be optimally adapted to the signals, so
it  is believed to be the advantageous over standard CNNs.
However, in order to use the fully connected layer as nonlocal bases,  a huge size network is required.
For example,   in order to recover $N\times N$ image, 
the number of parameters for the fully connected layer is $2N^2 \times N^2$ as shown in Fig.~\ref{fig:automap} (see  \cite{zhu2017image} for more calculation of required parameter numbers).
Thus, if one attempts to learn the CT image of size $512\times 512$ (i.e. $N=2^{9}$) using AUTOMAP, the required memory becomes $2N^4=2^{37}$, which is neither possible
to store nor to avoid any overfitting during the learning.
This is another reason  we prefer to use analytic non-local bases.
%Specifically,  this is because
% wavelet transforms do not need to be stored or learned but still provides near optimal
%nonlocal transform for piecewise smooth functions \cite{daubechies1992ten}.
However, if the measurement size is sufficiently small, the approach by AUTOMAP may be an interesting direction to investigate.

  \begin{figure}[!h] 
\center{\includegraphics[width=10cm]{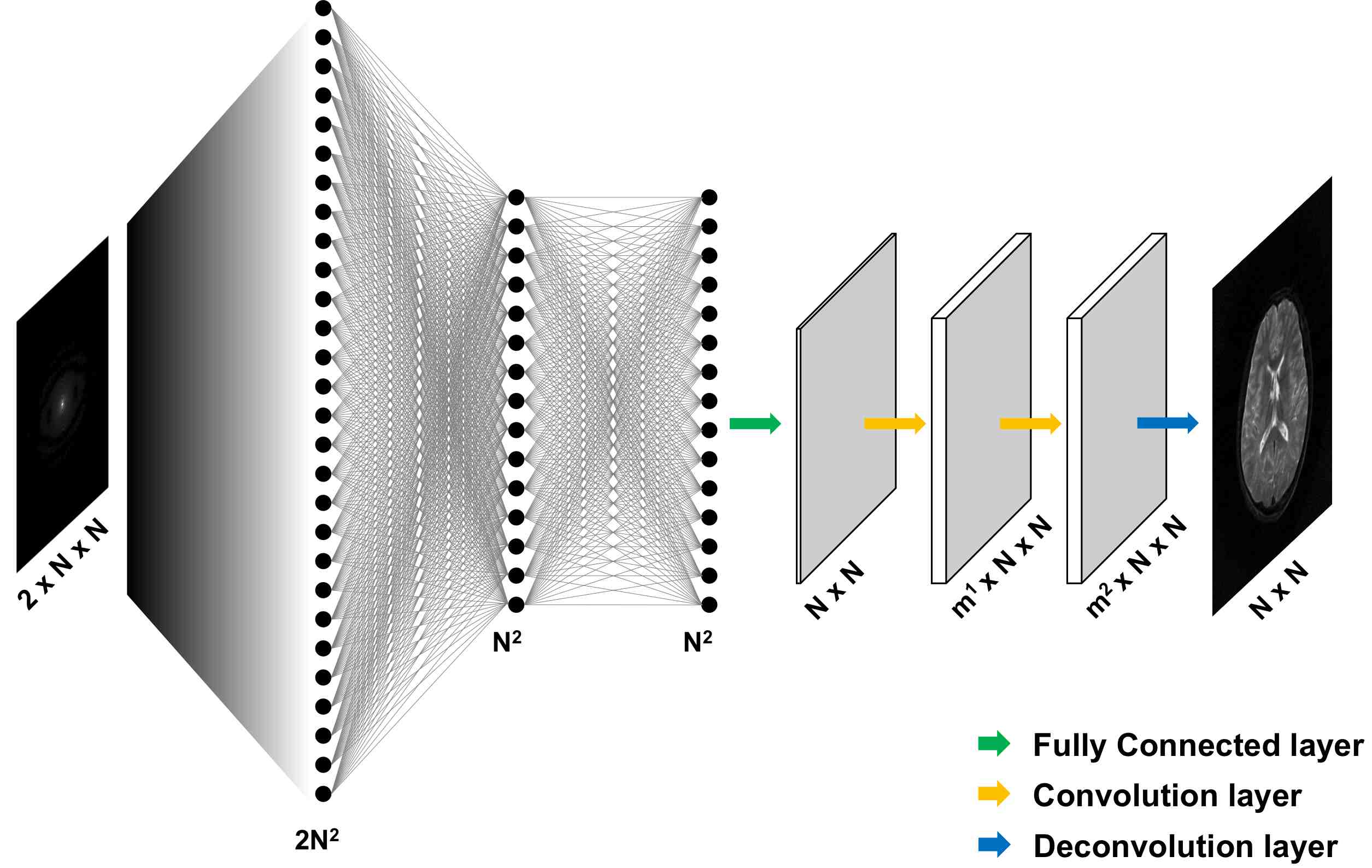}}
\caption{AUTOMAP architecture \cite{zhu2017image}. 
}
\label{fig:automap}
\end{figure}

\section{Conclusions}

In this paper, we propose a general deep learning framework called deep convolutional framelets  for inverse problems.
%Unlike the conventional deep learning approaches that are derived by trial and errors,  
The proposed network architectures
were obtained based on key fundamental theoretical advances  we have achieved. % in this paper.
First,  we show that the deep learning is closely related to the
 existing theory of  annihilating filter-based
low-rank Hankel matrix approaches  and   convolution framelets.  In particular, our theory was motivated by  the observation that when a signal is lifted to a high dimensional
Hankel matrix, it usually results in a low-rank structure. Furthermore, the lifted Hankel matrix can be decomposed using non-local and local bases.
%, which is an energy compacting representation  if the underlying Hankel matrix has a low rank structure. 
We further showed that the convolution framelet expansion can be equivalently represented as encoder-decoder convolutional
layer structure.
By extending this idea furthermore, we also derived the multi-layer convolution framelet expansion and associated encoder-decoder network.
Furthermore, we investigated the perfect reconstruction condition for the deep convolutional framelets. In particular, we showed that the perfect reconstruction is still possible
when the framelet coefficients are processed with ReLU.  
%Amazingly, many of the important aspects of the deep learning such as generalization power,
%residual blocks,  redundant channels and CReLU emerge from the PR condition under ReLU.
We also proposed a novel class of deep network using  multi-resolution convolutional framelets.

Our discovery provided a theoretical rationale for many existing deep learning architectures and components. In addition, our theoretical framework can address some of the fundamental questions that we raised at Introduction.
More specifically, we showed that the convolutional filters work as local bases and the number of channels can be determined based on the perfect reconstruction condition. Interestingly, by controlling the number of filter channels we can achieve a low-rank based shrinkage behavior.
%Secondly, we have demonstrated that the fully connected layer, as in AUTOMAP, works as data-driven non-local bases  learned from the training data. 
We further showed that ReLU can disappear when paired filer channels
with opposite polarity are available.
Another important and novel theoretical contribution is that, thanks to the lifting to the Hankel structured matrix, we can show that
the pooling and un-pooling layers actually come from  non-local bases, so they should be augmented with high-pass branches
to meet the frame condition.
Our deep convolutional framelets can also explain the role of the by-pass connection.
Finally, we have shown that the depth of the network is determined considering the intrinsic rank of the signal and the convolution filter length.
Our numerical results also showed that the proposed deep convolutional framelets can provide improved reconstruction performance
under various conditions.

One of the limitations of the current work  is that our analysis, such as PR, is based on a deterministic framework, while most of the mysteries
 of deep learning are probabilistic in nature, which apply in expectation over complex high-dimensional image distributions.
 To understand the link between the deterministic and probabilistic views of deep learning is very important, which need to be further explored 
 in the future.

\appendix
\numberwithin{equation}{section}

\section{Basic properties of Hankel Matrices}
\label{ap1}

 The following basic properties  of Hankel matrix, which are mostly novel,  % (some are  obtained from \cite{ye2016compressive} and \cite{yin2017tale})
will be useful in proving main theoretical results. % useful. 
%Here, we present basic properties of Hankel matrices which will be used for the rest of the paper.
% analysis. 
\begin{lemma}\label{lem:calculus}
For a given $f\in \Rd^n$,  let  $\hank_d(f) \in \Hc(n,d)$ denote the associated Hankel matrix.
% is refered to as, where the space
%of Hankel matrix $\Hc(n,d)$ is equipped with the matrix inner product $\langle A, B \rangle =  \tr(A^{\top}B) .$
Let $\Phi \in \Rd^{n\times m}$ and   $\Psi \in \Rd^{d\times q}$ be matrices
 whose dimensions are chosen such that $\Phi$ and $\Psi^\top$ can be multiplied to the left and right to $\hank_d(f)$.
% Moreover, 
%  $\phi_i$  and $\psi_i$ denote their $i$-th columns, respectively.
Suppose, furthermore,  that $\tilde\Phi\in \Rd^{n\times m}$ and $\tilde\Psi\in \Rd^{d\times q}$ are another set of  matrices that
match the dimension of $\Phi$ and $\Psi$, respectively.
  % of matrix $\Phi \in \Rd^{n\times m}$ and   $\Psi \in \Rd^{d\times q}$, respectively.
  Then, the following statements are true.
\begin{enumerate}
\item
Let
 \begin{eqnarray}\label{eq:Ak}
 E_k =  \frac{1}{\sqrt{d}} \hank_d (e_k) \quad \in \Hc(n,d), \quad k=1,\cdots, n
 \end{eqnarray}
where $e_k$ denotes the standard coordinate vector in $\Rd^n$, where only the $k$-th element  is one.
 Then, the set $\{E_k\}_{k=1}^n$ is the orthonormal
basis for the space of $\Hc(n,d)$.
\item %For a given $f\in \Rd^n$ and the associated Hankel matrix $\hank_d(f) \in \Hc(n,d)$, 
For a given $E_k$  in \eqref{eq:Ak}, we have
\begin{eqnarray}\label{eq:f2F}
F := \hank_d(f)   =  \sum_{k=1}^n \langle  E_k, F \rangle E_k, \quad %\notag \\ %=  \sqrt{d}\sum_{i=1}^n f[k] E_k ,
&\mbox{where} \quad & \langle E_k, \hank_d(f) \rangle = \sqrt{d} f[k].
\end{eqnarray}
\item For any vectors $u,v\in \Rd^n$ and  any Hankel matrix $F = \hank_d(f) \in \Hc(n,d)$, we have
\begin{eqnarray}\label{eq:inner}
\langle F, uv^{\top} \rangle   =  u^{\top} F v  = u^{\top} \left( f \circledast \overline v \right) = f^{\top} \left( u\circledast v \right) = \langle f, u\circledast v \rangle
\end{eqnarray}
where  $\overline v$ denotes the flipped version of the vector $v$.
\item For a given $E_k$  in \eqref{eq:Ak}, $u \in \Rd^n$, and $v\in \Rd^d$, we have %denotes the $k$-th basis of $\Hc_c(n,d)$. Then,
\begin{eqnarray}\label{eq:Akuv}
\langle E_k, uv^{\top} \rangle   = \frac{1}{\sqrt{d}}(u\circledast v)[k]. 
\end{eqnarray}
\item  A generalized inverse of  the lifting to the Hankel structure in \eqref{eq:f2F} is given by %Hankel  $\hank^\dag(B)$  for any  matrix $B \in \Rd^{n\times d}$ is  defined by
\begin{eqnarray}\label{eq:F2f}
\hank_d^\dag(B) = \frac{1}{\sqrt{d}} \begin{bmatrix} \langle  E_1, B \rangle  \\ \langle  E_2,B \rangle \\ \vdots \\ \langle E_n, B \rangle  \end{bmatrix}
\end{eqnarray}
where $B$ is any matrix in $\Rd^{n\times d}$ and $E_k$'s are defined   in \eqref{eq:Ak}.
\item For a given $C\in \Rd^{m\times q}$,
\begin{eqnarray}
\hank_d^\dag(\tilde\Phi C \tilde\Psi^{\top}) = \sum_{j=1}^q \hank_d^\dag(\tilde\Phi c _j \tilde\psi_j^{\top}) &=&  \frac{1}{d} \sum_{j=1}^q  ((\tilde \Phi c_j )\circledast\tilde \psi_j)  \ , \label{eq:recon1} \\
&=&    \frac{1}{d}   \sum_{i=1}^m \sum_{j=1}^q c_{ij}  (\tilde\phi_i \circledast \tilde \psi_j) , \quad \label{eq:invfilter}
\end{eqnarray}
where $\tilde \phi_j, \tilde \psi_j$ and $c_j$ denotes the $j$-th column of  $\tilde\Phi,\tilde \Psi$ and $C$, respectively;
 and $c_{ij}$ is the $(i,j)$ elements of $C$.
\item 
% Let 
% $\hank_{d|p}^\dag(\cdot)$ denote the generalized inverse  of an extended Hankel operator $\hank_{d|p}(\cdot)$.
%Suppose, furthermore,  
For any  $Y =\begin{bmatrix}
Y_1 & \cdots & Y_p \end{bmatrix} \in \Rd^{n\times dp}$ with $Y_i \in \Rd^{n\times d},i=1,\cdots,p$,
suppose that an operator $L$ satisfies
\begin{eqnarray}\label{eq:recon2}
L(Y) % \hank_{d|p}^\dag(Y) 
&=& \begin{bmatrix}  \hank_d^\dag(Y_1) & \cdots & \hank_d^\dag(Y_p) \end{bmatrix}  \in \Rd^{n\times p}
\end{eqnarray}
Then, $L$ is a generalized left-inverse of an extended Hankel operator $\hank_{d|p}$, i.e. $L=\hank_{d|p}^\dag$
\item We have %For $f\in \Rd^{n}$ and $\Psi,\tilde\Psi\in \Rd^{d\times q}$,
\begin{eqnarray}\label{eq:lowrank1}
\hank_d^\dag(\hank_d(f)  \Psi \tilde\Psi^{\top}) =   \frac{1}{d } \sum_{i=1}^q  (f\circledast \overline\psi_i \circledast\tilde \psi_i) \quad \in \Rd^n \ .
\end{eqnarray}
\item   Let $\Xi,~\tilde\Xi\in \Rd^{pd\times pq}$ denote any matrix with block structure:
%defined 
%with
\begin{eqnarray*}
% \hat C^{(l)} &=& \begin{bmatrix} \hat C_1^{(l)} & \cdots & \hat C_{p_{(l)}}^{(l)} \end{bmatrix},\quad  \mbox{where} \quad \hat C_{j}^{(l)} \in \Rd^{\frac{n}{2^{l-1}}\times d_{(l)}}\\
%  \Xi^{\top} = \begin{bmatrix}  \Xi_1^{\top} & \cdots &  \Xi_{p}^{\top} \end{bmatrix},\quad
%  &
   \tilde \Xi^{\top} = \begin{bmatrix} \tilde \Xi_1^{\top} & \cdots & \tilde \Xi_{p}^{\top} \end{bmatrix}&
  \quad  \mbox{where} \quad  \tilde \Xi_{j}^{\top} \in \Rd^{pq\times d }    \end{eqnarray*}
Then, we have
   \begin{eqnarray}\label{eq:lowrank2}
\hank_{d|p}^\dag(\hank_{d|p}([f_1,\cdots,f_p])  \Xi \tilde\Xi^{\top}) =  \frac{1}{d} \sum_{i=1}^{q } \sum_{j=1}^p\begin{bmatrix} 
f_j \circledast \overline\xi_{i}^j \circledast \tilde\xi_{i}^1, & \cdots, & 
f_j \circledast \overline\xi_{i}^j \circledast \tilde\xi_{i}^p  \end{bmatrix} \ .
\end{eqnarray}
where $\xi_{i}^j$ (resp. $\tilde\xi_{i}^j$) denotes the $i$-th column of $\Xi_j$  (resp. $\tilde\Xi_{j}$).
\end{enumerate}
\end{lemma}
\begin{proof}
The proof is a simple application of the definition of Hankel matrix and convolutional framelet. We prove claims one by one:
\begin{enumerate}
\item[(1)]
The proof  can be found in  \cite{ye2016compressive}.
\item[(2)] Because $\{E_k\}_{k=1}^n$ constitutes a orthonormal basis, for any $F \in \Hc(n,d)$, we have
$$F= \sum_{k=1}^n \langle E_k, F \rangle E_k.$$
Furthermore, the operator $\hank_d: f \mapsto \Hc(n,d)$ is linear, so we have
\begin{eqnarray*}
\hank_d(f) &=& \hank_d\left(\sum_{k=1}^n f[k] e_k \right) = \sum_{k=1}f[k] \hank_d(e_k) = \sqrt{d} \sum_{k=1}^n f[k] E_k
\end{eqnarray*}
where the last equality comes from \eqref{eq:Ak}. Thus, $\langle E_k, \hank_d(f) \rangle = \sqrt{d} f[k]$.
\item[(3)] The proof can be found in \cite{yin2017tale}.
\item[(4)] Using \eqref{eq:inner} and $E_k =\hank_d(e_k)/\sqrt{d}$, we have
\begin{eqnarray*}
\langle E_k, uv^{\top} \rangle  &  = & \frac{1}{\sqrt{d}}e_k^{\top}(u\circledast v) =  \frac{1}{\sqrt{d}}(u\circledast v) [k] \  .
%&=& \sqrt{d}(u\circledast v)[k]. 
\end{eqnarray*}
\item[(5)] We need to show that $\hank_d^\dag  \left(\hank_d(f)\right) = f$ for any $f=[f[1]\cdots f[n]]^T \in \Rd^n$.
\begin{eqnarray*}
\hank_d^\dag(\hank_d(f)) = \frac{1}{\sqrt{d}} \begin{bmatrix} \langle  E_1, \hank_d(f) \rangle  \\ \langle  E_2,\hank_d(f)\rangle \\ \vdots \\ \langle E_n, \hank_d(f) \rangle  \end{bmatrix}
=  \frac{1}{\sqrt{d}} \begin{bmatrix} \sqrt{d} f[1]  \\ \sqrt{d} f[2] \\ \vdots \\ \sqrt{d} f[n]  \end{bmatrix} = f 
\end{eqnarray*}
where we use $\langle E_k, \hank_d(f) \rangle = \sqrt{d} f[k]$.
\item[(6)]
 For $\tilde\Phi\in \Rd^{n\times m}$ and $C\in \Rd^{m\times q}$ and $\tilde\Psi\in \Rd^{d\times q}$,
\begin{eqnarray}%\label{eq:recon1}
\hank_d^\dag(\tilde\Phi C \tilde\Psi^{\top})  &=&   \hank_d^\dag\left(\tilde\Phi  \begin{bmatrix} c_1 & \cdots & c_q \end{bmatrix} 
\begin{bmatrix}  \tilde \psi_1^\top \\ \vdots \\ \tilde \psi_q^\top \end{bmatrix} \right)  \notag\\
&=& \sum_{j=1}^q \hank_d^\dag(\tilde\Phi c _j \tilde\psi_j^{\top})  \ .
\end{eqnarray}
Furthermore, using \eqref{eq:Akuv} and \eqref{eq:F2f}, we have
\begin{eqnarray*}
 \hank_d^\dag(\tilde\Phi c_j \tilde\psi_j^{\top})  &=& \frac{1}{\sqrt{d}}  \begin{bmatrix} \langle  E_1, \tilde \Phi c_j \tilde\psi_j^{\top} \rangle  \\ \langle  E_2,\tilde\Phi c_j \tilde\psi_j^{\top} \rangle \\ \vdots \\ \langle E_n,\tilde \Phi c_j \tilde\psi_j^{\top} \rangle  \end{bmatrix} =  \frac{1}{{d}}  \begin{bmatrix}  \left(\tilde\Phi c_j \circledast \tilde\psi_j\right)[1]  \\  \left(\tilde\Phi c_j \circledast \tilde\psi_j\right)[2]  \\ \vdots \\  \left(\tilde\Phi c_j \circledast \tilde\psi_j\right)[n]   \end{bmatrix} = \frac{1}{d}  \left(\tilde\Phi c_j \circledast \tilde\psi_j\right)
\end{eqnarray*}
Thus, we have
\begin{eqnarray*}%\label{eq:recon1}
\hank_d^\dag(\tilde\Phi C \tilde\Psi^{\top}) = \frac{1}{d} \sum_{j=1}^q  \left(\tilde \Phi c_j \circledast \tilde\psi_j\right)  \ .
\end{eqnarray*}
Finally, \eqref{eq:invfilter} can be readily obtained by noting that $$\tilde \Phi c_j = \sum_{i=1}^m \tilde \phi_i c_{ij} \quad .$$ 
\item[(7)] We need to show that  the operator $L$ defined by
\eqref{eq:recon2} satisfies the left inverse condition, i.e. we should show that
$L  \left(\hank_{d|p}(X)\right) = X$ for any $X=[x_1,\cdots, x_p]  \in \Rd^{n\times p}$.
This can be shown because we have 
\begin{eqnarray*}%\label{eq:recon2}
L(\hank_{d|p}(X)) &=& L\left(\begin{bmatrix} \hank_d(x_1) & \cdots & \hank_d(x_p) \end{bmatrix}\right) \\
%\hank^\dag(C_1) & \cdots & \hank^\dag(C_p) \end{bmatrix}  \in \Rd^{n\times p}
%\end{eqnarray*}
%where $\hank_d(a_i)\in \Rd^{n\times d}$. Thus, using \eqref{eq:recon2}, we have
%\begin{eqnarray*}%\label{eq:recon2}
%\hank_{d|p}^\dag(\hank_{d|p}(A)) 
&=& \begin{bmatrix}  \hank_d^\dag(\hank(x_1)) & \cdots & \hank_d^\dag(\hank(x_p)) \end{bmatrix}  \\
&=& \begin{bmatrix} x_1 & \cdots & x_p \end{bmatrix} = X \  ,
\end{eqnarray*}
where the first equality uses the definition of $\hank_{d|p}(X)$ and the second equality comes from the definition of $L$
and the last equality is from $\hank_d^\dag$ being the generalized inverse of $\hank_d$.
%This concludes the proof.
\item[(8)] 
 Since $f\in \Rd^{n}$ and $\Psi,\tilde\Psi\in \Rd^{d\times q}$,  we have
\begin{eqnarray*}
\hank_d^\dag(\hank_d(f)  \Psi \tilde\Psi^{\top}) &=& \frac{1}{d} \sum_{i=1}^q  \left(\hank_d(f) \psi_i \circledast\tilde\psi_i\right) = 
   \frac{1}{d} \sum_{i=1}^q  (f\circledast \overline\psi_i \circledast\tilde \psi_i) \quad \in \Rd^n \ .
\end{eqnarray*}
where  the first equality comes from  \eqref{eq:recon1} and
the last equality comes from  \eqref{eq:SISO}. %$\hank_d(f) \psi_i=f\circledast \psi_i$.
\item[(9)]  Since $\tilde \Xi^{\top} = \begin{bmatrix} \tilde \Xi_1^{\top} & \cdots & \tilde \Xi_{p}^{\top} \end{bmatrix} \in \Re^{pq\times pd}$ with
$\tilde \Xi_i^\top \in \Rd^{pq\times d}$, we have
$$\hank_{d|p}\left([f_1,\cdots,f_p]\right)  \Xi\tilde\Xi^\top= \begin{bmatrix}
\hank_{d|p}\left([f_1,\cdots,f_p]\right) \Xi \tilde \Xi_1^{\top}  & \cdots &  \hank_{d|p}\left([f_1,\cdots,f_p]\right) \Xi \tilde \Xi_p^{\top} \end{bmatrix},$$
where $\hank_{d|p}\left([f_1,\cdots,f_p]\right)   \Xi \tilde \Xi_i^{\top} \in \Rd^{n\times d}$ for $i=1,\cdots,p$. 
Thus,   we have
 % For  $[f_1,\cdots,f_p]\in \Rd^{n\times p}$ and $\Psi,\tilde\Psi\in \Rd^{pd\times q}$
%with
%\begin{eqnarray*}
%% \hat C^{(l)} &=& \begin{bmatrix} \hat C_1^{(l)} & \cdots & \hat C_{p_{(l)}}^{(l)} \end{bmatrix},\quad  \mbox{where} \quad \hat C_{j}^{(l)} \in \Rd^{\frac{n}{2^{l-1}}\times d_{(l)}}\\
%  \Psi^{\top} = \begin{bmatrix}  \Psi_1^{\top} & \cdots &  \Psi_{p}^{\top} \end{bmatrix},
%  & \tilde \Psi^{\top} = \begin{bmatrix} \tilde \Psi_1^{\top} & \cdots & \tilde \Psi_{p}^{\top} \end{bmatrix}&
%  \quad  \mbox{where} \quad \Psi_j^\top, \tilde \Psi_{j}^{\top} \in \Rd^{pd\times q }    \end{eqnarray*}
%   we have
%Suppose, furthermore,   $C =\begin{bmatrix}
%C_1 & \cdots & C_p \end{bmatrix} \in \Rd^{n\times dp}$ with $C_i \in \Rd^{n\times d},i=1,\cdots,p$.  Then, we have
%\begin{eqnarray}\label{eq:recon2}
%\hank_{d|p}^\dag(C) &=& \begin{bmatrix}  \hank_d^\dag(C_1) & \cdots & \hank_d^\dag(C_p) \end{bmatrix}  \in \Rd^{n\times p}
%\end{eqnarray}
   \begin{eqnarray*}
\hank_{d|p}^\dag(\hank_{d|p}([f_1,\cdots,f_p])  \Xi \tilde\Xi^{\top}) &=&\begin{bmatrix}
\hank_{d}^\dag\left(\hank_{d|p}\left([f_1,\cdots,f_p]\right) \Xi \tilde \Xi_1^{\top}\right)  & \cdots &  \hank_{d}^\dag\left(\hank_{d|p}\left([f_1,\cdots,f_p]\right) \Xi \tilde \Xi_p^{\top} \right)\end{bmatrix} \\
&=& \frac{1}{d} \sum_{i=1}^q  \begin{bmatrix}
\hank_{d|p}\left([f_1,\cdots,f_p]\right) \overline\xi_i \circledast \tilde \xi_{i}^1  & \cdots &  \hank_{d|p}\left([f_1,\cdots,f_p]\right) \xi_i \circledast
\tilde \xi_{i}^p \end{bmatrix} \\
&=&  \frac{1}{d } \sum_{i=1}^{q } \sum_{j=1}^p\begin{bmatrix} 
f_j \circledast \overline\xi_{i}^j \circledast \tilde\xi_{i}^1, & \cdots & 
f_j \circledast \overline\xi_{i}^j \circledast \tilde\xi_{i}^p  \end{bmatrix} \ .
\end{eqnarray*}
where the first equality comes from
 \eqref{eq:recon2} and the second equality is from \eqref{eq:recon1},  and the last equality
 is due to \eqref{eq:MIMO}.
Q.E.D.
\end{enumerate}
\end{proof}

\section{Proof of Theorem~\ref{thm:bias}}
\label{ap8}

By converting \eqref{eq:biasEnc} and \eqref{eq:biasDec} to Hankel operator forms, we have
\begin{eqnarray}
\hat Z &=& \hank_{d|p}^\dag (\tilde \Phi C \tilde \Psi^\top) + 1_{n}b_{dec}^\top \notag \\
&=&  \hank_{d|p}^\dag \left(\tilde\Phi (\Phi^\top \hank_{d|p}(Z) \Psi +  \Phi^\top 1_{n}b_{enc}^\top)\tilde \Psi^\top\right) + 1_{n}b_{dec}^\top \notag \\
&=&  \hank_{d|p}^\dag \left(\hank_{d|p}(Z) \right)  +  \hank_{d|p}^\dag \left(   1_{n}b_{enc}^\top\tilde \Psi^\top\right) + 1_{n}b_{dec}^\top  \notag\\
&=& Z +  \hank_{d|p}^\dag \left(   1_{n}b_{enc}^\top\tilde \Psi^\top\right) + 1_{n}b_{dec}^\top  \label{eq:Zcont}
\end{eqnarray}
where we use the frame condition \eqref{eq:phi0} for the third quality.
Thus, to satisfy PR, i.e. $\hat Z = Z $,we have
$$ 1_{n}b_{dec}^\top  = -  \hank_{d|p}^\dag \left(   1_{n}b_{enc}^\top\tilde \Psi^\top\right),$$
which concludes the first part.
Next, due to the block structure of $\tilde\Psi$  in \eqref{eq:blockPsi},  we have %Note that
$$b_{enc}^\top\tilde\Psi^\top = \begin{bmatrix} a_1^\top & \cdots & a_p^\top\end{bmatrix}, \quad
\mbox{where}\quad a_i = \tilde\Psi_i b_{enc}, ~i=1,\cdots, p$$
Accordingly, using \eqref{eq:recon2} and \eqref{eq:recon1} we have
\begin{eqnarray*}
 \hank_{d|p}^\dag \left(   1_{n} \begin{bmatrix} a_1^\top & \cdots & a_p^\top\end{bmatrix} \right)  &=&
 \begin{bmatrix} \hank_d(1_n a_1^\top)  & \cdots & \hank_d(1_n a_p^\top)  \end{bmatrix}\notag\\
 &=& \begin{bmatrix} 1_n \circledast a_1 & \cdots & 1_n \circledast a_p \end{bmatrix}  \\
 &=& 1_n \begin{bmatrix} 1_d^\top a_1 & \cdots & 1_d^\top a_p \end{bmatrix}
 \end{eqnarray*}
where the last equality comes from
$$1_n \circledast a_i =\hank_d(1_{n}) a_i = 1_n (1_d^\top a_i).$$
From \eqref{eq:Zcont}, we therefore have
\begin{eqnarray*}
\hat Z &=& Z + 1_n \begin{bmatrix} 1_d^\top a_1 & \cdots & 1_d^\top a_p \end{bmatrix} + 1_n b_{dec}^\top \ .
\end{eqnarray*}
Thus, to satisfy PR, i.e. $\hat Z = Z $, the decoder bias should be
$$b_{dec}^\top = - \begin{bmatrix} 1_d^\top a_1 & \cdots & 1_d^\top a_p \end{bmatrix} = -\begin{bmatrix} 1_d^\top \tilde \Psi_1 b_{enc} & \cdots & 1_d^\top \tilde \Psi_p b_{enc} \end{bmatrix} .$$
This concludes the proof.

\section{Proof of Proposition~\ref{prp:PRFourier}}
\label{ap2}

From  \eqref{eq:enc} and \eqref{eq:dec}, we have
\begin{eqnarray*}
Z 
&=& \left(\tilde \Phi C\right) \circledast \nu(\tilde \Psi)  \\
&=&  \left(\tilde \Phi  \Phi^{\top}\left(Z\circledast \overline\Psi\right) \right) \circledast \nu(\tilde \Psi) \\
&=& \left(Z\circledast \overline\Psi\right)\circledast \nu(\tilde \Psi)  \\
&=& \hank_{d|p}^\dag\left(\hank_{d|p}(Z) \Psi\tilde \Psi^\top \right)
%&=& \left(\sum_{j=1}^p \hank_d(z_j) \Psi_j \right) \circledast \nu(\tilde \Psi)  \\
 \end{eqnarray*} 
Thus,  \eqref{eq:lowrank2} informs that 
 \begin{eqnarray*}
Z :=[z_1,\cdots, z_p]  &=&  \hank_{d|p}^\dag(\hank_{d|p}([z_1,\cdots,z_p])  \Psi \tilde\Psi^{\top}) \notag\\
&=&  \frac{1}{d} \sum_{i=1}^{q } \sum_{j=1}^p\begin{bmatrix} 
z_j \circledast \overline\psi_{i}^j \circledast \tilde\psi_{i}^1, & \cdots, & 
z_j \circledast \overline\psi_{i}^j \circledast \tilde\psi_{i}^p  \end{bmatrix} \ .
\end{eqnarray*}
By taking the Fourier transform, we have
\begin{eqnarray}\label{eq:zk}
\hat z_k =
 \frac{1}{d} \sum_{i=1}^{q } \sum_{j=1}^p\hat z_j \widehat{\psi_{i}^j}^* \widehat{ \tilde\psi_{i}^k},\quad\quad k=1,\cdots, p\,
\end{eqnarray}
because the Fourier transform of the flipped signal is equal to the complex conjugate of the original signal.
Finally, \eqref{eq:zk} can be represented by a matrix representation form:
\begin{eqnarray*}
Z= Z  \frac{1}{d} \begin{bmatrix} \widehat{\psi_1^1}^* & \cdots &  \widehat{\psi_q^1}^*  \\ \vdots & \ddots & \vdots \\
\widehat{ \psi_1^p}^* & \cdots &  \widehat{ \psi_q^p}^*
\end{bmatrix}  \begin{bmatrix}  \widehat{\tilde \psi_1^1} & \cdots &  \widehat{\tilde \psi_1^p}  \\ \vdots & \ddots & \vdots \\
\widehat{\tilde \psi_q^1} & \cdots &  \widehat{\tilde \psi_q^p }
\end{bmatrix}  
\end{eqnarray*}
which is equivalent to the condition Eq.~\eqref{eq:Id}.
For the proof of \eqref{eq:prdual}, note that the PR condition \eqref{eq:encC} and \eqref{eq:decf} is for $p=1$.
Thus, Eq.~\eqref{eq:Id} is reduced to 
$$1= \frac{1}{d} \begin{bmatrix} \widehat{\psi_1^1}^* & \cdots &  \widehat{\psi_q^1}^* \end{bmatrix} 
 \begin{bmatrix}  \widehat{\tilde \psi_1^1}  \\ \vdots  \\
\widehat{\tilde \psi_q^1}   \end{bmatrix}= \frac{1}{d} \sum_{i=1}^q \widehat{\psi_i^1}^*\widehat {\tilde \psi_i^1} ,$$
which proves \eqref{eq:prdual}.
 Finally, for \eqref{eq:prorth}, note that $\tilde \psi_i =\psi_i$ for the orthonormal basis. Thus,
 \eqref{eq:prdual} is reduced to \eqref{eq:prorth}. This concludes the proof.

\section{Proof of Proposition~\ref{prp:exp}}
\label{ap3}

We prove this by mathematical induction.  At $l=1$, the input signal is $f\in \Rd^n$, so we need $\Phi^{(1)\top}\hank_{d_{(1)}}(f)\Psi^{(1)}$ to obtain
the filtered signal $C^{(1)}$.  Since $\hank_{d_{(1)}}(f)\in \Rd^{n\times d_{(1)}}$, the dimension of the local basis matrix should be
$\Psi^{(1)} \in \Rd^{d_{(1)}\times q_{(1)}}$ with $q_{(1)} \geq d_{(1)}$ to satisfy the frame condition \eqref{eq:ri0}. 
%Thus, the number of the input channel at $l=2$ becomes $p_{(2)}=q_{(1)}\geq d_{(1)}$.
%% this generates minimum 
%%$d_{(1)}$ filtered output, i.e. $p_{(2)}=q_{(1)}=d_{(1)}$.
%Now $\hank_{d_{(2)}|p_{(2)}}(C^{(1)})\in \Rd^{n\times d_{(2)}p_{(2)}}$, so the number of the columns
%of the local basis matrix $\Psi^{(2)}$ should be at least $ d_{(2)}p_{(2)}$ to satisfy the frame condition \eqref{eq:ri0}. Hence, the number of output channel
%at $l=2$ becomes
%$ q_{(2)} \geq  d_{(2)}p_{(2)} \geq d_{(1)}d_{(2)}$, so
%\eqref{eq:prod} holds
%for $l=2$.
Next, we assume that \eqref{eq:prod} is true at the $(l-1)$-th layer. % , which means that there are $p_{(l)}$ signals that needs to be filtered.
Then, the number of input channel at the $l$-layer is 
$p_{(l)}= q_{(l-1)}$ and the 
 filtering operation can be represented by \eqref{eq:multifilter} or
$\Phi^{(l)\top}\hank_{d_{(l)}|p_{(l)}}(C^{(l-1)})\Psi^{(l)}$, where $\hank_{d_{(l)}|p_{(l)}}(C^{(l-1)})\in \Rd^{n\times p_{(l)}d_{(l)}}$. Thus, to guarantee the PR,
the number of columns for the local basis  $\Psi^{(l)}$ should be at least  $ p_{(l)}d_{(l)}$ to satisfy the frame condition \eqref{eq:ri0}.
This implies that   the output channel number at the $l$-th layer should be  $q_{(l)}\geq p_{(l)}d_{(l)} =q_{(l-1)}d_{(l)} $. %  \geq \prod_{i=1}^l d_{(i)}$.
This concludes the proof.

\section{Proof of Proposition~\ref{prp:depth}}
\label{ap4}

%Suppose that $\Phi^{(l)} = I_{n\times n}, \forall l\geq 1$. Then,   
Note that  the extended Hankel matrix $\hank_{d_{(l)}|p_{(l)}}(C^{(l-1)})$ in \eqref{eq:Cenc}
has the following decomposition:
\begin{eqnarray*}
\hank_{d_{(l)}|p_{(l)}}(C^{(l-1)})%\Psi^{(l+1)} 
& =&   \hank_{n|p_{(l-1)}}(C^{(l-2)}) \circul_{d_{(l)}}\left(\Psi^{(l-1)}\right) \  .
\end{eqnarray*}
Here, %$\circul_{d_{(l)}}$ 
\begin{eqnarray}\label{eq:Cdl}
\circul_{d_{(l)}}\left(\Psi^{(l-1)}\right) &:=&
\overbrace{\begin{bmatrix} \circul_{d_{(l)}}(\overline\psi^{1}_1) & \cdots & \circul_{d_{(l)}}(\overline\psi^{1}_{q_{(l-1)}})  \\ \vdots & \ddots & \vdots \\  \circul_{d_{(l)}}(\overline\psi^{p_{(l-1)}}_1) & \cdots & \circul_{d_{(l)}}(\overline\psi^{p_{(l-1)}}_{q_{(l-1)}})\end{bmatrix}}^{d_{(l)}q_{(l-1)}= d_{(l)} p_{(l)}} 
%\overbrace{ \begin{bmatrix} \mathfrak{I}^{(l)}  & \cdots & 0 \\
%\vdots & \ddots & \vdots \\ 0  & \cdots & \mathfrak{I}^{(l)}  \end{bmatrix}}^{ d_{(l)} p_{(l)}}
% \\
%& =&  \hank_{n|p_{(l)}}(C^{(l-1)})\overbrace{\begin{bmatrix} \circul(\psi^{1,(l)}_1) & \cdots & \circul(\psi^{1,(l)}_{q_{(l)}})  \\ \vdots & \ddots & \vdots \\  \circul(\psi^{p_{(l)},(l)}_1) & \cdots & \circul(\psi^{p_{(l)},(l)}_{q_{(l)}})\end{bmatrix}}^{n q_{(l)}} \overbrace{\begin{bmatrix}
%\Psi_{1,aug}^{(l+1)} \\ \vdots \\ \Psi_{p_{(l)},aug}^{(l+1)} \end{bmatrix}}^{q_{(l+1)}}
\end{eqnarray}
where $\circul_{d}(h) \in \Rd^{n\times d}$ is defined in \eqref{eq:circul}.
%$$\mathfrak{I}^{(l)}   = \begin{bmatrix} I_{d_{(l)}\times d_{(l)}} \\ 0_{(n-d_{(l)})\times d_{(l)}} \end{bmatrix} \in \Rd^{n\times d_{(l)}}.$$
 Due to the rank inequality $\rank(AB) \leq\min\{\rank(A),\rank(B) \}$,
 we have 
 \begin{eqnarray}\label{eq:rank1}
 \rank \hank_{d_{(l)}|p_{(l)}}(C^{(l-1)}) & \leq& \min\left\{ \rank\hank_{n|p_{(l-1)}}(C^{(l-2)}), \rank( \circul_{d_{(l)}}\left(\Psi^{(l-1)}\right))\right\} \notag\\
 & \leq &  \min\left\{ \rank\hank_{n|p_{(l-1)}}(C^{(l-2)}), d_{(l)}p_{(l)}\right\}
 \end{eqnarray}
Similarly, we have
\begin{eqnarray*}
\hank_{n|p_{(l)}}(C^{(l-1)})%\Psi^{(l+1)} 
& =&  \hank_{n|p_{(l-1)}}(C^{(l-2)}) \circul_{n}\left(\Psi^{(l-1)}\right)
%\overbrace{\begin{bmatrix} \circul_n(\psi^{1,(l-1)}_1) & \cdots & \circul_n(\psi^{1,(l-1)}_{q_{(l-1)}})  \\ \vdots & \ddots & \vdots \\  \circul_n(\psi^{p_{(l-1)},(l-1)}_1) & \cdots & \circul(\psi^{p_{(l-1)},(l-1)}_{q_{(l-1)}})\end{bmatrix}}^{n q_{(l-1)}= n p_{(l)}} 
% \\
%& =&  \hank_{n|p_{(l)}}(C^{(l-1)})\overbrace{\begin{bmatrix} \circul(\psi^{1,(l)}_1) & \cdots & \circul(\psi^{1,(l)}_{q_{(l)}})  \\ \vdots & \ddots & \vdots \\  \circul(\psi^{p_{(l)},(l)}_1) & \cdots & \circul(\psi^{p_{(l)},(l)}_{q_{(l)}})\end{bmatrix}}^{n q_{(l)}} \overbrace{\begin{bmatrix}
%\Psi_{1,aug}^{(l+1)} \\ \vdots \\ \Psi_{p_{(l)},aug}^{(l+1)} \end{bmatrix}}^{q_{(l+1)}}
\end{eqnarray*}
where $ \circul_{n}\left(\Psi^{(l-1)}\right)$ can be constructed  using the definition in  \eqref{eq:Cdl}.
Thus, we have
 \begin{eqnarray}\label{eq:rank2}
 \rank \hank_{n|p_{(l)}}(C^{(l-1)}) %& \leq& \min\left\{ \hank_{n|p_{(l-1)}}(C^{(l-2)}), \rank( \circul_{n}\left(\Psi^{(l-1)}\right))\right\} \notag\\
 & \leq &  \min\left\{ \rank\hank_{n|p_{(l-1)}}(C^{(l-2)}),np_{(l)}\right\} \leq  \rank\hank_{n|p_{(l-1)}}(C^{(l-2)})
 \end{eqnarray}
%
% $\circul_n(\psi) \in \Rd^{n\times n}$ denote the 
% circulant matrix for a given vector $\psi \in \Rd^d$: 
%%  \begin{eqnarray*}% \label{eq:U}
%%\circul_d(\psi) &=& \left[
%%        \begin{array}{ccc}
%%        \psi[1]  &   \cdots   &  \psi[2]    \\
%%        \vdots &   \ddots  &   \vdots \\
%%         \psi[d]    &   \ddots  & \psi[d]  \\ 
%%%         \psi[d]   &  \psi_i^{(1)}[d-1] & \cdots &   \ddots &  0 \\
%%          0 &   \ddots&   \vdots \\
%%                           0  &  \ddots &   0 \\
%%         \vdots    & \vdots       &   \psi[1] \\
%%%         \vdots & \vdots & \vdots \\
%%
%%        \end{array}
%%    \right] \in \Cd^{n\times d}  \  .
%%    \end{eqnarray*}
%%
%  \begin{eqnarray*}% \label{eq:U}
%\circul_n(\psi) &=& \left[
%        \begin{array}{ccccc}
%        \psi[1]  &   0 & \cdots   &  \psi[3] &  \psi[2]  \\
%        \psi[2]  &   \psi[1] & \cdots  &  \psi[4]   &  \psi[3]  \\
%         \vdots    & \vdots     &   &\vdots    & \vdots    \\ 
%%         \psi[d]   &  \psi_i^{(1)}[d-1] & \cdots &   \ddots &  0 \\
%          0 &   \psi[d] &  &   &  0 \\
%         \vdots    & \vdots     &      & &  \vdots    \\
%                   0  &  0 & \cdots &  \psi[2] & \psi[1] \\
%        \end{array}
%    \right] \in \Cd^{n\times n}  \  .
%    \end{eqnarray*}
%    Thus, we have
%   \begin{eqnarray}\label{eq:rank2}
% \rank \hank_{n|p_{(l)}}(C^{(l-1)}) \leq \min\{ \rank\hank_{n|p_{(l-1)}}(C^{(l-2)}), n p_{(l)}\} = \rank\hank_{n|p_{(l-1)}}(C^{(l-2)}) ,
% \end{eqnarray}  
since $np_{(l-1)}\leq np_{(l)}$. By recursively applying the inequality \eqref{eq:rank2} with \eqref{eq:rank1},
 we have 
$$ \rank \hank_{d_{(l)}|p_{(l)}}(C^{(l-1)}) \leq \min\{ \hank_{n}(f), d_{(l)} p_{(l)}\}.$$
This concludes the proof.

\section{Proof of Proposition~\ref{prp:prN}}
\label{ap5}

We will prove by construction.
Let $F^{(l)}:= \hank_{d_{(l)}|p_{(l)}}(C^{(l)})$.
%Now, we have
%\begin{eqnarray}
%\min_{\tilde \Psi} \|F-\Phi\rho(\Phi^{\top}F\Psi)\tilde\Psi^\top  \|_F^2 &=& \min_{\tilde \Psi} \left\|\Phi^{\top}F-
%\begin{bmatrix}\rho(\Phi^{\top}F\Psi_+) & \rho(\Phi^{\top}F\Psi_2) \end{bmatrix}\begin{bmatrix} \tilde\Psi_+^\top  \\ \tilde \Psi_2^\top\end{bmatrix} \right\|_F^2   \label{eq:BsN}  \ .
%\end{eqnarray}
Because $\Phi^{(l)\top}\left(F^{(l)}(-\Psi_+^{(l)})-1_nb_{enc,+}^{(l)\top} \right)= -\Phi^{(l)\top}\left( F^{(l)}\Psi_+^{(l)}+1_nb_{enc,+}^{(l)\top} \right)$,
the negative part can be retrieved from $-\rho\left(\Phi^{(l)\top}\left(F^{(l)}(-\Psi_+^{(l)})-1_nb_{enc,+}^{(l)\top} \right)\right)$, while
the positive part of $ \Phi^{(l)\top}\left(F^{(l)}\Psi_+^{(l)}+1_nb_{enc,+}^{(l)\top} \right)$ can be retained from $\rho\left( \Phi^{(l)\top}\left(F^{(l)}\Psi_+^{(l)}+1_nb_{enc,+}^{(l)\top} \right)\right)$. 
Furthermore, their non-zero parts do not overlap.
Thus,
\begin{eqnarray}\label{eq:clue}
 \Phi^{(l)\top}\left(F^{(l)}\Psi_+^{(l)}+1_nb_{enc,+}^{(l)\top}\right)   &=&  \rho\left( \Phi^{(l)\top}\left(F^{(l)}\Psi_+^{(l)}+1_nb_{enc,+}^{(l)\top} \right) \right) 
 - \rho\left(\Phi^{(l)\top}\left(F^{(l)}(-\Psi_+^{(l)})-1_nb_{enc,+}^{(l)\top} \right)\right)  \  . \notag\\
\end{eqnarray}
Accordingly,  by choosing \eqref{eq:CReLU} and \eqref{eq:Cbias},  we have
$$\tilde\Phi^{(l)}\rho\left(\Phi^{(l)\top}F^{(l)}\Psi^{(l)}+1_nb_{enc,+}^{(l)\top} \right)\tilde\Psi^{(l)\top} $$
\vspace*{-0.5cm}
\begin{eqnarray*}
%&=& \Phi^{\top}F\Psi_+\Psi_+^R \\
&=& \tilde\Phi^{(l)} \begin{bmatrix} \rho\left( \Phi^{(l)\top}\left(F^{(l)}\Psi_+^{(l)}+1_nb_{enc,+}^{(l)\top} \right) \right) & 
 \rho\left(\Phi^{(l)\top}\left(F^{(l)}(-\Psi_+^{(l)})-1_nb_{enc,+}^{(l)\top} \right)\right) 
\end{bmatrix} \begin{bmatrix} \tilde \Psi_+^{(l)\top} \\ -\tilde \Psi_+^{(l)\top} \end{bmatrix}   \\
&=& \tilde\Phi^{(l)} \left(
\rho\left( \Phi^{(l)\top}\left(F^{(l)}\Psi_+^{(l)}+1_nb_{enc,+}^{(l)\top} \right) \right)
- \rho\left(\Phi^{(l)\top}\left(F^{(l)}(-\Psi_+^{(l)})-1_nb_{enc,+}^{(l)\top} \right)\right) 
\right)
\tilde \Psi_+^{(l)\top}  \\
&=&\tilde\Phi^{(l)}\Phi^{(l)\top}\left(F^{(l)}\Psi_+^{(l)}\tilde \Psi_+^{(l)\top}+1_nb_{enc,+}^{(l)\top}\tilde \Psi_+^{(l)\top}\right)\\ %+ 1_n b_{dec}^{(l)\top} \\
&=& F^{(l)} +1_nb_{enc,+}^{(l)\top}\tilde \Psi_+^{(l)\top} %\\ %+ 1_n b_{dec}^{(l)\top} \\
%&=& F^{(l)} 
%
% \begin{bmatrix} \rho(\Phi^{\top}F\Psi_+) & \rho\left(\Phi^{\top}F(-\Psi_+)\right)\end{bmatrix} \begin{bmatrix} \Psi_+^R \\ -\Psi_+^R \end{bmatrix}
\end{eqnarray*}
where we use \eqref{eq:clue} for the third inequality and the last equality comes from \eqref{eq:id}.
Finally, we have
\begin{eqnarray*}
\hat C^{(l)} &=& \hank_{d_{(l)}|p_{(l)}}^\dag\left(F^{(l)} +1_nb_{enc,+}^{(l)\top}\tilde \Psi_+^{(l)\top} \right)+1_n b_{dec}^{(l)\top} \\
&=& C^{(l)} + \hank_{d_{(l)}|p_{(l)}}^\dag\left(1_nb_{enc,+}^{(l)\top}\tilde \Psi_+^{(l)\top} \right)+1_n b_{dec}^{(l)\top} \\
&=& C^{(l)}
\end{eqnarray*}
where the last equality comes from \eqref{eq:Cbias}.
By applying this from $l=L$ to $1$ in \eqref{eq:G}, we can see
that $f =\Qc\left(f;\{\Phi^{(j)},\tilde \Phi^{(j)}\}_{j=1}^L\right)$. 
Q.E.D.

\section{Proof of Proposition~\ref{prp:PRinsufficient}}
\label{ap6}

We will prove by construction.
From the proof of Proposition~\ref{prp:prN}, we know that
$$\Phi\rho\left(\Phi^\top\hank_{d|p}(X)[\Psi_+~-\Psi_+]\right) \begin{bmatrix} \tilde \Psi_+^\top\\ -\tilde \Psi_+^\top\end{bmatrix} =\Phi\Phi^\top
\hank_{d|p}(X) \Psi_+\tilde \Psi_+^\top  = \hank_{d|p}(X) \Psi_+\tilde \Psi_+^\top \quad , $$
where the last equality comes from the orthonormality of $\Phi$. 
Since $m\geq r$, there always exist  $\Psi_+,\tilde\Psi_+ \in \Rd^{pd\times m}$ such that
$\Psi_+\tilde \Psi_+= P_{R(V)}$ where $V\in \Rd^{pd\times r}$ denotes the right singular vectors of
$\hank_{d|p}(X)$.  
 Thus, 
 $$\Phi\rho\left(\Phi^\top\hank_{d|p}(X)\Psi\right) \tilde \Psi^\top = \hank_{d|p}(X) P_{R(V)} = \hank_{d|p}(X).$$
 By applying $ \hank_{d|p}^\dag$ to both side, we conclude the proof.

\section{Proof of Proposition~\ref{prp:resnet}}
\label{ap7}
Again the proof can be done by construction.
Specifically, from the proof of Proposition~\ref{prp:prN}, we know that
$$\rho\left(\hank_{d|p}(X)[\Psi_+~-\Psi_+]\right) \begin{bmatrix} \tilde \Psi_+^\top\\ -\tilde \Psi_+^\top\end{bmatrix} =
\hank_{d|p}(X) \Psi_+\tilde \Psi_+^\top  =  \hank_{d|p}(X) \Psi_+\tilde \Psi_+^\top \quad . $$
  Thus,  %\eqref{eq:res} can be simplified as
 \begin{eqnarray*}%\label{eq:insufficient2}
 \hank_{d|p}(X) -  R(F;\Phi,\Psi_+,\tilde\Psi_+)  
  &=&
 \hank_{d|p}(X)- \rho\left(\ \hank_{d|p}(X) -  \hank_{d|p}(X)(\Psi_+\tilde \Psi_+^\top)  \right)
 \end{eqnarray*}
 Thus,  if we choose $\Psi_+$ such that $V^\top\Psi_+=0$, we have
 \begin{eqnarray*}
% &=&
% \hank_{d|p}(X)(\Psi_+\tilde \Psi_+^\top)     \\
 \hank_{d|p}(X) -  R(F;\Phi,\Psi_+,\tilde\Psi_+)  
 &=&  \hank_{d|p}(X)- \rho\left(\hank_{d|p}(X) \right)  = 0
% \hank_{d|p}(X)(\Psi_+\tilde \Psi_+^\top)     \\
%&=& U\Sigma V^\top \Psi_+\tilde \Psi_+^\top 
\end{eqnarray*}
% then $\hank_{d|p}(X) =\Phi R(F;\Phi,\Psi_+,\tilde\Psi_+)  $ and
where
the last equality comes from the non-negativity of $X$. 
Thus, we can guarantee \eqref{eq:PRres}. Now, the remaining issue to prove is the existence of   $\Psi_+$ such that $V^\top\Psi_+=0$.
This is always possible if $r<pd$. This concludes the proof.

%\se
%\clearpage
%\bibliographystyle{IEEEtran}
%\bibliographystyle{siam}
%\bibliography{IEEEabrv,convframelets,submit_bib}
%

\end{document}